%% file: main_arxiv.tex
\documentclass{article}

\usepackage{microtype}
\usepackage{graphicx}
\usepackage{subfigure}
\usepackage{booktabs} %
\usepackage{tikz}
\usetikzlibrary{calc,decorations.pathmorphing,shapes}

\usepackage{hyperref}
\usepackage[authoryear]{natbib}
\setcitestyle{round}

\renewcommand{\cite}{\citep}

\usepackage{amsmath}
\usepackage{amssymb}
\usepackage{mathtools}
\usepackage{amsthm}
\usepackage{relsize}

\usepackage[capitalize,noabbrev]{cleveref}

\usepackage[textsize=tiny]{todonotes}

\usepackage{mathtools}
\usepackage{booktabs}
\usepackage{amsthm,bbm}
\usepackage{cleveref} %
\usepackage{subcaption} %
\usepackage{mathtools}
\usepackage{stmaryrd}
\usepackage{etoolbox} %
\RequirePackage{paralist}
\usepackage{enumitem}
\usepackage{multirow} %
\usepackage{array}
\usepackage{graphicx}
\usepackage{etoc}
\usepackage{listings}
\usepackage{eqparbox}
\usepackage{bbm}
\usepackage{quoting}

\AtEndPreamble{%
  \newtheorem{remark}[theorem]{Remark}
  \newtheorem{assumption}{Assumption}
  \newtheorem{condition}{Condition}
  \newtheorem{property}{Property}
  \theoremstyle{plain}
  \newenvironment{definition*}
                 {\pushQED{\qed}\definition}
                 {\popQED\enddefinition}  
  \newenvironment{example*}
                 {\pushQED{\qed}\example}
                 {\popQED\endremark}  
  \newenvironment{remark*}
                 {\pushQED{\qed}\remark}
                 {\popQED\endremark}  
  \newenvironment{assumption*}
                 {\pushQED{\qed}\assumption}
                 {\popQED\endremark}  
  \newenvironment{property*}
                 {\pushQED{\qed}\property}
                 {\popQED\endremark}  
}

\crefname{assumption}{Assumption}{Assumptions}
\crefname{figure}{Fig{.}}{Figs{.}}%
\crefname{table}{Table}{Tables}
\crefname{definition}{Definition}{Definitions}
\crefname{theorem}{Theorem}{Theorems}
\crefname{lemma}{Lemma}{Lemmas}
\crefname{proposition}{Proposition}{Propositions}
\crefname{corollary}{Corollary}{Corollaries}
\crefname{problem}{Problem}{Problems}
\crefname{example}{Example}{Examples}
\crefname{fact}{Fact}{Facts}
\crefname{conjecture}{Conjecture}{Conjectures}
\crefname{remark}{Remark}{Remarks}
\crefname{condition}{Condition}{Conditions}
\crefname{requirement}{Requirement}{Requirements}
\crefname{enumi}{}{}
\crefname{equation}{Eq{.}}{Eqs{.}}
\crefname{section}{Section}{Sections}

\usepackage{amsfonts,bm}

\def\eqref#1{equation~\ref{#1}}

\def\ceil#1{\lceil #1 \rceil}
\def\floor#1{\lfloor #1 \rfloor}
\def\1{\bm{1}}

\DeclareMathAlphabet{\mathsfit}{\encodingdefault}{\sfdefault}{m}{sl}
\SetMathAlphabet{\mathsfit}{bold}{\encodingdefault}{\sfdefault}{bx}{n}

\DeclareMathOperator*{\argmin}{arg\,min}

\newcommand{\mathboldcommand}[1]{\mathbb{#1}}

\newcommand{\bbF}{\mathboldcommand{F}}

\newcommand{\bbN}{\mathboldcommand{N}}

\newcommand{\bbR}{\mathboldcommand{R}}

\newcommand{\bbZ}{\mathboldcommand{Z}}

\usepackage{euscript}
\newcommand{\mathcalcommand}[1]{\mathcal{#1}}

\newcommand{\mcD}{\mathcalcommand{D}}

\newcommand{\mcM}{\mathcalcommand{M}}

\newcommand{\mcS}{\mathcalcommand{S}}

\newcommand{\mcX}{\mathcalcommand{X}}
\newcommand{\mcY}{\mathcalcommand{Y}}

\DeclareMathAlphabet{\mathpzc}{T1}{pzc}{m}{it}

\allowdisplaybreaks

\newlength{\parskiptrue}
\definecolor{lred}{rgb}{1.0, 0.5, 0.5}
\definecolor{lorange}{rgb}{1.00, 0.90, 0.20}
\definecolor{lgreen}{rgb}{0.35, 0.95, 0.35}
\definecolor{lime}{rgb}{0.9, 1.0, 0.6}
\definecolor{lblue}{rgb}{1.0, 0.85, 0.75}

\makeatletter

\newcommand*\wthelper[2]{%
        \hbox{\dimen@\accentfontxheight#1%
                \accentfontxheight#11.1\dimen@
                $\m@th#1\widetilde{#2}$%
                \accentfontxheight#1\dimen@
        }%
}
\newcommand*\accentfontxheight[1]{%
        \fontdimen5\ifx#1\displaystyle
                \textfont
        \else\ifx#1\textstyle
                \textfont
        \else\ifx#1\scriptstyle
                \scriptfont
        \else
                \scriptscriptfont
        \fi\fi\fi3
}

\newcommand*\whhelper[2]{%
        \hbox{\dimen@\accentfontxheight#1%
                \accentfontxheight#11.2\dimen@
                $\m@th#1\widehat{#2}$%
                \accentfontxheight#1\dimen@
        }%
}
\makeatother

\makeatletter
\newcommand{\oset}[3][0ex]{%
  \mathrel{\mathop{#3}\limits^{
    \vbox to#1{\kern-3\ex@
    \hbox{$\scriptstyle#2$}\vss}}}}
\makeatother

\newcommand*{\defeq}{\coloneq}

\newcommand*{\relu}{\mathrm{ReLU}}

\newcommand*{\SiLU}{\mathrm{SiLU}}
\newcommand*{\SoftPlus}{\mathrm{SoftPlus}}

\newcommand*{\GeLU}{\mathrm{GELU}}
\newcommand*{\erf}{\mathrm{erf}}
\newcommand*{\Sigmoid}{\mathrm{Sigmoid}}
\newcommand*{\ELU}{\mathrm{ELU}}
\newcommand*{\Swish}{\mathrm{Swish}}
\newcommand*{\elu}{\mathrm{ELU}}

\newcommand*{\fpq}{\mathbb{F}}
\newcommand*{\efpq}{{\overline{\mathbb{F}}}}
\newcommand*{\fmin}{\omega}
\newcommand*{\fmax}{\Omega}

\newcommand*{\round}[1]{ \left\lceil{#1}\right\rfloor }

\newcommand*{\ceilZ}[1]{ \left\lceil{#1}\right\rceil_{\mathbb{Z}} }
\newcommand*{\floorZ}[1]{ \left\lfloor{#1}\right\rfloor_{\mathbb{Z}} }

\makeatletter
\newcommand*{\tran}{\xMapsto{\sigma}{}}
\newcommand{\xMapsto}[2][]{\ext@arrow 0599{\Mapstofill@}{#1}{#2}}
\def\Mapstofill@{\arrowfill@{\Mapstochar\Relbar}\Relbar\Rightarrow}

\newcommand*{\mtran}[1]{\xrsquigarrow{{ #1}}}

\newcounter{sarrow}
\newcommand\xrsquigarrow[1]{%
\stepcounter{sarrow}%
\mathrel{\begin{tikzpicture}[baseline= {( $ (current bounding box.south) + (0,-0.5ex) $ )}]
\node[inner sep=.5ex] (\thesarrow) {$\scriptstyle #1$};
\path[draw,<-,decorate,
  decoration={zigzag,amplitude=0.7pt,segment length=1.2mm,pre=lineto,pre length=4pt}] 
    (\thesarrow.south east) -- (\thesarrow.south west);
\end{tikzpicture}}%
}

\makeatother
\makeatletter
\def\moverlay{\mathpalette\mov@rlay}
\def\mov@rlay#1#2{\leavevmode\vtop{%
   \baselineskip\z@skip \lineskiplimit-\maxdimen
   \ialign{\hfil$\m@th#1##$\hfil\cr#2\crcr}}}
\newcommand{\charfusion}[3][\mathord]{
    #1{\ifx#1\mathop\vphantom{#2}\fi
        \mathpalette\mov@rlay{#2\cr#3}
      }
    \ifx#1\mathop\expandafter\displaylimits\fi}
\makeatother

\newcommand{\lrp}[1]{\left({#1}\right)}
\newcommand{\lra}[1]{\left\langle{#1}\right\rangle}
\newcommand{\lrb}[1]{\left|{#1}\right|}
\newcommand{\lrs}[1]{\left\{{#1}\right\}}

\newcommand{\emin}{\mathfrak{e}_{\min}}
\newcommand{\emax}{\mathfrak {e}_{\max}}
\newcommand*{\mant}[1]{\mathfrak {m}_{#1}}
\newcommand*{\sn}[1]{\mathfrak {s}_{#1}}

\newcommand*{\expo}[1]{\mathfrak {e}_{#1}}

\newcommand{\nan}{\text{NaN}}
\newcommand{\hatsig}{\hat{\sigma}}

\newcommand*{\A}[3]{A_{(#1,),(#2,),#3}}
\newcommand*{\AS}[3]{AS_{(#1,),(#2,),#3}}
\newcommand*{\AP}[3]{A_{(#1),(#2),#3}}
\newcommand*{\ASP}[3]{AS_{(#1),(#2),#3}}
\newcommand*{\domain}{\mathcal{X}}

\newcommand*{\concat}{||}
\newcommand*{\bigconcat}{\Big|\Big|}
\newcommand*{\funcsum}{\#}
\newcommand*{\bigfuncsum}{{\mathlarger{\mathlarger{\#}}}}
\newcommand{\mbit}{M}
\newcommand{\ebit}{E}

\newcommand*{\indcc}[1]{\mathbbm{1}_{#1}}
\newcommand*{\indc}[2]{\mathbbm{1}_{#1}\left({#2}\right)}

\newtheorem{innercustomgeneric}{\customgenericname}
\providecommand{\customgenericname}{}
\newcommand{\newcustomtheorem}[2]{%
  \newenvironment{#1}[1]
  {%
   \renewcommand\customgenericname{#2}%
   \renewcommand\theinnercustomgeneric{##1}%
   \innercustomgeneric
  }
  {\endinnercustomgeneric}
}

\newcustomtheorem{ctheorem}{Theorem}

\newtheorem{theorem}{Theorem}%

\newtheorem{lemma}[theorem]{Lemma}
\newtheorem{corollary}[theorem]{Corollary}
\newtheorem{definition}{Definition}

\newcommand*{\commentout}[1]{}

\definecolor{lred}{rgb}{1.0, 0.5, 0.5}
\definecolor{lorange}{rgb}{1.00, 0.90, 0.20}
\definecolor{lgreen}{rgb}{0.35, 0.95, 0.35}
\definecolor{lime}{rgb}{0.9, 1.0, 0.6}
\definecolor{lblue}{rgb}{1.0, 0.85, 0.75}

\makeatother

\usepackage{url}

\usepackage[left=3cm,right=3cm]{geometry}

\begin{document}
\title{Floating-Point Networks with Automatic Differentiation Can Represent Almost All Floating-Point Functions and Their Gradients}
\date{ }

\author{ 
Sejun Park \thanks{Department of Artificial Intelligence, Korea University, Seoul, 02841 , Republic of Korea}
 \and
 Yeachan Park \thanks{Department of Mathematics and Statistics, Sejong University, Seoul, 05006, Republic of Korea }
 \and
 Geonho Hwang \thanks{Department of Mathematical Sciences, Gwangju Institute of Science and
Technology, Gwangju, 61005, Republic of Korea
\newline
Email: \texttt{sejun.park000@gmail.com, ychpark@sejong.ac.kr, hgh2134@gist.ac.kr } }    \thanks{Corresponding author}
}

\maketitle

\vskip 0.3in

\begin{abstract}
Theoretical studies show that for any differentiable function on a compact domain, there exists a neural network that approximates both the function values and gradients.
However, such results do not directly apply to practical models since they assume real parameters and exact internal operations. In contrast, real implementations only use a finite subset of reals and machine operations with round-off errors.
In this work, we investigate whether a similar result holds for neural networks under floating-point arithmetic, when the gradient with respect to the input is computed by the automatic differentiation algorithm $D^\mathtt{AD}$.
We first show that given a floating-point function $\phi$ (e.g., a loss function), arbitrary function values and gradients can be represented by a floating-point network $f$ and $D^\mathtt{AD}(\phi\circ f)$, respectively.
We further extend this result: given $\phi_1,\dots,\phi_n$, $D^\mathtt{AD}(\phi_i\circ f)$ can simultaneously represent arbitrary gradients while $f$ represents the target values, under mild conditions.
Our results hold for practical activation functions, e.g., $\relu$, $\elu$, $\GeLU$, $\Swish$, $\Sigmoid$, and $\tanh$.
\end{abstract}

\section{Introduction}\label{sec:intro}

Modern neural network applications often require not only the function values but also the gradients with respect to the input.
For example, in the field of scientific machine learning, neural networks are employed as surrogates for physical models or solutions to partial differential equations \cite{raissi19}.
Gradients are also used in various problems, such as sensitivity analysis \cite{Simonyan14a}, inverse problems \cite{bora17}, and reinforcement learning \cite{lillicrap15}, where decisions depend on how outputs change in response to small perturbations of inputs.
Furthermore, gradients are often used to revise inputs, e.g., for adversarial attacks \cite{goodfellow15,fredrikson2015model} and generation \cite{gatys16}. They are also used to improve the performance of trained networks, e.g., several studies show that fitting the gradient together with the function value can increase the final performance of networks \cite{czarnecki17,gulrajani17}.

Theoretical results show that neural networks can approximate complex functions when the networks use real parameters and exact operations (e.g., addition and multiplication).
Early results focused on fully-connected networks \cite{pinkus99}, which have been extended to modern architectures such as convolutional neural networks \cite{zhou20}, recurrent neural networks \cite{schafer07}, residual networks \cite{lin18}, and transformers \cite{Yun2020Are}.
While most of these results focus on approximating function values only, it is also known that neural networks can approximate function values and gradients at the same time \cite{li96,hwangoptimal}.

In practice, neural networks are executed on computers using floating-point arithmetic \cite{ieee754}. All network parameters and intermediate values are represented by a finite set of floating-point numbers, and internal operations follow the rules of floating-point computations, involving round-off errors.
Hence, a practical network is an algorithm determined by a sequence of floating-point operations rather than a real-valued function in theory.
Likewise, gradients in such systems are computed by applying chain-rule-based automatic differentiation algorithms (e.g., backpropagation) to this computational graph \cite{griewank08}. 
Therefore, this gradient is not the derivative of a differentiable, real-valued function, but rather the output of a computational procedure.

A few recent studies have investigated the representability of floating-point neural networks that only use floating-point parameters and operations. 
\citet{park2024expressive} show that floating-point networks can represent all functions from floating-point vectors to floating-point values when the activation function is $\relu$ or the binary step function.
This result has been extended to general activation functions \cite{hwang2025floating} and floating-point interval arithmetic \cite{hwang2025floatinginterval}.
Nevertheless, as in most existing results under exact arithmetic, these results aim to fit the function values only, not the gradients.

\subsection{Contributions}\label{sec:contribution}

In this work, we study the representability of floating-point networks together with automatic differentiation (AD). 
Specifically, we investigate the following problem:
\begin{equation*}
\begin{gathered}
\textit{Can floating-point networks with AD represent}
\textit{arbitrary function values and arbitrary gradients?}
\end{gathered}%
\end{equation*}
To better describe our results, we briefly introduce a floating-point network $f$ consisting of floating-point affine transformations $\rho_1,\dots,\rho_l$ and a pointwise activation function $\sigma$:
$$f(x)=\rho_l\circ\sigma\circ\rho_{l-1}\circ\sigma\circ\cdots\circ\sigma\circ\rho_1(x).$$
For a floating-point function $\phi_x$ (e.g., a loss function that may depend on the input $x$) %
 AD for $\phi_x\circ f$ is computed as follows: for predefined derivatives $\phi_x'$, $\rho'_i$, $\sigma'$ of $\phi_x$, $\rho_i$, $\sigma$,
\begin{align*}
D^\mathtt{AD}(\phi_x\circ f,x)\!=&\phi'_x(f(x))\!\otimes\! \rho_l'(z_{l})\!\otimes\! \sigma'(y_{l-1})\!\otimes\!\cdots\!\otimes\!\rho_1'(x)%
\end{align*}
where $z_{i}$ denotes the input to $\rho_i$, $y_{i}$ denotes the input to $\sigma$ in the layer $i$, and 
$\otimes$ denotes the floating-point multiplication. %
Importantly, the multiplications here should be computed from left to right; otherwise, the final value can be different due to the non-associativity of floating-point operations.

In our first result sketched below, we prove that floating-point networks with AD can represent almost all function values and gradients,
where $\fpq$ denotes the set of finite floating-point numbers.
\begin{ctheorem}{3.1}[informal]\label{thm:main1-informal}
For almost all $f^*:\fpq^d\to\fpq^d$, $g^*:\fpq^d\to\fpq^d$, and $\phi'_x:\fpq\to\fpq$ such that $g^*(x)=0$ for all $x$ with $\phi'_x(f^*(x))=0$, there exists a floating-point network $f$ such that $f=f^*$ and $D^\mathtt{AD}(\phi_x\circ f,x)=g^*(x).$
\end{ctheorem}
Theorem~\ref{thm:main1-informal} states that given some function $\phi'_x$, $f$ and $D^\mathtt{AD}(\phi_x\circ f)$ can represent the target function $f^*$ and the gradient $g^*$, respectively.
Here, we do not consider the antiderivative $\phi_x$ of $\phi'_x$ since it does not affect $D^\mathtt{AD}(\phi_x\circ f)(x)$.
Furthermore, $\phi'_x$ can be arbitrarily chosen. For example, if $\phi'_x(x)=1$ for all $x$, then $D^\mathtt{AD}(\phi_x\circ f)$ corresponds to the gradient of $f$ computed by AD.
For general $\phi'_x$, Theorem~\ref{thm:main1-informal} shows that floating-point networks can represent the gradient of $\phi_x\circ f$, as long as the value of $\phi'_x$ is non-zero.

Theorem~\ref{thm:main1-informal} implies that floating-point networks can solve problems that require fitting both function values and gradients simultaneously, e.g., \cite{raissi19,czarnecki17}.
Furthermore, Theorem~\ref{thm:main1-informal} shows that floating-point networks can successfully perform target tasks (by fitting function values) while having manipulated gradients (e.g., zero for all inputs).
Such manipulation can be useful when one does not want to leak information from gradients, e.g., to protect against gradient-based attacks and reasoning \cite{goodfellow15,Simonyan14a}.

Theorem~\ref{thm:main1-informal} applies to various practical activation functions, including $\relu$, $\elu$, $\GeLU$, $\Swish$, $\Sigmoid$, and $\tanh$. Here, the activation function determines the possible range of $g^*$ and $\phi'_x$. For example, for $\relu$, $\elu$, $\GeLU$, and $\Swish$, Theorem~\ref{thm:main1-informal} covers all $f^*,g^*:[-\Omega/8,\Omega/8]^d\to\fpq$ and $\phi_x':\fpq\to\fpq$ where $\Omega$ denotes the largest finite float, e.g., $\Omega\approx2^{128}$ and $2^{16}$ for the 32-bit single-precision and 16-bit half-precision formats \cite{ieee754}.
For $\Sigmoid$, Theorem~\ref{thm:main1-informal} covers $f^*,g^*:[-\Omega/4,\Omega/4]^d\to\fpq$ and $\phi_x':\fpq\to[-\Omega/4,\Omega/4]$, while $\tanh$ covers a slightly smaller range of $\phi_x'$.

The floating-point network $f$ in
Theorem~\ref{thm:main1-informal} is a function of $\phi_x'$. Hence, given $f$, $D^{\mathtt{AD}}(\phi_x\circ f,x)$ is a function of $x$ only (i.e., $g^*(x)$), and it may output an arbitrary gradient for a different choice of $\phi_x'$ (e.g., when $\phi_x$ attached to $f$ changes).
Then, can we construct a floating-point network $f$ so that $D^{\mathtt{AD}}(\phi_x\circ f,x)$ is a function of $\phi_x'$ and $x$ given $f$?
We provide a positive answer in our next result.
\begin{ctheorem}{3.2}[informal]\label{thm:main2-informal}
For almost all $f^*:\fpq^d\to\fpq^d$ and $g^*:\fpq^d\times\fpq\to\fpq^d$ such that $g^*(x,-y)=-g^*(x,y)$, there exists a floating-point network $f$ such that for almost all $\phi_x':\fpq\to\fpq$, 
$f=f^*$ and  $D^\mathtt{AD}(\phi_x\circ f)(x)=g^*(x,\phi'_x(f(x)))$.
\end{ctheorem}
Theorem~\ref{thm:main2-informal} implies that there exists a single floating-point network $f$ such that, even when $\phi'_x$ changes, $D^\mathtt{AD}(\phi_x\circ f,x)$ changes accordingly, following $g^*(x,\phi_x'(f(x)))$. 
Importantly, we can choose an arbitrary $g^*(x,\phi'_x(f(x)))$ as long as $g^*(x,-y)=-g^*(x,y)$; such a restriction cannot be bypassed due to the symmetry inherent in the definition of $D^\mathtt{AD}(\phi_x\circ f,x)$. 
We note that an analogous result cannot hold under exact arithmetic since $$\nabla(\tilde\phi\circ \tilde f)(x)=\tilde\phi'(\tilde f(x))\times\nabla f(x)$$ should be proportional to $\tilde\phi'(\tilde f(x))$ for real differentiable functions $\tilde\phi$ and $\tilde f$ under exact arithmetic. %
Nevertheless, under floating-point arithmetic, it is possible to represent gradients that are not proportional to $\phi_x'$ by exploiting the non-associativity of floating-point multiplication.
That is, floating-point networks can perform the target task by fitting function values while having artificially manipulated gradients, even if $\phi_x'$ can change over time.
We note that Theorem~\ref{thm:main2-informal} also applies to floating-point networks using activation functions that we discussed after  Theorem~\ref{thm:main1-informal}.
However, our results are not limited to these activation functions; we also provide similar results for general activation functions in \cref{thm:main1,thm:main2}.

\subsection{Organization}
We introduce the notation and precise problem setups in \cref{sec:preliminary}.
We then formally present our main results and describe their proof outlines in \cref{sec:main_results}. The formal proofs of our main results are given in \cref{sec:proof}.
We finally conclude the paper in \cref{sec:conclusion}.

\section{Preliminaries}\label{sec:preliminary}
\subsection{Notation}
Throughout this paper, $\bbN$, $\bbZ$, and $\bbR$ denote the sets of natural numbers, integers, and real numbers, respectively.
We also define $\bbN_0 \defeq \bbN \cup \{0\}$.
For $x \in \bbR$, the ceiling function is defined by
$\ceil{x}_{\bbZ} \defeq \min\{ z \in \bbZ \mid z \ge x \}$.
For $a,b \in \bbR$ and a set $\mcS$, we define
$[a,b]_{\mcS} \defeq [a,b] \cap \mcS$,
where $[a,b] \subset \bbR$ denotes the closed interval.
We also use $(a,b)_{\mcS}$, $(a,b]_{\mcS}$, and $[a,b)_{\mcS}$ analogously.
For $d \in \bbN$, we use
$\boldsymbol{0}_d \defeq (0,\dots,0) \in \bbR^d$, $\boldsymbol{1}_d \defeq (1,\dots,1) \in \bbR^d$ and $[d]\defeq\{1,\dots,d\}$.
For $k \in [d]$, let $e_k$ denote the $k$-th standard basis vector in $\bbR^d$, whose $k$-th component is equal to one and the remaining components are zero.
For $x \in \bbR^d$ and $k \in [d]$, $x$ is always a column vector, and $x_k$ denotes the $k$-th component of $x$; that is,
$x = (x_1,\dots,x_d)$.
The same notation applies to vector-valued functions $f$ with codomain $\bbR^d$, i.e.,
$f = (f_1,\dots,f_d)$.
For a function $f :\mcX \to \mcX$ and $n \in \bbN$, the $n$-fold composition of $f$ is denoted by $f^{\circ n}$. 
For a set $\mcX$, $d_1,d_2\in\bbN$, $x_1=(x_{11},\dots,x_{1d_1})\in\mcX^{d_1}$, and $x_2=(x_{21},\dots,x_{2d_2})\in\mcX^{d_2}$, we define $x_1\concat x_2\defeq(x_{11},\dots,x_{1d_1},x_{21},\dots,x_{2d_2})\in\mcX^{d_1+d_2}$.
We also use
$\bigconcat_{i=1}^n x_i \defeq ((\cdots (x_1\concat x_2)\concat x_3) \cdots \concat x_n).$

Let $\mcX,\mcY \subset \fpq^{d_1}$ and $\alpha,\beta \in \fpq^{d_2}$. We say ``a function $f \colon \fpq^{d_1} \to \fpq^{d_2}$ maps $(\mcX,\mcY)$ to $(\alpha,\beta)$'' if
\begin{equation*}
    f(x) =
    \begin{cases}
        \alpha & \text{if } x \in \mcX, \\
        \beta  & \text{if } x \in \mcY .
    \end{cases}
\end{equation*}
We denote this by $f : (\mcX,\mcY) \mapsto (\alpha,\beta)$.
For a singleton set $\mcX = \{x_0\}$ and for $\mcY =\{y_0\} \text{ or } \fpq^{d_1} \setminus \{x_0\}$, we omit the set braces and $\fpq^{d_1}$ and simply write $f \colon (x_0, y_0) \mapsto (\alpha,\beta)$ or
$f \colon (x_0, \setminus x_0) \mapsto (\alpha,\beta).$
If the value of $\mcY$ is not important, we simply write $f \colon \mcX \mapsto \alpha$.

\subsection{Floating-Point Arithmetic}\label{sec:float}
In this paper, we follow the IEEE-754 standard for floating-point arithmetic \cite{ieee754}.
The set of finite floating-point numbers $\fpq_{\mbit,\ebit} \subset \bbR$ is defined as
\begin{align*}
\fpq_{\mbit,\ebit} \defeq 
\{& s \times (1+i/{2^{\mbit}}) \times 2^{\expo{}}, s \times (i/{2^{\mbit}}) \times 2^{\emin}\;\big|\; s \in \{-1,1\},
i \in [0,2^{\mbit}-1]_{\bbZ},\ 
\expo{} \in [\emin,\emax]_{\bbZ}
\},
\end{align*}
where $\emin \defeq -2^{\ebit-1}+2$ and $\emax \defeq 2^{\ebit-1}-1$.
When it is clear from the context, we omit $\mbit$, $\ebit$, and use $\fpq \defeq \fpq_{\mbit,\ebit}$. 
We say that ``$x\in \fpq$ is even'' if the $i$ in the representation of $x$ (i.e., $s(1+i/{2^{\mbit}})2^{\expo{}}$ or $s(i/{2^{\mbit}})2^{\emin}$) is even.
The set of all floating-point numbers $\efpq$ contains some non-finite numbers:
$\efpq \defeq \fpq \cup \{-\infty, \infty, \nan\}$. However, we note that non-finite floats $-\infty$, $\infty$, and $\nan$ will not appear during the intermediate computations of our network constructions, and hence, we will mainly focus on operations on $\fpq$.
For $x \in \fpq$, we define
$x^+ \defeq \min\{ y \in \efpq \mid y > x \}$ and
$x^- \defeq \max\{ y \in \efpq \mid y < x \}$.
The largest and smallest positive finite floating-point numbers are denoted by
$\fmax \defeq (2 - 2^{-\mbit}) 2^{\emax}$ and
$\fmin \defeq 2^{\emin-\mbit}$, respectively, e.g., $\Omega^+=\infty$.

To define floating-point operations, we define the floating-point rounding operation $\round{\cdot} : \bbR\cup\{-\infty,\infty,\nan\} \to \fpq$ as follows:
\begin{equation*}
\round{x} \defeq
\begin{cases}
\argmin_{y \in \fpq} |x - y| & \text{if } |x| < \fmax + 2^{\emax-\mbit-1}, \\
-\infty & \text{if } x \le -\fmax - 2^{\emax-\mbit-1}, \\
\infty & \text{if } x \ge \fmax + 2^{\emax-\mbit-1},\\
\nan & \text{if } x = \nan.
\end{cases}
\end{equation*}
If a tie occurs in $\argmin_{y \in \fpq} |x - y|$, we choose the even $y$.
For $\hat{\sigma}:\bbR\to\bbR$, we define its rounded version $\round{\hat{\sigma}}:\efpq\to\efpq$ as $\round{\hat{\sigma}}(x) \defeq \round{\hat{\sigma}(x)}$ for $x\in\fpq$; see \cref{sec:extended_float} for the definition of $\round{\hat{\sigma}}$ on non-finite floats.
While the rounding operation can output infinities and $\nan$, we note that all the outputs of the rounding operations in this paper will be finite.
For the real operations $+$, $-$, and $\times$, we define their floating-point counterparts as follows:
for $x,y \in \fpq$, $x \oplus y \defeq \round{x + y}$, $x \ominus y \defeq \round{x - y}$, and $x \otimes y \defeq \round{x \times y}$.
The definition of floating-point operations when $x$ or $y$ is in $\{-\infty,\infty,\nan\}$ is deferred to \cref{sec:extended_float}. %

Since floating-point operations are not associative in general, the order of evaluation must be specified.
For $n \in \bbN$ and $x_1,\dots,x_n \in \efpq$, we define $\bigoplus_{i=1}^n x_i
\defeq
(\bigoplus_{i=1}^{n-1} x_i) \oplus x_n$ with the convention $\bigoplus_{i=1}^{0} x_i \defeq 0$.
Unless otherwise specified by parentheses, summations are evaluated in a left-associative way, i.e., from left to right.
In particular, we also compute $\oplus$ and $\bigoplus$ in a left-associative way: for $n_1,n_2 \in \bbN$ and $x_i,y_j \in \efpq$, 
\begin{equation*}
\bigoplus_{i=1}^{n_1} x_i \oplus \bigoplus_{j=1}^{n_2} y_j
\defeq
\bigg( \bigoplus_{i=1}^{n_1} x_i \oplus \bigoplus_{j=1}^{n_2-1} y_j \bigg) \oplus y_{n_2}.
\end{equation*}
The same ordering convention applies to
$\bigoplus_{i=1}^{n_1} \bigoplus_{j=1}^{n_2} x_{i,j}$ for $x_{i,j} \in \fpq$.
When the output of the floating-point summations is not affected by the order, we may use {$\bigoplus_{s\in \mcS}x_s$ for an index set $\mcS$ and $x_s\in\efpq$}. %

We also extend floating-point multiplication to matrices, vectors, and scalars. %
For $x\in \fpq$ and $y\in \fpq^d$, %
we define $x\otimes y=y\otimes x\defeq(x\otimes y_1,\dots,x\otimes y_d)$.
Likewise, 
for $d_1,d_2\in \bbN$, $A = [a_{i,j}]_{i,j}\in \fpq^{d_1\times d_2}$, and $x\in \fpq^{d_2}$, we define $A\otimes x$ as follows:
$$A\otimes x\defeq\lrp{\bigoplus_{j=1}^{d_2} \lrp{a_{1,j}\otimes x_j},\dots,\bigoplus_{j=1}^{d_2} \lrp{a_{d_1,j}\otimes x_j}}\!\!.$$ 
As in exact addition and multiplication, the operation $\otimes$ always takes precedence over $\oplus$.

Throughout this paper, we assume the following conditions on $\mbit$, $\ebit$, and $\sigma'$. Our assumption on $\mbit$, $\ebit$ covers float16, float32, float64, bfloat16, 8bit-E5M2, and 8bit-E4M3 (see \cref{tab:floating} in Appendix \ref{appendix:prelimnaries}).
\begin{assumption}\label{asm:pq}
$2\le \mbit\le2^{\ebit-2}$ and $ \ebit\ge 4$.
\end{assumption}

\subsection{Floating-Point Neural Network}

We consider an activation function $\sigma:\efpq\to\efpq$ and its derivative $\sigma':\efpq\to\efpq$.
One may regard $\sigma$ and $\sigma'$ as the rounded versions of a real activation function and its derivative.
Nevertheless, we do not impose any restrictions on $\sigma,\sigma'$ and consider them as general functions from $\efpq$ to $\efpq$.
We often apply $\sigma$ and $\sigma'$ to vectors: for $x=(x_1,\dots,x_d)\in\smash{\efpq^d}$, $\sigma(x)=(\sigma(x_1),\dots,\sigma(x_d))$, and $\sigma'(x)$ is the $d\times d$ diagonal matrix whose $i$-th diagonal entry is $\sigma'(x_i)$.
Throughout this paper, we assume that $\sigma'(\fpq)\subset\fpq$.

Let $L\in\bbN$ be the number of layers and $d_0,\dots,d_L\in\bbN$ denote the input/hidden/output dimensions.
For each $l\in[L]$, let \smash{$\rho_l:\efpq^{d_{l-1}}\to\efpq^{d_l}$} be the (floating-point) affine transformation defined as 
$\rho_l(x)=A_l\otimes x\oplus b_l$ for some $A_l\in\fpq^{d_l\times d_{l-1}}$ and $b_l\in\fpq^{d_l}$.
We define a floating-point neural network \smash{$f(\cdot;\theta):\efpq^{d_0}\to\efpq^{d_L}$} with parameters $\theta=(A_1,\dots,A_L,b_1,\dots,b_L)$ as
follows:
\begin{equation}\label{eq:def:nerual_network}
        f(x;\theta) \defeq \rho_L\circ \sigma\circ \rho_{L-1}\circ\cdots \circ\sigma\circ \rho_2\circ\sigma\circ \rho_1(x).
\end{equation}
We call such an $f$ ``a $\sigma$ network''.
We often omit $\theta$ and use $f(x)$ to denote $f(x;\theta)$.

For convenience, we also define variants of floating-point networks.
For $f_1,f_2,f_3$ defined as
\begin{align*}
f_1(x)&=\rho_L\circ \sigma\circ \cdots \circ\sigma\circ \rho_1\circ\sigma(x),\\
f_2(x)&=\sigma\circ \rho_{L}\circ\cdots \circ\sigma\circ \rho_2\circ\sigma\circ \rho_1(x),\\
f_3(x)&=\sigma\circ \rho_{L}\circ\cdots \circ\sigma\circ \rho_1\circ\sigma(x),
\end{align*}
we call $f_1$ ``a network starting with $\sigma$'', $f_2$ ``a network ending with $\sigma$'', and $f_3$ ``a network starting and ending with $\sigma$''.
For these variants, we count the number of layers as the number of affine transformations they contain (i.e., all $f_1$, $f_2$, and $f_3$ have $L$ layers).
For a (proper) network $f$ defined in \cref{eq:def:nerual_network},
one can observe that $f_1\circ f\circ f_2$ is a floating-point network, and the number of layers in $f_1\circ f\circ f_2$ is the sum of the numbers of layers in $f_1$, $f_2$, $f$.

For two $L$-layer networks $h_1(\cdot;\theta_1),h_2(\cdot;\theta_2):\efpq^{d}\to\efpq$ %
with $\theta_i=(A_{i1},\dots,A_{iL},b_{i1},\dots,b_{i(L-1)},0)$ for $i\in[2]$, %
we define the operation $h_1\funcsum h_2$ by concatenating $h_1$ and $h_2$ as follows:
\begin{align*}
(h_1\funcsum h_2)(x;\tilde\theta)=\tilde\rho_{L}\circ\sigma\circ\cdots\circ\sigma\circ \tilde\rho_1(x)
\end{align*}
where
$\tilde\rho_l(x)=\begin{bmatrix}A_{1l} & 0\\0 & A_{2l}\end{bmatrix}\otimes x\oplus(b_{1l}\concat b_{2l})$ 
for all $l\in[L-1]\setminus\{1\}$, $\tilde\rho_1(x)=(A_{11}\otimes x\oplus b_{11},A_{21}\otimes x\oplus b_{21})$, 
$\tilde\rho_L(x)=[A_{1L}~ A_{2L}]\otimes x$ (note that $A_{1L},A_{2L}$ are row vectors), and $\tilde\theta$ is the corresponding parameters. Note that \smash{$h_1\funcsum h_2:\efpq^{d}\to\efpq$} is also an $L$-layer network.
We also define $${\bigfuncsum}_{i=1}^n f_i \defeq ((\cdots (f_1\funcsum f_2)\funcsum f_3) \cdots \funcsum f_n).$$

\subsection{Automatic Differentiation}\label{sec:ad}

Automatic differentiation (AD) computes derivatives by repeatedly applying local differentiation rules along a computational graph \citep{griewank08}. In modern deep learning frameworks such as PyTorch and TensorFlow, gradients are typically computed by reverse-mode AD (i.e., backpropagation), where gradients are propagated backward through the primitive operations used in the forward pass. Throughout this paper, we consider AD under floating-point arithmetic, meaning that both the forward computation and the backward differentiation procedure use floating-point operations with intermediate rounding.

To derive the gradients of floating-point networks computed by the (reverse-mode) automatic differentiation algorithm, we first define the derivatives of two types of \emph{primitive functions}: affine transformations and activation functions. For \smash{$\rho:\efpq^{d_1}\to\efpq^{d_2}$} defined as $\rho(x)=A\otimes x\oplus b$ %
and for $\sigma:\efpq^{d_2}\to\efpq^{d_2}$,
we define \smash{$\mcD_{\rho,x}:\efpq^{1\times d_2}\to\efpq^{1\times d_1}$} and \smash{$\mcD_{\sigma,x}:\efpq^{1\times d_2}\to\efpq^{1\times d_2}$} as follows: for \smash{$g\in\efpq^{1\times d_2}$},
$$\mcD_{\rho,x}(g)\defeq g\otimes A~~\text{and}~~\mcD_{\sigma,x}(g)\defeq g\otimes\sigma'(x).$$
Here, $\mcD_{\rho,x}(g)$ can be viewed as the gradient of $\phi_x\circ\rho$ for some \smash{$\phi_x:\efpq^{d_2}\to\efpq$}, where $\phi_x'(\rho(x))=g$.
$\mcD_{\sigma,x}(g)$ can be considered in the same way.

These notations are useful when defining the gradient of a composite function: for $\psi=\psi_2\circ\psi_1$, where $\psi_2$ and $\psi_1$ are compositions of primitive functions,
$\mcD_{\psi,x}(g)\defeq\mcD_{\psi_1,x}\circ\mcD_{\psi_2,\psi_1(x)}(g).$
By using this definition, we can characterize the gradient (computed by automatic differentiation) of a neural network defined in \cref{eq:def:nerual_network} as follows:
\begin{align*}
  \mcD_{f, x}(g)  &=\mcD_{\rho_1,x}\circ\mcD_{\sigma,y_1}\circ\cdots\circ\mcD_{\sigma,y_{L-1}}\circ\mcD_{\rho_L,z_{L-1}}(g)\\
  &= g\otimes A_L \otimes{\sigma'}  (y_{L-1}) \otimes\cdots\otimes {\sigma'}(y_{1})  \otimes A_1
\end{align*}
where 
$y_l=\rho_l\circ\sigma\circ\cdots\circ\sigma\circ\rho_1(x)$ and $z_l=\sigma\circ y_l$.
We note that the operations in the second line \emph{must be} computed in a left-associative way, unlike in the exact arithmetic case.
Note that if $d_L=1$, then $\mcD_{f,x}(1)$ corresponds to the gradient of $f$.
Furthermore, for a scalar-valued function $\phi_x$ and its derivative $\phi_x'$, $D_{f,x}(\phi_x'(f(x)))$ corresponds to $D^\mathtt{AD}(\phi_x\circ f,x)$ introduced in \cref{sec:contribution}.
{Throughout this paper, we refer to $g$ as the \textit{input gradient} (e.g., $\phi_x'(f(x))$ in \cref{sec:contribution}) and $\mcD_{f,x}(g)$ as the \textit{output gradient}.}

Unlike classical derivatives under exact arithmetic, AD-computed gradients under floating-point arithmetic are not solely determined by the underlying function. Under exact arithmetic, two computational graphs representing the same differentiable function necessarily yield identical derivatives by the chain rule. In contrast, under floating-point arithmetic, intermediate rounding depends on the specific sequence of operations used in the computation. Consequently, two computational graphs representing the same function may produce different AD-computed gradients. For example, a linear mapping $x \mapsto ABx$ may be implemented either as two successive matrix multiplications or as a single multiplication by $AB$. Although these implementations define the same function, the gradients computed by AD can differ because intermediate rounding depends on the chosen computational graph.

\section{Main Results}\label{sec:main_results}
We are now ready to formally state our main results. To this end, we introduce the set of activation functions of interest
$$\Sigma\!\hspace{-0.02in}=\!\hspace{-0.02in}\{\hspace{-0.02in}\round{\relu}\hspace{-0.035in},\hspace{-0.02in}\round{\elu}\hspace{-0.035in},\round{\GeLU}\hspace{-0.035in},\hspace{-0.02in}\round{\Swish}\hspace{-0.035in},\hspace{-0.02in}\round{\Sigmoid}\hspace{-0.035in},\hspace{-0.02in}\round{\tanh}\hspace{-0.02in}\}\hspace{-0.02in}$$ 
and the following constants: 
\begin{align*}
M_\sigma&\!\!\:\!\defeq\!\!\:\!\begin{cases}
2^{\emax-2}&\!\!\!\!\text{if}~\sigma\!\!\:\!\in\!\!\:\!\{\!\!\:\round{\relu}\!\!\!\:,\hspace{-0.02in}\round{\elu}\!\!\!\:,\round{\GeLU}\!\!\!\:,\hspace{-0.02in}\round{\Swish}\!\!\:\}\hspace{-0.015in},\\
2^{\emax-1}&\!\!\!\!\text{if}~\sigma\!\!\:\!\in\!\!\:\!\{\!\!\:\round{\Sigmoid}\!\!\:\!,\hspace{-0.02in}\round{\tanh}\!\!\:\}\hspace{-0.015in},
\end{cases}\\
H_\sigma&\!\!\:\!\defeq\!\!\:\!\begin{cases}
\Omega&~~\text{if}~\sigma\!\!\:\!\in\!\!\:\!\{\!\!\:\round{\relu}\!\!\!\:,\hspace{-0.02in}\round{\elu}\!\!\!\:,\hspace{-0.02in}\round{\GeLU}\!\!\!\:,\hspace{-0.02in}\round{\Swish}\!\!\:\}\hspace{-0.015in},\\
2^{\mbit-1}&~~\text{if}~\sigma\!\!\:\!=\!\!\:\!\round{\Sigmoid}\!\!\!\:,\\
2^{\mbit-2}&~~\text{if}~\sigma\!\!\:\!=\!\!\:\!\round{\tanh}\!\!\!\:.
\end{cases}
\end{align*}
Our main results are stated below. The proofs of \cref{thm:main1-activation,thm:main2-activation} are in \cref{appendix:proof_lemma_main1-activation,appendix:proof_lemma_main2-activation}.
\begin{theorem}\label{thm:main1-activation}
Let $\ebit \ge 6$, $\sigma\in\Sigma$, $\domain = [-M_\sigma,M_\sigma]_{\fpq}^d$, $f^*:\domain\to\fpq$, $h^*:\domain\to[-H_\sigma,H_\sigma]_{\fpq}$, and $g^*:\fpq^{d}\to\fpq^{1\times d}$ such that $g^*(x)=0$ for all $x\in\mcX$ with $h^*(x)=0$.
Then, for any $L\ge 7$, there exists an $L$-layer $\sigma$ network $f$ such that for all $x\in\domain$, 
$$f(x)=f^*(x)~~\text{and}~~\mcD_{f,x}(h^*(x))=g^*(x).$$ 
\end{theorem}
\begin{theorem}\label{thm:main2-activation}
Let $\ebit \ge 6$, $\sigma\in\Sigma$, $\domain = [-M_\sigma,M_\sigma]_{\fpq}^d$, $f^*:\domain\to\fpq$, and $g^*:\fpq^{d}\times\fpq\to\fpq^{1\times d}$ such that $g^*(x,-y)=-g(x,y)$.
Then, for any $L\ge 2^{\ebit+1}+2\mbit+9$, there exists an $L$-layer $\sigma$ network $f$ such that for any $x\in\domain$ and $y\in [-H_\sigma,H_\sigma]_{\fpq}$,
$$f(x)=f^*(x)~~\text{and}~~\mcD_{f,x}(y)=g^*(x,y).$$
\end{theorem}

\cref{thm:main1-activation} shows that $7$-layer floating-point networks using $\sigma\in\Sigma$ can represent arbitrary function values and gradients on a wide domain $[-M_\sigma,M_\sigma]_{\fpq}^d$, when the input gradient $h^*(x)$ is bounded by $H_\sigma$ (i.e., $\phi_x'(x)$ in \cref{sec:contribution}); recall that if $h^*(x)=1$, then $\mcD_{f,x}(h^*(x))$ corresponds to the gradient of $f$ at $x$.
This implies that floating-point networks can fit almost all function values and gradients.
This result is extended to cover %
an arbitrary input gradient (i.e., $h^*(x)$ in \cref{thm:main1-activation}) that may not be a function of the input $x$ in \cref{thm:main2-activation}, using more layers.
\cref{thm:main2-activation} implies that there exists a single floating-point network $f$ such that the gradient of $\phi_x\circ f$ computed by AD can output an arbitrary value, which is a function of $y=\phi_x'(f(x))$ for a large class of $\phi_x'$.

In \cref{thm:main1-activation,thm:main2-activation}, we assume $\ebit\ge6$ in addition to \cref{asm:pq}. Nevertheless, our results cover floating-point formats using $\ge16$ bits, e.g., float32, float64, and bfloat16.
We also present similar results for $4 \le \ebit \le 5$ in \cref{appendix:proof_q5}, which require more layers but cover 8bit formats (E5M2, E4M3).

In the remainder of this section, we sketch the main idea behind \cref{thm:main1-activation}.
We first describe our network construction for \cref{thm:main1-activation} in \cref{sec:pf-outline}.
Then, we introduce technical conditions on the activation functions and describe their implications in \cref{sec:condition}.
Based on these conditions, we introduce technical lemmas and present generalizations of \cref{thm:main1-activation,thm:main2-activation} to general activation functions in \cref{sec:formal}; \cref{thm:main1-activation,thm:main2-activation} are corollaries of them. %

\subsection{Proof Sketch of \cref{thm:main1-activation}}\label{sec:pf-outline}
To prove \cref{thm:main1-activation}, we construct two floating-point networks (say, $f_1$ and $f_2$) of the same depth.
Here, $f_1$ represents the function values and has zero gradient, i.e., $f_1=f^*$ and $\mcD_{f_1,x}(y)=0$ for all $x,y$.
$f_2$ does the exact opposite: $f_2(x)=0$ and $\mcD_{f_2,x}(h^*(x))=g^*(x)$ for all $x$.
Then, $f=f_2\funcsum f_1$ satisfies the desired property in \cref{thm:main1-activation}.

We construct $f_1$ as a composition of three networks: $f_1=\psi_{13}\circ\psi_{12}\circ\psi_{11}$. Here, we design $\psi_{11}$ so that %
\begin{equation}
 \psi_{11}(x)=(a\otimes\indcc{x=z_1}\oplus b,\dots,a\otimes\indcc{x=z_m}\oplus b), \label{eq:psi11} 
\end{equation}
where $z_1,\dots,z_m$ denote all inputs in the domain and $a,b$ are some floats.
Then, $\psi_{12}$ suppresses the gradient so that $$\mcD_{\psi_{13}\circ\psi_{12},\psi_{11}(x)}(y)=(0,\dots,0),$$
for all $x,y$ by sequentially multiplying small weights.
Here, by adding proper biases, $\psi_{12}$ can preserve the output of $\psi_{11}$, i.e., $\psi_{12}\circ\psi_{11}$ is a vector of weighted indicator functions.
Lastly, $\psi_{13}$ performs a linear combination so that \vspace{-0.1in}
$$f_1(x)=\psi_{13}\circ\psi_{12}\circ\psi_{11}(x)=\bigoplus_{i=1}^mf^*(x)\times\indcc{x=z_i}.\vspace{-0.1in}$$

We construct $f_2$ as a composition of two floating-point networks: $f_2=\psi_{22}\circ\psi_{21}$.
In the forward pass, $\psi_{21}(x)=(a^\dagger\otimes\indcc{x=z_1}\oplus b^\dagger,\dots,a^\dagger\otimes\indcc{x=z_m}\oplus b^\dagger)$ for some floats $a^\dagger$ and $b^\dagger$ as in $\psi_{11}$. Then, $\psi_{22}$ suppresses this output so that $\psi_{22}\circ\psi_{21}(x)=0$ for all $x$.
The output of $\psi_{21}$ is used in the backward pass. Here, $\psi_{22}$ is designed to satisfy
\begin{align}
\mcD_{\psi_{22},\psi_{21}(x)}(y)_i=\begin{cases}
0~~&\text{if}~x\ne z_i,\\
y~~&\text{if}~x=z_i.
\end{cases}\label{eq:suppress}
\end{align}
Using this information, we further design the backward pass of $\psi_{21}$ so that
$$\mcD_{\psi_{22}\circ\psi_{21},x}(h^*(x))=
g^*(x).
$$
We also use a similar strategy to prove Theorem~\ref{thm:main2-informal}, where the network corresponding to $f_2$ requires more complex operations.

\subsection{Conditions on Activation Functions}\label{sec:condition}
In \cref{sec:pf-outline}, we briefly sketched our construction of $f$ in \cref{thm:main1-activation}.
There were three important parts in our construction.
First, we construct indicator functions for all possible elements in the domain ($\psi_{11}$ and $\psi_{21}$).
We also have a network that preserves some specific gradients and suppresses others ($\psi_{12}$ and $\psi_{22}$).
Lastly, we transform an input non-zero gradient to the target gradient ($\psi_{21}$).
In this section, we introduce conditions on general activation functions that suffice to implement these three components.
To this end, we first introduce the following definition that is necessary to represent function values.

\begin{definition}[\citet{hwang2025floating}]\label{definition:distinguishable_bounded}
Let $\sigma:\efpq\rightarrow\efpq$, $d\in\bbN$, $\domain\subset\fpq^d$, and $\mathcal{Y}\subset \efpq$. We say that ``$\domain$ is $\sigma$-distinguishable with range $\mathcal{Y}$'' if for every $x,x^\dagger\in \domain$ with $x\ne x^\dagger$, there exists a floating-point affine transformation $\rho:\fpq^d\to\efpq$ such that
\begin{equation*}
\label{eq:distinguishable}
\sigma(\rho(x))\neq\sigma(\rho(x^\dagger))
\quad\text{and}\quad
\sigma(\rho(\domain))\subset\mathcal{Y}.\vspace{-0.1in}
\end{equation*}
\end{definition}
It is known that to represent all functions from $\domain\subset\fpq^d$ to $\fpq$ using $\sigma$ networks, $\domain$ should be distinguishable with range $\fpq\cup\{-\infty,\infty\}$ (see Lemma 3.2 in \cite{hwang2025floating}).
Hence, when we formally present our results, we will assume the distinguishability of the domain $\domain$. 

We now introduce technical conditions on general $\sigma$.
\begin{condition}\label{condition:final_layer_refined}
There exist $\nu\in \bbN$, $\kappa\in\bbN_{0}$,  $k_1,\dots,k_{\nu}\in\bbZ$, and $\gamma_0,\dots,\gamma_{\nu}\in[-2^{\emax},2^{\emax}]_{\fpq}$ with $|\gamma_i-\gamma_j|\le2^{\emax}$ for all $i,j\in[\nu]\cup\{0\}$ such that \vspace{-0.1in}
\begin{itemize}[leftmargin=0.15in]
\item $\sigma(\gamma_0)=0$, $|\sigma'(\gamma_i)|,|\sigma'(\gamma_0)|\le 2^{-k_i}|\sigma(\gamma_i)|$ $\forall i\in[\nu]$,\vspace{-0.05in}
\item $[\emin-\mbit, \emax-\mbit]_{\bbZ}\subset \bigcup_{i=1}^{\nu} 
   [\expo{\sigma(\gamma_i)} + \emin, \emax -\kappa +\min(\expo{\sigma(\gamma_i)} ,  k_i)]_{\bbZ}.$\vspace{-0.05in}
\end{itemize}
\end{condition}
We use \cref{condition:final_layer_refined} to represent the function values in  $f_1$ while ensuring that the computed gradient does not overflow in the last layer.
Specifically, we use $\sigma(\gamma_0)=0$ to represent zero and $\sigma(\gamma_1),\dots,\sigma(\gamma_\nu)$ to represent non-zero function values in the last layer; the second bullet ensures that we can represent all values in $\fpq$. 
Here, if we use a large weight in the last layer, then we cannot accept a large input gradient %
$h^*(x)$
due to the overflow. $\kappa$ in \cref{condition:final_layer_refined} controls this issue: for a larger $\kappa$, we can use smaller weights in the last layer. 
Furthermore, to represent a target value $v$ by $w\otimes\sigma(\gamma_i)$,
$|w|$ should be proportional to $|v/\sigma(\gamma_i)|$. 
Hence, 
if $|\sigma'(\gamma_i)|\gg|\sigma(\gamma_i)|$, then overflow may occur for large $|h^*(x)|$ since the computed gradient in the last layer will be proportional to $|h^*(x)\sigma'(\gamma_i)v/\sigma(\gamma_i)|$.
Thus, larger values of $k_i$ and $\kappa$ are preferable to avoid overflow during the evaluation of AD.
{Since $\relu$ and its variants can attain large output values with small gradients, they satisfy $\kappa \ge \emax + 1$. On the other hand, $\Sigmoid$ and $\tanh$ have $\kappa \approx \mbit$ due to their bounded function values.}

Our next condition is used to suppress undesired gradients.
\begin{condition}\label{condition:gradient_two_exponent}
    There exist $\delta_0, \delta_1\in\fpq$ with $|\delta_0|,|\delta_1|,|\delta_0-\delta_1|\le2^{\emax}$, and $\eta\in \bbN$ such that %
    \begin{itemize}[leftmargin=0.15in]
        \item $|\sigma(\delta_0)|,|\sigma(\delta_1)|, |\sigma(\delta_0) -\sigma(\delta_1)| \le 2^{\emax}, \sigma(\delta_0)\!\ne\!\sigma(\delta_1)$,\vspace{-0.05in}
        \item   $ -2^{-\mbit-1}\times \sigma'(\delta_1)\le 2 {\sigma'(\delta_0)}\le {\sigma'(\delta_1)},$ 
        \item $2^{-\mbit}\le\lrb{\sigma'(\delta_1)} = 2^{\expo{\sigma'(\delta_1)}}\le1$, \vspace{-0.05in}%
        \item $-\expo{\sigma'(\delta_1)} < \eta\le \emax$, $\lrb{\sigma(\delta_0)},\lrb{\sigma(\delta_1)}\le 2^{\emax -\eta}.$ %
    \end{itemize}
\end{condition}
We can multiply appropriate weights to control the magnitude of gradients without overflow: choosing $w=2^e$ for $e\in[\emin,\eta]_{\bbZ}$ guarantees $w\otimes\sigma(\delta_0), w\otimes\sigma(\delta_1)\in\fpq$.
By repeatedly applying this strategy, multiplying appropriate weights, and adding appropriate biases, we can implement a network that increases or decreases the computed gradients and controls them depending on the input, as in \cref{eq:suppress}.
For $\sigma\in\Sigma$, there exist $\delta_0,\delta_1$ such that $\sigma(\delta_0)\approx0$ and $|\sigma(\delta_1)|\in[1/4,4]$.
Hence, for such $\sigma$, we have $\eta \approx\emax$ since $\max\!\lrp{|\sigma(\delta_1)|, |\sigma(\delta_0)|}\in[1/4,4]$.

\begin{condition}\label{condition:first_layer_new}

There exists $\zeta\in(0,(2^{\emax})^-]_\fpq$
such that at least one of the following conditions holds: \vspace{-0.1in}
\begin{itemize}[leftmargin=0.15in]
    \item %
    There exists $x\in \fpq\setminus(-\zeta,\zeta)_{\fpq}$ such that $\zeta\oplus|x|\in\fpq$, $ |\sigma'(x)|\ge 2^{-\mbit}$, $x$ is even,\footnote{See \cref{sec:float} for the definition of ``even''.} and\vspace{-0.05in}
    $$ \sigma([-(\zeta\oplus |x|), (\zeta\oplus |x|)]_{\fpq})\subset \left[-2^{\emax}, 2^{\emax}\right]\!.\vspace{-0.05in}$$
    \item $ |\sigma'(0)|\ge 2^{-\mbit}$ and $\sigma([-2\zeta, 2\zeta]_{\fpq})\!\subset\!\left[-2^{\emax}, 2^{\emax}\right]$.\vspace{-0.05in}
\end{itemize}
\end{condition}

\cref{condition:first_layer_new} is a technical condition used to create the desired gradient value in the first layer of $f_2$.
In particular, if the condition in the first bullet is satisfied, we map each coordinate of the input in $[-\zeta,\zeta]_{\fpq}$ to $x$ so that the resulting gradient in the backward pass is not too small (it will be multiplied by $|\sigma'(x)|\ge2^{-\mbit}$).
Here, we additionally use $\zeta\oplus |x|\in\fpq$ to avoid overflow in the first layer while keeping the activation values bounded. 
If the condition in the second bullet is satisfied, we can similarly map each coordinate of the input to zero.

\subsection{Results for General Activation Functions}\label{sec:formal}
In this section, we introduce lemmas that are used to construct $f_1$ and $f_2$, which were introduced in \cref{sec:pf-outline} using \cref{condition:final_layer_refined,condition:gradient_two_exponent,condition:first_layer_new}.
Then, using these lemmas, we derive general versions of \cref{thm:main1-activation,thm:main2-activation} that hold for general activation functions.
To better describe our results, 
given an activation function $\sigma$ satisfying \cref{condition:final_layer_refined,condition:gradient_two_exponent,condition:first_layer_new}, we define
\begin{align}
\tau \defeq\ceilZ{{\max\lrp{\frac{2^{\ebit}+\mbit}{\eta +\expo{\sigma'(\delta_1)}}+4,  \frac{2^{\ebit}+\mbit}{-\emin - \expo{\sigma'(\delta_1)}}+ 4,7}}}, \label{eq:tau}
\end{align}
where $\delta_1$ and $\eta$ are from \cref{condition:gradient_two_exponent}.

The following lemmas are for $f_1$ and $f_2$. We present the proofs of \cref{lemma:value_approximator,lemma:gradient_approximator} in \cref{sec:pflem:value_approximator,sec:pflem:gradient_approximator}.
\begin{lemma}\label[lemma]{lemma:value_approximator}
Suppose $\sigma$ satisfies Conditions~\ref{condition:final_layer_refined}--\ref{condition:gradient_two_exponent} and 
$\domain\subset[-\zeta,\zeta]_{\fpq}^d$ is $\sigma$-distinguishable with range $[-2^{\emax},2^{\emax}]$.
Let $f^*:\domain\to\fpq$.
Then, for any $L\ge 5$, 
there exists an $L$-layer $\sigma$ network $f$ such that for any $x\in\fpq^d$ and $y\in[-2^\kappa,2^\kappa]_{\fpq}$, and $L$-layer $\sigma$ network $\tilde f:\mcX\to\fpq$,
$$f(x)=f^*(x),~\mcD_{f,x}(y)={\bf0}_d^\top,~\text{and}~\mcD_{\tilde f\funcsum f,x}(y)=\mcD_{\tilde f,x}(y).$$ 
\end{lemma}

\begin{lemma}\label[lemma]{lemma:gradient_approximator}
Suppose $\sigma$ satisfies \cref{condition:final_layer_refined,condition:gradient_two_exponent,condition:first_layer_new} and 
$\domain\subset[-\zeta,\zeta]_{\fpq}^d$ is $\sigma$-distinguishable with range $[-2^{\emax},2^{\emax}]$.
Let 
$g^*:\domain\to\fpq^{1\times d}$ and 
$h^*:\fpq^{d}\to[-\Omega\cdot2^{\expo{\sigma'(\delta_1)}},\Omega\cdot2^{\expo{\sigma'(\delta_1)}}]_{\fpq}$ {such that $g^*(x)=0$ for all $x\in\domain$ with $h^*(x)=0$.} %
Then, for any $L\ge \tau$, there exists an $L$-layer $\sigma$ network $f$ such that, for any $x\in\domain$ and $L$-layer $\sigma$ network $\tilde f:\mcX\to\fpq$,
$$f(x)=0,~\mcD_{f,x}(h^*(x))=g^*(x),~\text{and}~(f\funcsum \tilde f)(x)=\tilde f(x).$$
\end{lemma}

The network in \cref{lemma:value_approximator} is for $f_1$, where three layers are used to construct the indicator function ($\psi_{11}$), and the remaining layers are used to suppress gradients ($\psi_{13}\circ\psi_{12}$). 
The network in \cref{lemma:gradient_approximator} is for $f_2$. In this case, we use two layers to control gradients as in \cref{eq:suppress} ($\psi_{22}$) and use the rest to construct the target gradient ($\psi_{21}$).
 
Let $f_1$ and $f_2$ be $\sigma$ networks of $\ge\tau$ layers in \cref{lemma:value_approximator,lemma:gradient_approximator}, respectively.
Then, by choosing $f=f_2\funcsum f_1$, \cref{thm:main1} follows, which is a generalization of \cref{thm:main1-activation}.

\begin{theorem}\label{thm:main1}
Suppose $\sigma$ satisfies \cref{condition:final_layer_refined,condition:gradient_two_exponent,condition:first_layer_new} and 
$\domain\subset[-\zeta,\zeta]_{\fpq}^d$ is $\sigma$-distinguishable with range $[-2^{\emax},2^{\emax}]$.
Let $f^*, h^*:\domain\to\fpq$, $g^*:\domain\to\fpq^{1\times d}$ such that $|h^*(x)|\le \min(\Omega\times 2^{\sigma'(\delta_1)}, 2^{\kappa})$ for all $x\in \domain$ and $g^*(x)=0$ for all $x\in \domain$ with $h^*(x)=0$. 
Then, for any $L\ge\tau$, there exists an $L$-layer $\sigma$ network $f$ such that for any $x\in\domain$,
$$f(x)=f^*(x)~~\text{and}~~\mcD_{f,x}(h^*(x))=g^*(x).$$ 
\end{theorem}

To prove \cref{thm:main2-activation}, we construct a variant of $f_2$ that accepts an arbitrary input gradient using additional layers, as illustrated in the following lemma.
The proof of \cref{lemma:gradient_approximator2} is presented in \cref{appendix:pf_lem_gradient_approximator2}.
\begin{lemma}\label[lemma]{lemma:gradient_approximator2}
Suppose $\sigma$ satisfies \cref{condition:final_layer_refined,condition:gradient_two_exponent,condition:first_layer_new} and 
$\domain\subset[-\zeta,\zeta]_{\fpq}^d$ is $\sigma$-distinguishable with range $[-2^{\emax},2^{\emax}]$.
Let 
$g^*:\domain\times \fpq\to\fpq^{1\times d}$, such that $g^*(x,-y) = -g^*(x,y)$.
Then, for any $L\ge 2^{\ebit+1} + 2\mbit +2 +\tau$, there exists an $L$-layer $\sigma$ network $f$ such that for any $x\in\domain$, $y\in [-\Omega\cdot 2^{\expo{\sigma'(\delta_1)}},\Omega\cdot 2^{\expo{\sigma'(\delta_1)}}]_{\fpq}$, and $L$-layer $\sigma$ network $\tilde f:\mcX\to\fpq$,
$$f(x)=0,~\mcD_{f, x}(y) = g^*(x,y),~\text{and}~(f\funcsum \tilde f)(x)=\tilde f(x).$$
\end{lemma}

By combining \cref{lemma:value_approximator,lemma:gradient_approximator2}, we can prove the following theorem. %

\begin{theorem}\label{thm:main2}
Suppose $\sigma$ satisfies \cref{condition:final_layer_refined,condition:gradient_two_exponent,condition:first_layer_new} and 
$\domain\subset[-\zeta,\zeta]_{\fpq}^d$ is $\sigma$-distinguishable with range $[-2^{\emax},2^{\emax}]$.
Let $f^*:\domain\to\fpq$ and 
$g^*:\domain\times \fpq\to\fpq^{1\times d}$ such that $g^*(x,-y) = -g^*(x,y)$. %
Then, for any $L\ge 2^{\ebit+1} + 2\mbit +2 +\tau$, there exists an $L$-layer $\sigma$ network $f$ such that for any $x\in\domain$ and $y\in\fpq$ with $|y|\le \min\lrp{\Omega\cdot2^{\expo{\sigma'(\delta_1)}}, 2^{\kappa}}$,
$$f(x)=f^*(x)~~\text{and}~~\mcD_{f, x}(y) = g^*(x,y).$$
\end{theorem}
We note that \cref{thm:main1-activation,thm:main2-activation} can be derived from \cref{thm:main1,thm:main2} by using constants in \cref{table:act}.
\input{figures/table_activation}

We also implement our
constructions in \cref{lemma:value_approximator} and \cref{lemma:gradient_approximator} for ReLU and observe its correctness.\footnote{The code is available at \url{https://github.com/yechanp/fp-grad-rep}.}

Although we assumed that the gradient through an activation function follows the rule
$
g \mapsto g \otimes \sigma'(x)
$
(see \cref{sec:ad}),
the discussion can be extended to a more general setup in which the gradient is computed as an arbitrary elementwise function. 
For example, if $\sigma$ is the sigmoid function, then a typical implementation is
$
g \mapsto g \otimes \sigma'(x) \otimes (1 \ominus \sigma'(x)).
$
In this case, by slightly modifying conditions, we can obtain analogous results.
For example, in \cref{condition:final_layer_refined}, which will be presented in the next section,
$
\sigma'(\gamma_i) \le 2^{-k}|\sigma(\gamma_i)|
$
should be replaced with
$
|\mathcal{D}_{\sigma,\gamma_i}(g)| \le g \otimes 2^{-k}|\sigma(\gamma_i)|.
$
Other conditions can be refined accordingly.

\section{Proofs}\label{sec:proof}

\subsection{Proof of \cref{lemma:value_approximator}}\label{sec:pflem:value_approximator}
In this section, we prove \cref{lemma:value_approximator} by explicitly constructing the target $\sigma$ network $f$ (i.e., $f_1$ in \cref{sec:pf-outline}).
Let $\{z_1,\dots,z_m\}=\domain$.
We will construct $f$ as follows:
\begin{align*}
f(x)=\bigfuncsum_{i=1}^m\lambda_i(x),~\lambda_i:(z_i,\setminus z_i)\mapsto(f^*(z_i),0),
\end{align*}
where $\lambda_i:\domain\to\efpq$ are $\sigma$ networks with the same number of layers. %
Here, we will properly design the last layer of $\lambda_i$ so that $\bigfuncsum_{i=1}^m\lambda_i(x)=f^*(x)$ in the target domain $\mcX$.

As illustrated in \cref{sec:pf-outline}, our construction of $\lambda_i$ consists of three parts: $\lambda_i=\psi_{i,13}\circ\psi_{i,12}\circ\psi_{i,11}$.
In particular, we use the following lemma to construct $\psi_{i,11}$.
The proof of \cref{lemma:previous} is presented in \cref{appendix:proof_lem_previous}.
\begin{lemma}\label[lemma]{lemma:previous}
Let $d\in\bbN$.
Suppose that $\sigma$ satisfies \cref{condition:final_layer_refined} and that $\domain\subset\fpq^{d}$ is $\sigma$-distinguishable with range $\left[-2^{\emax }, 2^{\emax }\right]$.
Then, for any $z\in \fpq^d$ and any $c\in\{\delta_0, \delta_1\}$, %
there exists a three-layer $\sigma$ network $f:\domain\rightarrow\fpq^2$ ending with the activation function such that\vspace{-0.1in}
\begin{itemize}[leftmargin=0.15in]
\item $f(x) = \begin{bmatrix}
            {\sigma}\lrp{c}\indcc{z=x} & {\sigma}\lrp{c}\indcc{z=x}
        \end{bmatrix}$ and 
\item 
       $\mcD_{f, x}\lrp{\begin{bmatrix}
           g & -g
       \end{bmatrix}}=0$
    for all $x\in\domain$ and $g\in\fpq$.
\end{itemize}
\end{lemma}

By \cref{lemma:previous}, there exists a three-layer $\sigma$ network $\psi_{i,11}$ ending with the activation such that 
\begin{equation*}\label{eq:pflem:value_approx-1}
    \psi_{i,11}:(z_i,\setminus z_i)\mapsto(\begin{bmatrix}
        \sigma(\delta_1) & \sigma(\delta_1)
    \end{bmatrix},\begin{bmatrix}
        \sigma(\delta_0) & \sigma(\delta_0)
    \end{bmatrix}),
\end{equation*}
and $\mcD_{\psi_{i,11}, x}\lrp{\begin{bmatrix}
           g& -g
       \end{bmatrix}}=0.$

We next construct $\psi_{i,13}$ and $\psi_{i,12}$ using the following lemmas.
The proofs of \cref{lemma:equivalent,lemma:final_layer} are presented in \cref{appendix:proof_lemma_equivalent,appendix:proof_final_layer}.

\begin{lemma}\label[lemma]{lemma:equivalent}
Suppose $\sigma$ satisfies \cref{condition:final_layer_refined,condition:gradient_two_exponent}.
    For any $y_1, y_2\in \fpq$ such that $|y_1|,|y_2|, |y_1-y_2|\le2^{\emax}$, there exists a one-layer $\sigma$ network $f:\fpq^2\to \fpq$ ending with the activation function such that\vspace{-0.1in}
    \begin{itemize}[leftmargin=0.15in]
    \item $f:\left(\begin{bmatrix}
            \sigma(\delta_0) \!\!&\!\! \sigma(\delta_0)
        \end{bmatrix}, \begin{bmatrix}
            \sigma(\delta_1) \!\!&\!\! \sigma(\delta_1)
        \end{bmatrix} \right)\mapsto (\sigma(y_1), \sigma(y_2))$,
    \item $\mcD_{f, x}(g) = \begin{bmatrix}
            g_1 & -g_1
        \end{bmatrix}$ for $x= \sigma(\delta_1), \sigma(\delta_0)$ and for all $g$ satisfying $g\otimes \sigma'(y_1), g\otimes \sigma'(y_2)\in \fpq$ {for some $g_1\in \fpq$ depending on $g$.} %
    \end{itemize}    
\end{lemma}

\begin{lemma}\label[lemma]{lemma:final_layer}
    Suppose $\sigma$ satisfies \cref{condition:final_layer_refined}.
    Then, for any $c\in \fpq$, there exist $n\in \bbN$, $c_i\in\{\gamma_0,\dots,\gamma_\nu\}$, and $w_i\in \fpq$ for $i\in [n]$ such that for any $g\in[-2^\kappa,2^\kappa]_{\fpq}$, 
    $$c = \bigoplus_{i=1}^n w_i\otimes \sigma(c_i)$$
    and $g\otimes  w_i\otimes \sigma'(c_i),  g\otimes w_i \otimes \sigma'(\gamma_0)\in \fpq$.
\end{lemma} 

By \cref{lemma:final_layer}, there exist $n_i\in\bbN$, $w_{i,1},\dots,w_{i,n_i}\in\fpq$, and $c_{i,1},\dots,c_{i,n_i}\in\{\gamma_1,\dots,\gamma_\nu\}$ such that
\begin{align*}
    f^*(z_i)=\bigoplus_{j=1}^{n_i}(w_{i,j}\otimes\sigma(c_{i,j})).
\end{align*}

By \cref{lemma:equivalent}, there exists a two layer network $\psi^j_{i,12}$ satisfying the followings:
$$\psi^j_{i,12}: (\begin{bmatrix}
    \sigma(\delta_0) \!&\! \sigma(\delta_0)
\end{bmatrix}, \begin{bmatrix}
    \sigma(\delta_1) \!&\! \sigma(\delta_1)
\end{bmatrix}) \mapsto (\sigma(\gamma_0), \sigma(c_{i,j}))$$
and
    \begin{equation*}
      \mcD_{\psi^j_{i,12} \circ \psi_{i,11}, x}(g) = 0
    \end{equation*}
   for $x= \sigma(\delta_1), \sigma(\delta_0)$ and for all $g$ satisfying $g\otimes \sigma'(y_1), g\otimes \sigma'(y_2)\in \fpq$.
 By defining $\psi_{i,12}:\fpq \to\fpq^{n_i}$ as $\psi_{i,12}=(\psi^1_{i,12}, \dots, \psi^{n_i}_{i,12})$,
 we get a $\sigma$ network $\psi_{i,12}:\fpq\to\fpq^{n_i}$ ending with the activation function such that 
\begin{align}
\psi_{i,12}(\delta_0)&=(\sigma(\gamma_0),\dots,\sigma(\gamma_0)),\notag\\
\psi_{i,12}(\delta_1)&=(\sigma(c_{i,1}),\dots,\sigma(c_{i,n_i})).\label{eq:pflem:value_approx-2}
\end{align}
and $\mcD_{\psi_{i,12}\circ \psi_{i,11},\sigma(\delta_0)}(g)=\mcD_{\psi_{i,12}\circ \psi_{i,11},\sigma(\delta_1)}(g)=0$ for all $g\in\fpq^{1\times n_i}$ satisfying $g_k\otimes{\sigma'}(\gamma_0),g_k\otimes\sigma'(c_{i,k})\in\fpq$ for $k\in[n_i]$.
We lastly define $\psi_{i,13}:\efpq^{n_i}\to\efpq$ as an affine transformation:
\begin{align}
\psi_{i,13}(x)=(w_{i,1},\dots,w_{i,n_i})^\top\otimes x.\label{eq:pflem:value_approx-3}
\end{align}
As we introduced at the beginning of the proof, we construct $\lambda_i$ as $\lambda_i=\psi_{i,13}\circ\psi_{i,12}\circ\psi_{i,11}$.
Then, by the definitions of $\psi_{i,11}$, $\psi_{i,12}$, and $\psi_{i,13}$ (\cref{eq:pflem:value_approx-1,eq:pflem:value_approx-2,eq:pflem:value_approx-3}),
we have $\lambda_i=f^*(z_i)\times\indcc{x=z_i}$ on $\domain$.
Furthermore, one can observe that $$\mcD_{\psi_{i,13}\circ\psi_{i,12},\delta_0}(g)=\mcD_{\psi_{i,13}\circ\psi_{i,12},\delta_1}(g)=0$$ for all $g\in[-2^\kappa,2^\kappa]_{\fpq}$, i.e., $\mcD_{\lambda_i,x}(g)=0$ for all $x\in\domain$ and $g\in[-2^\kappa,2^\kappa]_{\fpq}$.
We now construct the target network as $f=\bigfuncsum_{i=1}^m\lambda_i$.
Then, from our construction of $\lambda_i$, we have
$$f(x)=f^*(x),~\mcD_{f,x}(g)=0,~\text{and}~\mcD_{f\funcsum\tilde f,x}(g)=\mcD_{\tilde f,x}(g)$$ for all $x\in\domain$, $g\in[-2^\kappa,2^\kappa]_{\fpq}$, and $\sigma$ networks $\tilde f:\domain\to\efpq$ of the same number of layers.
Since $\psi_{i,11}$ uses three layers, $\psi_{i,12}$ uses one layer, and $\psi_{i,13}$ uses one layer, for any $L\ge5$, $f$ can be an $L$-layer $\sigma$ network.
This completes the proof of \cref{lemma:value_approximator}.

\subsection{Proof of \cref{lemma:gradient_approximator}}\label{sec:pflem:gradient_approximator}

Let $\{z_1,\dots,z_m\}=\mcX$.
As in the proof of \cref{lemma:value_approximator}, we construct the target network $f$ (i.e., $f_2$ in \cref{sec:pf-outline}) as $f=\bigfuncsum_{i=1}^m \lambda_i$, where each $\lambda_i$ satisfies that for any $x\in\mcX$,
\begin{align}
\lambda_i(x)=0~~\text{and}~~\mcD_{\lambda_i,x}(h^*(x))=g^*(z_i)\times\indcc{x=z_i}.\label{eq:pflem:gradient_approx}
\end{align}
Following our proof sketch in \cref{sec:pf-outline}, our construction of $\lambda_i$ consists of two parts: $\lambda_i=\psi_{i,22}\circ\psi_{i,21}$, where $\psi_{i,22}$ and $\psi_{i,21}$ are designed so that for $j\in\{0,1\}$,\vspace{-0.1in}
\begin{itemize}[leftmargin=0.23in]
    \item[(1)] $\psi_{i,21}:(z_i,\setminus z_i)\mapsto(\delta_1,\delta_0)$, $\mcD_{\psi_{i,21},x}(h_1^*(x))=g^*(x)$,\vspace{-0.05in}
    \item[(2)]  {$\psi_{i,22}(\delta_j)=0$, $\mcD_{\psi_{i,22},\psi_{i,21}(x)}(h^*(x))=h_1^*(x)\!\times\!\indcc{x=z_i}$,}\vspace{-0.05in}
\end{itemize}
{where $h_1:\fpq^d\to\fpq$ is a function satisfying $0<|h_1(x)|\le |h(x)|$ whenever $h(x)\neq 0$, and $h_1(x) = 0$ if $h(x) =0$.}
In particular, we construct $\psi_{i,21}$ as $\psi_{i,21}=\pi_{i,2}\funcsum \pi_{i,1}$ where $\pi_{i,1}$ and $\pi_{i,2}$ are $\sigma$ networks satisfying\vspace{-0.1in}
\begin{itemize}[leftmargin=0.35in]
    \item[(1-1)] $\pi_{i,1}:(z_i,\setminus z_i)\mapsto(\delta_1,\delta_0)$, $\mcD_{\pi_{i,1},x}(y)=0$ for all $x\in\domain$ and $y\in[-2^\kappa,2^\kappa]_{\fpq}$,\vspace{-0.05in}
    \item[(1-2)] $\pi_{i,2}(x)=0$, $\mcD_{\pi_{i,2},x}(h_1^*(x))=g^*(x)$ for all $x\in\domain$.\vspace{-0.05in}
\end{itemize}
By \cref{lemma:value_approximator}, there exists such a $\sigma$ network $\pi_{i,1}$.

We construct $\pi_{i,2}$ using the following lemmas. The proofs of \cref{lemma:gradient_control2,lemma:first_layer2} are in \cref{appendix:proof_gradient_control,appendix:proof_first_layer}.

\begin{lemma}\label[lemma]{lemma:gradient_control2}
Suppose $\sigma$ satisfies \cref{condition:gradient_two_exponent}.
Then, for any $n\in\bbN$, $y^*\in\fpq^{1\times n}$, $y\in\fpq$ with $0<\lrb{y} \le \Omega\cdot  2^{\expo{\sigma'(\delta_1)}}$ and 
$L\ge  \ceilZ{\max\lrp{\frac{2^{\ebit}+\mbit}{\eta + \expo{\sigma'(\delta_1)}}, \frac{2^{\ebit}+\mbit}{-\emin - \expo{\sigma'(\delta_1)}}}}$, there exists an $L$-layer $\sigma$ network $f:\efpq^n\to\efpq$ starting with $\sigma$ such that 
$$f: \delta_1 \times {\boldsymbol1}_n \mapsto  0~~\text{and}~~\mcD_{f, \delta_1\times{\boldsymbol1}_n}(y) = y^*.$$
\end{lemma}

\begin{lemma}\label[lemma]{lemma:first_layer2}
Suppose $\sigma$ satisfies \cref{condition:first_layer_new} and  $\domain\subset[-\zeta,\zeta]_{\fpq}^d$ is $\sigma$-distinguishable with range $[-2^{\emax}, 2^{\emax}]$.
Then, for any $y^*\in\fpq^{1\times d}$ and $z\in \domain$, there exist $n\in \bbN$, $y\in[-\Omega\times 2^{\expo{\sigma'(\delta_1)}},\Omega\times2^{\expo{\sigma'(\delta_1)}}]_{\fpq}^{1\times n}$, and a two-layer $\sigma$ network \smash{$f:\efpq^d\to\efpq^n$} such that for any $x\in[-\zeta,\zeta]_{\fpq}^d$,
$$f(x)=\delta_1\times\boldsymbol{1}_n~~\text{and}~~\mcD_{f,z}(y)=y^*.$$
\end{lemma}
By \cref{lemma:gradient_control2,lemma:first_layer2}, 
there exists a $\sigma$ network $\pi_{i,2}$ that satisfies (1-2).
Then, since $\pi_{i,1}$ is constructed using \cref{lemma:value_approximator}, $\psi_{i,21}=\pi_{i,2}\funcsum\pi_{i,1}$ satisfies (1) for all $x\in\mcX$.

Lastly, we construct $\psi_{i,22}$ using the following lemma. The proof of \cref{lemma:gradient_splitting} is presented in \cref{appendix:proof_gradient_splitting}.
\begin{lemma}\label[lemma]{lemma:gradient_splitting}
    Suppose $\sigma$ satisfies \cref{condition:final_layer_refined,condition:gradient_two_exponent}.
    Let $c_1,c_2\in \fpq$ be such that $|c_1|,|c_2|,|c_1-c_2|\le 2^{\emax}$.
     Then, there exists a two-layer $\sigma$ network $f:\fpq\to\fpq$ starting with $\sigma$ such that for any $y\in  [-\Omega\times 2^{\expo{\sigma'(\delta_1)}}, \Omega\times 2^{\expo{\sigma'(\delta_1)}}]_{\fpq}$, 
    $$f: (\delta_0, \delta_1)\mapsto (c_1, c_2),~\mcD_{f, \delta_0}(y)=0,~\text{and}~\mcD_{f, \delta_1}(y)=y_1,$$
    where $y_1\in \fpq$ is a constant satisfying $0<|y_1|\le |y|$ whenever $y\neq 0$, and $y_1 = 0$ if $y =0$. 
\end{lemma}
By \cref{lemma:gradient_splitting}, there exists a $\sigma$ network $\psi_{i,22}$ starting with the activation function such that for any $j\in\{0,1\}$ and $y\in[-\Omega\times 2^{\expo{\sigma'(\delta_1)}}, \Omega\times 2^{\expo{\sigma'(\delta_1)}}]_{\fpq}$,
$$\psi_{i,22}(\delta_j)=0~~\text{and}~~\mcD_{\psi_{i,22},\delta_j}(y)=y_1\times\indcc{j=1}.$$
Hence, together with our $\psi_{i,21}$, (2) is satisfied for all $x\in\mcX$, i.e., $\lambda_i=\psi_{i,22}\circ\psi_{i,21}$ satisfies \cref{eq:pflem:gradient_approx} on $\mcX$.
Furthermore, from our construction of $\lambda_i$, $f=\bigfuncsum_{i=1}^m\lambda_i$ satisfies 
$$f(x)=0,~\mcD_{f,x}(h_1^*(x))=g^*(x),~\text{and}~(f\funcsum \tilde f)(x)=\tilde f(x)$$
for all $x\in\mcX$ and $\sigma$ networks $\tilde f$ with the same depth.
Specifically, we note that $\mcD_{f,x}(h_1^*(x))=g^*(x)$ follows from our construction of $\psi_{i,22}$ where $\mcD_{\psi_{i,22},\psi_{i,21}(x)}=0$ if $x\ne z_i$; hence, $\mcD_{f,x}(h_1^*(x))$ can be represented by the sum of zeros and $\mcD_{\lambda_k,x}(h_1^*(x))=g^*(x)$ for $x = z_k$.
We also note that $f$ can be constructed by using $\ge\tau$ layers; see our constructions of $\pi_{i,1}$, $\pi_{i,2}$, $\psi_{i,22}$, and \cref{lemma:value_approximator,lemma:gradient_control2,lemma:first_layer2,lemma:gradient_splitting}.
This completes the proof of \cref{lemma:gradient_approximator}.

\section{Conclusion}\label{sec:conclusion}
In this work, we show that floating-point neural networks can represent arbitrary function values and gradients computed via automatic differentiation.
We further extend this result to the case where the target gradient is a function of $\phi_x'$, and $\phi_x$ denotes the function composed with the network.
Our results show the discrepancy between exact and floating-point arithmetic: the target gradient should be proportional to the input gradient under exact arithmetic, but this is not the case under floating-point arithmetic.
We believe our results help in understanding the properties of real implementations of neural networks.

\newpage
\appendix
\section{Additional Preliminaries}\label[Appendix]{appendix:prelimnaries}
\subsection{Additional Notations}\label[Appendix]{appendix:additional_notations}
If $x \in \fpq$ and $|x|\ge2^{\emin}$, we define $\sn{x}\in\{-1,1\}$, $\mant{x}\in\{1+i/2^\mbit\,|\,i\in[0,2^\mbit-1]_{\bbZ}\}$, and $\expo{x}\in[\emin,\emax]_{\bbZ}$ be the unique numbers satisfying
$x = \sn{x} \times \mant{x} \times 2^{\expo{x}}$. %
For $|x|<2^\emin$, we define $\sn{x}\in\{-1,1\}$ and $\mant{x}\in\{i/2^\mbit\,|\,i\in[0,2^\mbit-1]_{\bbZ}\}$ be the unique numbers satisfying
$x = \sn{x} \times \mant{x} \times 2^{\emin}$; in this case, $\expo{x}=\emin$.
We also define $\mant{x,0},\dots,\mant{x,\mbit} \in \{0,1\}$ as the unique values satisfying
$\mant{x} = \sum_{i=0}^{\mbit} \mant{x,i} 2^{-i}$. 

\subsection{Floating-Point Operations on Non-Finite Floats}\label[appendix]{sec:extended_float}
For $x=\nan$ and $y\in\efpq$, we define $x\oplus y=y\oplus x=x\otimes y=y\otimes x=\nan$.
For $x=\pm\infty$ and $y\in\fpq$, we define $x\oplus y=y\oplus x=\pm\infty$. We also define $\infty\oplus(-\infty)=(-\infty)\oplus\infty=\nan$, $\infty\oplus\infty=\infty$, $(-\infty)\oplus(-\infty)=-\infty$.
We define $x\ominus y=x\oplus(-y)$ where $-\nan=\nan$ and $-(\pm\infty)=\mp\infty$.
For $x=\infty$, $0<y\in\fpq\cup\{\infty\}$, and $\sn{x},\sn{y}\in\{-1,1\}$, we define $(\sn{xx})\otimes(\sn{yy})=(\sn{yy})\otimes(\sn{xx})=(\sn{x}\sn{y})\infty$.
For $x=\pm\infty$ and $y=0$, we define $x\otimes y=y\otimes x=\nan$.

For $\hat\sigma:\bbR\to\bbR$, we define
$\round{\hat\sigma}(\nan)=\nan$ and the values of $\round{\hat\sigma}$ at infinity is the rounded value of the right limit of $\hat\sigma$ if it exists; otherwise, it is $\nan$. $\round{\hat\sigma}(-\infty)$ is also defined in a similar way.

\subsection{Additional Definitions and Technical Lemmas}

We adopt the definitions of \textit{sequential addition} and \textit{transferability} from \citet{hwang2025floating}, with a slight modification.
\begin{definition}[Sequential addition \cite{hwang2025floating}]\label[definition]{def:sequential_addition}
Let $\sigma:\efpq\to\efpq$ and $\mcS_\sigma\defeq \{w\otimes \sigma(c): w,c\in \fpq~\text{with}~w\otimes \sigma(c)\in\fpq\}.$
We say that a function $f:\efpq\to\efpq$ is a ``sequential addition using $\sigma$'' if there exist $n\in \bbN$ and 
$z_1,\dots,z_n\in\mcS_\sigma$
such that for each $x\in\fpq$,
\begin{equation*}
        f(x)= x \oplus z_1\oplus\dots\oplus z_n.\label{eq:sequential_addition}
\end{equation*}
 We often drop $\sigma$ and use $\mcS$ to denote $\mcS_\sigma$ when it is clear from the context.
\end{definition}
Unlike \citet{hwang2025floating}, we omit the condition $f(\fpq)\subset \fpq$.
Sequential additions are closed under composition with neural networks. 
In particular, if $s$ is a sequential addition using $\sigma$ and 
$f:\fpq^d \to \efpq$ is a $\sigma$-neural network, then $s \circ f$ is again a $\sigma$-neural network.
Note that composing with a sequential addition does not alter the output gradient.

\begin{definition}[Transferability \cite{hwang2025floating}]\label[definition]{def:transferability}
    Let $n\in\bbN$ and $(x_1,\dots,x_n)$, $(y_1,\dots, y_n)\in \fpq^n$. We say ``$(x_1,\dots, x_n)$ is transferable to $(y_1,\dots,y_n)$ using $\sigma$'' or write ``$(x_1, \dots, x_n) \tran (y_1,\dots, y_n)$'' if there exists a sequential addition $f:\fpq\rightarrow\fpq$ using $\sigma$ such that           $f(x_i) = y_i$
    for all $i\in [n]$.
\end{definition}

The following result is a powerful tool for characterizing the representational power of sequential additions.
\begin{lemma}[Lemma 4.6 of \citet{hwang2025floating}]\label[lemma]{lemma:sequential_addition_main}
Let $\sigma:\efpq\to\efpq$ and suppose that $\sigma$ satisfies \cref{condition:final_layer_refined}.
Then, for any $y\in [-2^{\emax},2^{\emax})_{\fpq}$ and $x_1, x_2\in [-2^{\emax},2^{\emax}]_{\fpq}$ such that $x_2 - x_1\in (0,2^{\emax}]_{\fpq}$,
it holds that
\begin{equation*}
    (-2^{\emax},y,y^+, 2^{\emax})\tran (x_1, x_1, x_2, x_2).
\end{equation*}
\end{lemma} 
Note that Condition 1 in \citet{hwang2025floating} is replaced here by \cref{condition:final_layer_refined}. Since any activation function satisfying \cref{condition:final_layer_refined} also satisfies Condition 1 of \citet{hwang2025floating}, this replacement is valid.
\begin{corollary}\label[corollary]{cor:exist_seqadd}
Let $\sigma:\efpq\to\efpq$ and suppose that $\sigma$ satisfies \cref{condition:final_layer_refined}.
    Then, for any $x,y\in \fpq$,
    \begin{equation*}
        (x)\tran (y).
    \end{equation*}
\end{corollary}
\begin{proof}
    By \cref{lemma:sequential_addition_main}, we get the desired result for $|x|,|y| \le 2^{\emax}$. Therefore, it is sufficient to prove that
    \begin{equation*}
        (x)\tran \lrp{2^{\emax}}, \text{ and } \lrp{2^{\emax}}\tran (y),
    \end{equation*}
    for $|x| , |y| > 2^{\emax}$.
    
    First suppose $x,y > 2^{\emax}$. 
    By the second item of \cref{condition:final_layer_refined}, we have 
    $[\emin-\mbit, \emax-\mbit]_{\bbZ}\subset \bigcup_{i=1}^{\nu} 
   [\expo{\sigma(\gamma_i)} + \emin, \emax -\kappa +\min(\expo{\sigma(\gamma_i)} ,  k_i)]_{\bbZ}.$
    Hence there exists $\gamma_i$ such that
    $ \emax-\mbit \le \emax -\kappa +\min(\expo{\sigma(\gamma_i)},k_i) $ leading to $ \min(\expo{\sigma(\gamma_i)},k_i) \ge -\mbit$. and $2^{-\mbit}\le |\sigma(\gamma_i)|$. Therefore, there exists $w\in \fpq $ such that
    \begin{equation*}
        \frac{1}{2} \times 2^{\emax-\mbit} <w\otimes \sigma(\gamma_i)<\frac{3}{2}\times 2^{\emax-\mbit}.
    \end{equation*}
    Then, since $z \oplus \left(( - w) \otimes \sigma(\gamma_i)\right) = z^- $ for $z> 2^{\emax}$ and $z \oplus \left(  w \otimes \sigma(\gamma_i)\right) = z^+$ for $z\ge 2^{\emax}$, there exists $n_x, n_y \in \bbN$ such that
    \begin{align*}
        x\oplus \bigoplus_{j=1}^{n_x} (-w)\otimes \sigma(\gamma_i) &= 2^{\emax}, \\ 
        2^{\emax} \oplus \bigoplus_{j=1}^{n_y} \left( w\otimes \sigma(\gamma_i) \right) &= y.
    \end{align*}
    Similar arguments apply to other cases.
    This completes the proof.
\end{proof}
\begin{definition}\label[definition]{definition:affine_transformation}
Now let $\alpha_1, \alpha_2, w\in \fpq$ such that 
\begin{equation*}
    \alpha_1\otimes w, \; \alpha_2\otimes w \in \fpq.
\end{equation*}
We define $\A{\alpha_1}{\alpha_2}{w}$ as the set of affine transformations with sequential additions 
\begin{equation*}
        \A{\alpha_1}{\alpha_2}{w}\defeq
         \big\{ x\mapsto s(w\otimes x):\fpq\to\fpq \,|\, s: (w\otimes \alpha_1) \tran (\alpha_2) \big\}.
    \end{equation*}
    
   Let $\alpha_1,\beta_2, \alpha_2,\beta_2,w\in \fpq $ such that
    \begin{equation*}
         \alpha_1\otimes w \neq \beta_1 \otimes w,
    \end{equation*}
     and 
    \begin{equation*}
        |\alpha_1\otimes w|, |\beta_1\otimes w|, |\alpha_2|, |\beta_2|, |\alpha_2 - \beta_2| \le 2^{\emax}.
    \end{equation*}
    If $w(\alpha_1-\beta_1)(\alpha_2 - \beta_2)> 0$, we define $\AP{\alpha_1,\beta_1}{\alpha_2,\beta_2}{w} :\fpq\to\fpq$ as the set of affine transformations with sequential additions
    \begin{equation*}
       \AP{\alpha_1,\beta_1}{\alpha_2,\beta_2}{w}\defeq
         \big\{ x\mapsto s(w\otimes x):\fpq\to\fpq \,|\, s: (w\otimes \alpha_1, w\otimes \beta_1) \tran (\alpha_2, \beta_2) \big\}.
    \end{equation*}
    If $w(\alpha_1-\beta_2)(\alpha_2-\beta_2)< 0$, $ \AP{\alpha_1,\beta_1}{\alpha_2,\beta_2}{w}$ is defined as 
    \begin{equation*}
         \AP{\alpha_1,\beta_1}{\alpha_2,\beta_2}{w}\defeq  \AP{\alpha,\beta}{\alpha_2,\beta_2}{-w}.
    \end{equation*}
\end{definition}
Note that, by \cref{lemma:sequential_addition_main}, such a sequential addition exists since $|\sigma(\alpha_1)\otimes w|, |\sigma(\beta_1)\otimes w|, |\alpha_2|, |\beta_2|, |\alpha_2 - \beta_2| \le 2^{\emax}$.

One of our key ideas is to utilize the following module to control the input gradient while preserving the function value.
\begin{definition}
Let $\alpha_1,\alpha_2,w\in \fpq$ or $\alpha_1,\beta_1, \alpha_2,\beta_2, w\in \fpq $ satisfy the same conditions as in \cref{definition:affine_transformation}.
    We define $\AP{\alpha_1}{\alpha_2}{w}$ and $\ASP{\alpha_1,\beta_1}{\alpha_2,\beta_2}{w}$ as follows:
    \begin{equation*}
       \AS{\alpha_1}{\alpha_2}{w}\defeq \lrs{f \circ\sigma:\fpq\to\fpq \mid f\in  \A{\sigma(\alpha_1)}{\alpha_2}{w}},
    \end{equation*}
    and 
    \begin{equation*}
       \ASP{\alpha_1,\beta_1}{\alpha_2,\beta_2}{w} \defeq \lrs{f \circ\sigma:\fpq\to\fpq \mid f\in  \AP{\sigma(\alpha_1),\sigma(\beta_1)}{\alpha_2,\beta_2}{w}}.
    \end{equation*}
\end{definition}
Note that for any $f\in \ASP{\alpha_1,\beta_1}{\alpha_2,\beta_2}{w}$, we have
\begin{equation*}
    f(x) = \begin{cases}
     \alpha_2 &\text{ if } x=\alpha_1,
    \\ \beta_2 &\text{ if } x=\beta_1.
    \end{cases}
\end{equation*}
Moreover,
\begin{equation*}\label{eq:AS_gradient_calculation_appendix}
    \mcD_{f, x}(g) = \begin{cases}
            g\otimes w \otimes {\sigma'}(\alpha_1) &\text{ if } x = \alpha_1,\\ 
            g\otimes w \otimes {\sigma'}(\beta_1)  &\text{ if } x = \beta_1.
        \end{cases}
\end{equation*}

\begin{lemma}\label[lemma]{lem:exist_afftrans}
Let $\sigma:\efpq\to\efpq$ and suppose that $\sigma$ satisfies \cref{condition:final_layer_refined}.
    For $\alpha_1, \alpha_2 , w \in \fpq$ with $w \otimes \alpha_1 , w \otimes \alpha_2 \in \fpq$, then   $\A{\alpha_1}{\alpha_2}{w} \neq \emptyset$. \textit{i.e.} such sequential addition always exists.  Moreover, if $w \otimes \sigma(\alpha_1) , w \otimes \sigma(\alpha_2) \in \fpq$, then   $\AS{\alpha_1}{\alpha_2}{w} \neq \emptyset$.
\end{lemma}
\begin{proof}
    It follows from \cref{cor:exist_seqadd}.
\end{proof}

\begin{definition}
   We define $\lra{AS}_{(\alpha_{\rm{in}},),(\alpha_{\rm{out}},)}$ as follows:
    \begin{equation*}
        \lra{AS}^n_{(\alpha_{\rm{in}},),(\alpha_{\rm{out}},)}\defeq \{f=f_n\circ \dots \circ f_1 :\fpq\to \fpq \,\mid\, 
         f_i=\ASP{\alpha_{i-1},}{\alpha_{i},}{w_i}, \alpha_0=\alpha_{\rm{in}}, \alpha_n = \alpha_{\rm{out}},  i\in [n]\}.
    \end{equation*}
   We define $\lra{AS}_{(\alpha_{\rm{in}},\beta_{\rm{in}}),(\alpha_{\rm{out}},\beta_{\rm{out}})}$ as follows:
    \begin{multline*}
        \lra{AS}^n_{(\alpha_{\rm{in}},\beta_{\rm{in}}),(\alpha_{\rm{out}},\beta_{\rm{out}})}\defeq \{f=f_n\circ \dots \circ f_1 :\fpq\to \fpq\, \mid\, 
        f_i=\ASP{\alpha_{i-1},\beta_{i-1}}{\alpha_{i},\beta_{i}}{w_i}, \alpha_0=\alpha_{\rm{in}}, \beta_0=\beta_{\rm{in}}, 
        \\ \alpha_n = \alpha_{\rm{out}}, \beta_n = \beta_{\rm{out}}, i\in [n]\}.
    \end{multline*}
\end{definition}
In other words, $\lra{AS}^n_{(\alpha_{\rm{in}},),(\alpha_{\rm{out}},)}$ and $\lra{AS}^n_{(\alpha_{\rm{in}},\beta_{\rm{in}}),(\alpha_{\rm{out}},\beta_{\rm{out}})}$ are the set of $n$-layered neural networks that start with an activation layer and map $(\alpha_{\rm{in}},)$ to $(\alpha_{\rm{out}},)$ and $(\alpha_{\rm{in}},\beta_{\rm{in}})$ to $(\alpha_{\rm{out}},\beta_{\rm{out}})$, respectively.

\begin{definition}[Gradient Transferability]
    Assume that $\sigma$ satisfies \cref{condition:gradient_two_exponent}.
    We write
    \begin{equation*}
        (g_1,\dots, g_m) \mtran{n} (g'_1,\dots, g'_m),
    \end{equation*}
    if there exists $f\in \lra{AS}^n_{(\alpha_1,), (\alpha_2,)}$ such that
    \begin{equation*}
        \mcD_{f,\alpha_1}(g_i) = g'_i.
    \end{equation*}
    For a scalar value, we denote $g\mtran{n} g_2$ if and only if $(g)\mtran{n} (g_2)$.
   When the precise target value is not important—except that it does not overflow and is distinct from other outputs—we use the symbol ``$-$''.
For example,
    \begin{equation*}
        (g_1,g_2)\mtran{n} (g_1', -),
    \end{equation*}
    means that there exists $f\in \lra{AS}^n_{(\alpha_1,), (\alpha_2,)}$ such that
    \begin{equation*}
        \mcD_{f,\alpha_1}(g_1) = g'_1,
    \end{equation*}
    and
    \begin{equation*}
       g'_1\neq  \mcD_{f,\alpha_1}(g_2)\in\fpq.
    \end{equation*}
\end{definition}

\begin{definition}[Separating Point, Definition 3.5 in \cite{hwang2025floating}]\label[definition]{def:separating_point}
For $\sigma:\efpq\to\efpq$, we say that $\theta\in\fpq$ is a ``separating point of $\sigma$'' if 
\begin{align*}
    \sigma(\theta^-) \notin \{\sigma(\theta), \sigma(\theta^+)\}
    \quad\text{or}\quad
    \sigma(\theta^+) \notin \{\sigma(\theta), \sigma(\theta^-)\}.
\end{align*}

\end{definition}

\begin{definition}[Activation functions]\label[definition]{def:activation_functions}
We present the definitions of widely used activation functions. 

\begin{itemize}
    \item $\Sigmoid(x) \defeq \frac{1}{1+e^{-x}}$,
    \item $\relu(x) \defeq \max(0,x)$,
        \item $        \ELU(x) \defeq \begin{cases}
            x  &\text{ if } x\ge 0,
            \\ \exp(x)-1 &\text{ if } x< 0,
        \end{cases}$
        \item $\SiLU(x) \defeq \frac{x}{1 + e^{-x}}$, %
        \item $\SoftPlus(x) \defeq \log(1+ \exp(x))$,
        \item $\GeLU(x) \defeq \frac{x}{2}\left(1+ \erf\left(\frac{x}{\sqrt{2}}\right) \right) = \frac{x}{2}\left(1+ \frac{2}{\sqrt{\pi}} \int_0^{x/ \sqrt{2} } e^{-t^2}dt  \right)$,
\end{itemize}
\end{definition}
We also present the specification of various floating-point formats in \cref{tab:floating}. 

\begin{table}[h]
    \centering
    \caption{List of frequently-used floating-point formats.}
    \label{tab:floating}
    \begin{tabular}{l|l}
    \toprule
        Floating-point format &  $(\mbit, \ebit)$\\
        \midrule
             16-bit half-precision  \citep{ieee754} & $(10,5)$
             \\ 32-bit single-precision \citep{ieee754} & $(23,8)$
             \\ 64-bit double-precision \citep{ieee754} & $(52,11)$
             \\ 8-bit E5M2 \cite{micikevicius2022fp8}& $(2,5)$ 
             \\  8-bit E4M3 \cite{micikevicius2022fp8}  & $(3,4)$
             \\ bfloat16 \cite{bfloat,abadi2016tensorflowbfloat} & $(7,8)$
            \\ \bottomrule
    \end{tabular}
    
\end{table}

\section{Lemmas and Proofs for General Activation Functions}
The following lemma is used to align the number of layers in networks of the form $\ASP{\delta_0,\delta_1}{\delta_0,\delta_1}{w}$.
\begin{lemma}
    Suppose $\sigma$ satisfies \cref{condition:gradient_two_exponent}. Then there exists a floating-point number $w\in \fpq$ and a neural network $f\in \ASP{\delta_0,\delta_1}{\delta_0,\delta_1}{w}$ such that, for any $g\in\fpq$ satisfying $|g|\le \Omega \times 2^{\expo{\sigma'(\delta_1)}}$,
    \begin{equation*}
        \mcD_{f,\delta_1}(g) = g,  
    \end{equation*}
    and
\begin{equation*}
        \lrb{\mcD_{f,\delta_0}(g)} \le |g|.
    \end{equation*}
\end{lemma}
\begin{proof}
Take $w\defeq 2^{-\expo{\sigma'(\delta_1)}}$. 
Since $\lrb{\mcD_{f,x}(g)} = g \otimes w \otimes \sigma'(x)$, the claim then follows directly from the definition.
\end{proof}

\subsection{Lemmas about \cref{condition:final_layer_refined}}

\begin{lemma}\label[lemma]{lem:endbit_control}
    Let $K\in\fpq$ with $|K|\in [(1+2^{-\mbit+1})\times2^{-\mbit-2},1+2^{-2}-2^{-\mbit}]_{\fpq}$.
    Consider $\expo{\zeta}\in \bbZ$ such that $\emin -\mbit \le \expo{\zeta}\le \emax-\mbit$.
    Then, there exists $\gamma\in \fpq$ such that the following inequality holds:
    \begin{equation*}
      \frac{1}{2}\times 2^{\expo{\zeta}}   < \gamma\otimes K\le \frac{5}{4}\times 2^{\expo{\zeta}}.
    \end{equation*}
    Furthermore, $\gamma\le 2^{\expo{\zeta} - \expo{K}}$.
\end{lemma}
\begin{proof}
    This result follows from Lemma 24 of \cite{hwang2025floatinginterval}.
    The range of $\gamma$ follows directly from the proof.
\end{proof}

\subsection{Proof of \cref{lemma:previous}}\label[appendix]{appendix:proof_lem_previous}
\begin{proof}
\cref{lemma:previous} is a direct corollary of Lemma 4.1 in \citep{hwang2025floating}, which constructs a network $\tilde f$ such that $\tilde f(x)=\sigma(c)\indc{z}{x}$.
Note that $\tilde f$ can be decomposed as $\tilde f = \sigma\circ\rho\circ f_0$, where $n\in \bbN$, $f_0:\fpq^d\to\fpq^n$ is a two-layer neural network, and $\rho:\fpq^n\to\fpq$ is an affine transformation whose weights are all $0, 1$, or $-1$.
Write $\rho(x) = W\otimes x \oplus b$.
Define $\rho_2: \begin{bmatrix}
    W
    \\ W
\end{bmatrix}\otimes x \oplus \begin{bmatrix}
    b
    \\b 
\end{bmatrix}$.
Define $f:\fpq^d\to\fpq^2$ as $f \defeq \sigma\circ \rho_2 \circ f_0$.
Then, $f = (\tilde f, \tilde f)$ and $\mcD_{f,x}\begin{bmatrix}
    g & -g
\end{bmatrix} = \mcD_{f_0,x}(\mcD_{\sigma\circ \rho, f_0(x)})\begin{bmatrix}
    g & -g
\end{bmatrix}=0$.
This completes the proof.
\end{proof}

\subsection{Proof of \cref{lemma:equivalent}}\label[appendix]{appendix:proof_lemma_equivalent}
\begin{proof}
Without loss of generality, assume that $(\sigma(\delta_1)-\sigma(\delta_0))(y_2-y_1)\ge 0$.
In other case, we can flip the sign of weights.

Define $f:\fpq^2\to\fpq$ as follows:
\begin{equation*}
    f(x_1, x_2)\defeq \sigma(s_2(s_1(x_1) \oplus (-x_2))),
\end{equation*}
where $s_1, s_2$ are sequential additions (\cref{lemma:sequential_addition_main}) satisfying
\begin{equation*}
    s_1:(\sigma(\delta_0),\sigma(\delta_1))\tran  (2\sigma(\delta_0),2\sigma(\delta_0)),
\end{equation*}
and
\begin{equation*}
    s_2:(\sigma(\delta_0),\sigma(\delta_1))\tran  (y_1,y_2).
\end{equation*}
This construction satisfies all assumptions of the lemma.
This completes the proof.
\end{proof}

\subsection{Proof of \cref{lemma:gradient_approximator2}}\label[appendix]{appendix:pf_lem_gradient_approximator2}
Once \cref{lemma:gradient_approximator2} is established, taking $f_1$ from \cref{lemma:value_approximator} and $f_3$ from \cref{lemma:gradient_approximator2}, the sum $f_1 \funcsum f_3$ completes the proof \cref{thm:main2}.
Therefore, it suffices to prove \cref{lemma:gradient_approximator2}.
    \begin{proof}[Proof of \cref{lemma:gradient_approximator2}]
The proof follows the outline of \cref{lemma:gradient_approximator}.
Unlike in \cref{lemma:gradient_approximator}, the input gradient is independent of the input value, which necessitates the construction of a gradient indicator function depending solely on the gradient value.  \cref{lemma:gradient_flattening} establishes this construction.

Let $\lrs{t_1<\dots <t_{l}} = [0,\Omega\times 2^{\expo{\sigma'(\delta_1)}}]_{\fpq}$.
As before, we use networks $\psi_{i,21}$ and $\psi_{i,22}$ defined in the proof of \cref{lemma:gradient_approximator} to construct the gradient indicator function.
The primary difference lies in replacing $\psi_{i,22}$ with the composition $\psi_{i,23,j} \circ \psi_{i,22} $ where $\psi_{i,23,j}$ is a neural network from \cref{lemma:gradient_flattening}:
$\psi_{i,23,j}:\delta_1\mapsto 0$ and 
\begin{equation}
    \mcD_{\psi_{i,23,j}} = \begin{cases}
               0 &\text{ if } 0\le |g|< t_j,
            \\  \fmin &\text{ if } t_j  \le g\le \Omega \times 2^{\expo{\sigma'(\delta_1)}},
            \\ -\fmin &\text{ if } -t_j  \ge g\ge -\Omega \times 2^{\expo{\sigma'(\delta_1)}}.
        \end{cases}
\end{equation}
Then we obtain
\begin{equation*}
    \mcD_{\psi_{i,23,j} \circ \psi_{i,22},\delta_1}(g) = \begin{cases}
        \fmin &\text{ if } t_j  \le g\le \Omega \times 2^{\expo{\sigma'(\delta_1)}} \text{ and } x=z_i,
       \\  -\fmin &\text{ if } -t_j  \ge g\ge -\Omega \times 2^{\expo{\sigma'(\delta_1)}} \text{ and } x=z_i,
         \\0  &\text{ if } 0\le |g|< t_j \text{ or } x\neq z_i.
    \end{cases} 
\end{equation*}

Then, using arguments similar to those in the proof of \cref{lemma:gradient_approximator}, we obtain $\phi_{i,j}$ such that $\phi_{i,j}(x) = 0$ and 
$$\mcD_{\phi_{i,j},x}(y) =\begin{cases}
    g^*(x,y)\otimes \indcc{x=z_i, y\ge t_j}{} &\text{ if } y\ge 0,
    \\ g^*(x,y)\otimes \indcc{x=z_i, y\le -t_j}{}  &\text{ if } y< 0.
\end{cases} $$

Define $\Phi_i$ as 
\begin{equation*}
    \Phi_i\defeq \bigfuncsum_{j\in [m_2]}  \phi_{i,j}\ominus \phi_{i,j+1},
\end{equation*}
where $\phi_{i,m_2+1}\defeq 0$.
One can verify that $\Phi_i(x)=0$ and 
\begin{equation*}
    \mcD_{\Phi_i, x}(y) = \begin{cases}
        g^*(z_i, y) &\text{ if } x=z_i,
        \\ 0 & \text{ if } x\neq z_i.
    \end{cases}
\end{equation*}
$f\defeq \bigfuncsum_{i=1}^m \Phi_i$ gives the desired result.

\end{proof}

\begin{lemma}\label[lemma]{lemma:gradient_flattening}
Suppose $\sigma$ satisfies \cref{condition:gradient_two_exponent}.
   Let $g_0\in \fpq$ such that $g_0>0$.
     Then, for $L\ge 2^{\ebit+1} + 2\mbit +2$, there exists an $L$-layer neural network $f:\fpq\to\fpq$ such that $ f(\delta_1) = 0,$ and
    \begin{equation*}
        \mcD_{f, \delta_1}(g) = \begin{cases}
               0 &\text{ if } 0\le |g|< g_0,
            \\  \fmin &\text{ if } g_0  \le g\le \Omega \times 2^{\expo{\sigma'(\delta_1)}},
            \\ -\fmin &\text{ if } -g_0  \ge g\ge -\Omega \times 2^{\expo{\sigma'(\delta_1)}}.
        \end{cases}
    \end{equation*}
\end{lemma}
\begin{proof}

As the gradient is symmetric with respect to zero, it suffices to consider nonnegative inputs.
Define $u$ as $u\defeq 2^{\emax  +\expo{\sigma'(\delta_1)}}$.
It is sufficient to prove that for $T = 2^{\ebit+1} + 2\mbit +2$,
\begin{equation*}
    \lrp{0, g_0^-, g_0, u}
    \mtran{T} \lrp{0,0, \fmin, \fmin}.
\end{equation*}
Suppose that $2^{i}\fmin \le g_0< 2^{i+1}\fmin$ for some $i \in \bbZ$ with $0\le i<\mbit -1$.
Then, by \cref{lemma:subnormal_to_good_mantissa}, we obtain
\begin{equation*}
    \lrp{0,g_0^-, g_0, 2^{\emax-\mbit +\expo{\sigma'(\delta_1)}}}\mtran{1} \lrp{0,-, \lrp{2^i+1} \fmin,u'},
\end{equation*}
where $u'< u$.
We use $h\in \AS{\delta_1}{\delta_1}{w}$, where $w = \lrp{1+2^{-\mbit}}\times 2^{-i-1-\expo{\sigma'(\delta_1)}}$.
Then because $ \expo{\sigma'(\delta_1)} \le 0$
\begin{equation*}
   \lrp{ 2^i \fmin} \otimes w = 
   \round{(2^{-\expo{\sigma'(\delta_1)}-1}+2^{-\expo{\sigma'(\delta_1)}-1-\mbit} )\times \fmin }= 
   2^{-\expo{\sigma'(\delta_1)} -1} \fmin <\lrp{\lrp{2^i +1 } \fmin} \otimes w.
\end{equation*}
Therefore, 
\begin{equation*}
    \mcD_{h,\delta_1}(g) = g \otimes w \otimes \sigma'(\delta_1)= \begin{cases}
        0 &\text{ if } g = 2^i\fmin,
        \\ v_1 &\text{ if } g = \lrp{2^i+1}\fmin,
        \\ v_2 &\text{ if } g > 2^i\fmin,
    \end{cases}
\end{equation*}
where $v_1>0$ and $g> v_2\in \fpq$.
Thus, it suffices to prove that
\begin{equation*}
    \lrp{0, \fmin, {u}}\mtran{2^{\ebit+1} + 2\mbit - 2} \lrp{0,\fmin, \fmin}.
\end{equation*}
Use $h_2\in \AS{\delta_1}{\delta_1}{w_2}$ where $w_2 = 1^+ \otimes 2^{-\expo{\sigma'(\delta_1)}-1}$.
Then for any $g \in \fpq$ with $\fmin \le g \le u$, 

\[\mcD_{h_2, \delta_1}(g) = g \otimes w_2 \otimes \sigma'(\delta_1) = \round{1^+ \times (\mant{g} \times 2^{\expo{g}}) \times 2^{-1}} \le \round{(\mant{g}+2^{-\mbit}\times\mant{g} )\times 2^{\expo{g}-1}} \le g^{++}\otimes 2^{-1}.\]
And $\mcD_{h_2^{\circ 2}, \delta_1}(g)\le g^{++++} \times 2^{-2} \le  \frac{g}{2}$ since $\mbit \ge 2$.
By applying $h_2$ a total of $2^{\ebit+1} + 2\mbit - 2$ times, we get
\begin{equation*}
    \lrp{0,\fmin, u} \mtran{2^{\ebit+1} + 2\mbit - 2} \lrp{0,\fmin, \fmin}.
\end{equation*}
    Now consider the case $g_0\ge 2^{\emin}$.
    By \cref{lemma:to_good_mantissa}, there exist $x_1,x_2,x_3\in (1/2,1]_{\fpq}$ such that
    \begin{equation*}
        \mant{g_0^-} \otimes x_1\otimes x_2\otimes x_3 = 1,
    \end{equation*}
    and
    \begin{equation*}
        \mant{g_0} \otimes  x_1\otimes x_2\otimes x_3 = 1^+.
    \end{equation*}
    Define $w_i\in \fpq$ for $i=1,2,3$ by $w_i\defeq x_i\times 2^{-\expo{\sigma'(\beta_1)}- n_i} $, where $n_1,n_2,n_3\in [\emin, \emax]_{\bbN_0}$ satisfy $\expo{g_0} -\emin = n_1+n_2+n_3$, and let $f_i\in \AS{\delta_1}{\delta_1}{w_i}$.
    Then
    \begin{equation*}
        (g_0^-,g_0,u)\mtran{3} (-, g_0\otimes w_1\otimes w_2 \otimes w_3   , u_2) = (-, 1^+\times 2^{\emin}, u_2),
    \end{equation*}
    for some $u_2 \in \fpq$ with $0<u_2<u$.
    By using $h_4\in\AS{\delta_1}{\delta_1}{w_4}$, where $$w_4 = \begin{cases}
        \lrp{1 + 2^{\expo{\sigma'(\delta_1)}}}\times 2^{-\expo{\sigma'(\delta_1)}-\mbit-1} &\text{ if } \expo{\sigma'(\delta_1)} < -1,\\ 
         \lrp{1 + 2^{-1}}^-\times 2^{-\mbit} &\text{ if } \expo{\sigma'(\delta_1)}=-1,\\
           2^{-\mbit-1} &\text{ if } \expo{\sigma'(\delta_1)}=0,
    \end{cases}$$ we get
    \begin{equation*}
        (-, 1^+\times 2^{\emin}, u_2)\mtran{1} \lrp{0, \fmin, u_3},
    \end{equation*}
    for some $u_3 \in \fpq$ with $0<u_3<u_2$.
    Applying the same argument as above, we conclude that
    \begin{equation*}
        \lrp{0, \fmin, u_3} \mtran{2^{\ebit+1} + 2\mbit-2} \lrp{0, \fmin, \fmin}.
    \end{equation*}
    This completes the proof.

\end{proof}

\begin{lemma}\label[lemma]{lemma:subnormal_to_good_mantissa}
Consider $-k\in [-\mbit, 0]_{\bbZ}$.
    For every $i\in [0, \mbit )_{\bbZ}$ and $g\in [2^i\fmin, 2^{i+1}\fmin)$, 
    there exists $x\in \left(2^{k-1}, 2^k\right]_{\fpq}$ such that
    \begin{equation*}
        g\otimes x  \le 2^{i}\fmin,
    \end{equation*}
    and
    \begin{equation*}
        g^+\otimes x =  (2^{i}+1)\fmin.
    \end{equation*}
\end{lemma}
\begin{proof}
Let $\mcM=\{1+\sum_{i=1}^\mbit b_i2^{-i}:b_1,\dots,b_n\in\{0,1\}\}\cup\{2\}$.
Then, there exists $\alpha,\alpha'\in\mcM$ such that $g=\alpha\times2^i$, $g^+=\alpha'\times2^i$, and $\alpha'-\alpha=2^{-i}$. 
Let $d$ be the largest number in $\mcM$ such that $\alpha\times(d/2)\le(1+2^{-i-1})$. Then, it is easy to observe that such $d$ always exists and $1<d<2$.
Furthermore, since $\alpha<2$, $d<2$, and $d^+-d\le2^{-\mbit}$, it holds that 
$$(1+2^{-i-1})-\alpha\times(d/2)<\alpha\times(d^+/2)-\alpha\times(d/2)<2^{-\mbit}.$$
This implies that $\alpha\times(d/2)>1+2^{-i-1}-2^{-\mbit}$, and hence,
$$\alpha'\times(d/2)=\alpha\times(d/2)+(d/2)\times 2^{-i}>\alpha\times(d/2)+2^{-i-1}>1+2^{-i}-2^{-\mbit}$$
as $d>1$.
Let $x=d\times2^{k-1}$.
Then, $g\otimes x\in[2^{i+k},(1+2^{-i-1})\times2^{i+k}]_{\fpq}$ and $g^+\otimes x\in((1+2^{-i-1})\times2^{i+k},(1+2^{-i})\times2^{i+k}]_{\fpq}$. This implies $g\otimes x\otimes2^{-k}=2^{i}\omega$ and $g^+\otimes x\otimes2^{-k}=(1+2^{i})\times\omega$, and
completes the proof.
\end{proof}

\begin{lemma}[To $(1^+, 1)$]\label[lemma]{lemma:to_good_mantissa}
    For every $g\in [1, 2)_{\fpq}$, there exist $x_1, x_2, x_3 \in (1/2,1]_{\fpq}$ such that
    \begin{equation*}
        g \otimes x_1\otimes x_2 \otimes  x_3 = 1,
    \end{equation*}
    and
    \begin{equation*}
        g^+ \otimes  x_1\otimes x_2 \otimes  x_3 = 1^+.
    \end{equation*}
\end{lemma}

\begin{proof}
If $g=1$, it is trivial. Let $h=g^+,\; h^- = g$. 
    Denote $a^\dag \in \bbF$ such that $a \otimes a^\dag = 1^+$. By \cref{lem:one-plus}, the such $x^\dag$ always exists. 
Let $h = 1+  n 2^{-\mbit}$. Then we have 
\[ 1 > ( 1+  n \cdot 2^{-\mbit} ) ( 1^+ - n \cdot 2^{-\mbit}) = 1^+ - n^2 2^{-2\mbit}.   \]

First note that for $a,b \in \fpq$ with $a \in (1,2)_\fpq, \; b \in (2^{-1},1)_\fpq$, if $r = ab - (a \otimes b) < 0$, we have 
\begin{equation*}
    a^- \otimes b < a \times b, \label{eq:cond1}
\end{equation*} since 
\[ a^- b = (a-2^{-\mbit})b= ab-2^{-\mbit}b < a\otimes b -2^{-\mbit}b < a\otimes b - 2^{-\mbit-1}. \]

Let $\mathcal{R}_0 := \{ 1^+ + n\cdot 2^{-\mbit } : n \le 2^{\frac{\mbit-1}{2}}, \; n \in \mathbb{N} \}. $ 
If $n <2^{\frac{\mbit-1}{2}}$, since
\[ 1 > (1^+ + n\cdot 2^{-\mbit })(1^+ - n\cdot 2^{-\mbit }) = 1^+-n^2 2^{-2\mbit} > 1^+ - 2^{-\mbit-1},\]
we have $  (1^+ + n\cdot 2^{-\mbit }) \otimes (1^+ - n\cdot 2^{-\mbit }) = 1^+.$ 
and $  (1^+ + (n-1)\cdot 2^{-\mbit }) \otimes (1^+ - n\cdot 2^{-\mbit }) = 1$. 

If $n = 2^{\frac{\mbit-1}{2}}$, since 
\begin{align*}
     (1^+ + n\cdot 2^{-\mbit })(1^+ - n\cdot 2^{-\mbit } + 2^{-\mbit-1}) &= 1^+ + (n+1)2^{-2\mbit-1} = 1^+  + 2^{-\tfrac{3}{2}\mbit - \tfrac{3}{2}} + 2^{-2\mbit-1}, \\
  (1^+ + n\cdot 2^{-\mbit } + 2^{-\mbit })(1^+ - n\cdot 2^{-\mbit } + 2^{-\mbit-1}) & = 1^+  + 2^{-\tfrac{3}{2}\mbit - \tfrac{3}{2}} + n2^{-2 \mbit}- 2^{-2\mbit} \\
  &= 1 + 3 \cdot  2^{-\tfrac{3}{2}\mbit - \tfrac{3}{2}} - 2^{-2 \mbit } < 1 + 2^{-\mbit-1},
\end{align*}
we have 
\[ (1^+ + n\cdot 2^{-\mbit }) \otimes (1^+ - n\cdot 2^{-\mbit } + 2^{-\mbit-1})=1^+,  \;   (1^+ + n\cdot 2^{-\mbit } + 2^{-\mbit }) \otimes (1^+ - n\cdot 2^{-\mbit } + 2^{-\mbit-1}) = 1. \]
Therefore, if $a \in \mathcal{R}_0$ with $a \otimes x = 1^+$, then $a^- \otimes x = 1$.  

Let $x_h \in (2^{-1},1)_\fpq$ such that $h\otimes x_h = 1^+$ and $r_h = 1^+ - h x_h$. If $r_h \le 0$, then the proof is finished with $x_1 = x_h$. Suppose $r_h > 0$. We find the smallest $m \in \bbN$, such that $h x_h^{m+} - h \otimes x_h^{m+} \le 0$. Note that if $r_h >0$, then $ h \otimes x_h^+ > h \otimes x_h$ since $ h \times x_h^+ = 1^+ + r_h + h 2^{-\mbit-1}> 1^+ 2^{-\mbit-1}$. Therefore if $h x_h^{m+} - h \otimes x_h^{m+} > 0$, we have $h \otimes x_h^{m+} = 1^+ + m2^{-\mbit}$. Hence
\begin{align*}
h x_h^{m+} - h \otimes x_h^{m+} = h ( x_h+ m 2^{-\mbit-1}) - (1^+ + m2^{-\mbit})= r_h-m (2-h) 2^{-\mbit-1} > 0.
\end{align*}
Therefore, we have $$m< \frac{r_h}{(2-h)2^{-\mbit-1}}\le \frac{2^{-\mbit-1}}{(2-h)2^{-\mbit-1}}= \frac{1}{2-h},$$ leading to $m_h \ge \frac{1}{2-h}$.

We consider the following cases. \\

\textbf{ $\bullet$ Case (1) $h \le 3 \times 2^{-1}$.} \\

In this case, we have $m_h \le 2$ and $h \times x_h^{m_h+} \le  hx_h + h(m_h2^{-\mbit-1}) \le 1^{++}+ 3 \cdot 2^{-\mbit-1} \le 1^{4+}$. Hence $h \otimes x_h^{m_h+}  \in \mathcal{R}_0$, and it completes the proof with $x_1 = x_h^{m_h+}, x_2 = (h \otimes x_h^{m_h+})^\dag$.

\textbf{ $\bullet$ Case (2) $3 \times 2^{-1} < h <=2^{3-} $.} \\

In this case, we have $m_h \le \frac{1}{2-h} \le \frac{2^{\mbit}}{3}$ and $h \times x_h^{m_h+} \le  hx_h + h(m_h2^{-\mbit-1}) \le (1 + 3 \cdot 2^{-\mbit-1})+ \frac{h}{6} \le 3 \times 2^{-1}$. Hence it reduces to case (1). 

\textbf{ $\bullet$ Case (3) $h = 2^{--} $.} \\

In this case, we have

\begin{align*}
    (2- 2 \cdot 2^{-\mbit})( 2^{-1} + 2^{-2}+ 2^{-\mbit}) &=  
    1+ 2^{-1} + 2^{-\mbit-1}- 2^{-2\mbit+1}, \\ 
    (2- 2 \cdot 2^{-\mbit}) \otimes ( 2^{-1} + 2^{-2}+ 2^{-\mbit}) &=  1+ 2^{-1} ,  \\
    (2-3\cdot 2^{-\mbit})( 2^{-1} + 2^{-2}+ 2^{-\mbit}) &= 1+ 2^{-1} - 2^{-\mbit-1}- 3\cdot 2^{-2\mbit} , \\
    (2-3\cdot 2^{-\mbit}) \otimes ( 2^{-1} + 2^{-2}+ 2^{-\mbit}) &=  1+ 2^{-1} - 2^{-\mbit}. 
\end{align*}
Hence it reduces to case (2) with $ x_1 = 2^{-1} + 2^{-2}+ 2^{-\mbit}, \;  h \otimes x_1 =  1+ 2^{-1} , \; h^- \otimes x_1 = 1+ 2^{-1} - 2^{-\mbit}$.

\textbf{ $\bullet$ Case (4) $h = 2^{-} $.} \\

In this case, we have 
\begin{align*}
    (2-2^{-\mbit})( 2^{-1} + 2^{-2}) &= 1+2^{-1}-3 \cdot 2^{-\mbit-2}, \quad  (2-2^{-\mbit}) \otimes ( 2^{-1} + 2^{-2}) =  1+ 2^{-1} - 2^{-\mbit},  \\
    (2-2\cdot 2^{-\mbit})( 2^{-1} + 2^{-2}) &= 1+2^{-1}-3 \cdot 2^{-\mbit-1}, \quad (2-2\cdot 2^{-\mbit}) \otimes ( 2^{-1} + 2^{-2}) =  1+ 2^{-1} - 2 \cdot 2^{-\mbit}. 
\end{align*}
Hence it reduces to case (2) with $ x_1 = 2^{-1} + 2^{-2}, \;  h \otimes x_1 =  1+ 2^{-1} - 2^{-\mbit}, \; h^- \otimes x_1 = 1+ 2^{-1} - 2 \cdot 2^{-\mbit} $. 

\end{proof}

\begin{lemma}\label[lemma]{lem:multiple_product_leading_zero}
    For any $x \in \fpq$ and $2^e \in \fpq$ for some $e \in \bbN$, we have $ x \otimes \underbrace{2^e \otimes \dots \otimes 2^e}_{m \text{ times }} = 0$ for $m = \ceilZ{\frac{2^{\ebit}+\mbit}{e}}$. 
\end{lemma}
\begin{proof}
    Since $|x| < 2^{1+\emax}$, we have 
    \begin{align*}
        | x \otimes \underbrace{2^e \otimes \dots \otimes 2^e}_{m \text{ times }}| &\le   2^{1+\emax} \otimes \underbrace{2^e \otimes \dots \otimes 2^e}_{m \text{ times }} \\
        &= 2^{1+\emax -em} = 2^{1+2^{\ebit-1} - (2^{\ebit}+\mbit)} = 2^{1-2^{\ebit-1}-\mbit} = 2^{\emin-\mbit-1} = \frac12 \fmin.
    \end{align*}
Hence $ | x \otimes \underbrace{2^e \otimes \dots \otimes 2^e}_{m \text{ times }}|=0$. 
\end{proof}

\begin{lemma}[Lemma 17 of \citet{park26}]\label[lemma]{lem:one-plus}
Suppose $\mbit \ge 2$. For any $x\in(1,2]_{\fpq}$, there exist $y\in(2^{-1},1]_{\fpq}$ and $z\in(2^{-1},1^+]_{\fpq}$ such that $x\otimes y=1^+$.
\end{lemma}

\subsection{Proof of \cref{lemma:final_layer}}\label[appendix]{appendix:proof_final_layer} 

\begin{proof}
By \cref{lemma:make_proper_number}, it suffices to show that there exist $w\in \fpq$ and $ \gamma \in \{ \gamma_0 , \dots \gamma_\nu \}$ such that
\begin{equation*}
    \frac{1}{2}\times 2^{\expo{c}-\mbit} < \lrp{w\otimes \sigma(\gamma_i)} <\frac{3}{2}\times 2^{\expo{c}-\mbit}, \quad  \expo{c} \in [\emin-\mbit, \emax-\mbit]_{\bbZ}
\end{equation*}
and
\begin{equation*}
    g\otimes w\otimes \sigma'(\gamma_i), \quad g\otimes w\otimes \sigma'(\gamma_0)\in \fpq. 
\end{equation*}
By \cref{condition:final_layer_refined}, there exists $i\in [\nu]$ such that
\begin{equation*}
    \expo{c} \in [\expo{\sigma(\gamma_i)} + \emin, \emax -\kappa +\min\lrp{\expo{\sigma(\gamma_i)}, k_i}].
\end{equation*}

By \cref{lem:endbit_control}, there exists $v \in \fpq$ such that $ \frac12 < v \otimes \mant{\sigma(\gamma_i)} < \frac32 $. Let $ w = e^{\expo{c}-\expo{ \sigma(\gamma_i)}-\mbit} \times v$ and we have
\begin{equation*}
    \frac{1}{2}\times 2^{\expo{c}-\mbit}<w\otimes \sigma(\gamma_i) = v \otimes 2^{\expo{c}} < \frac{3}{2}\times2^{\expo{c}-\mbit}.
\end{equation*}

Moreover,
\begin{align*}
\lrb{g\otimes w \otimes \sigma'(\gamma_i)} &\le 2^\kappa \otimes w \otimes 2^{-k_i} |\sigma(\gamma_i)|  
\le \frac32 \times 2^{\expo{c}-\mbit+\kappa-k_i} \le \frac32 \times 2^{\expo{c}} < \fmax.
\end{align*}

A similar estimate shows that
\begin{equation*}
    \lrb{g\otimes w \otimes \sigma'(\gamma_0)} \le \fmax.
\end{equation*}
This completes the proof.
\end{proof}

\subsection{Proof of \cref{lemma:gradient_control2}}\label[appendix]{appendix:proof_gradient_control}
\begin{lemma}\label[lemma]{lemma:gradient_control}
   Suppose $\sigma$ satisfies \cref{condition:gradient_two_exponent}.
Consider $g_1, g_0\in \fpq$ such that $0<\lrb{g_1} \le \Omega\times  2^{\expo{\sigma'(\delta_1)}}$.
    Then, for any $L\ge  \ceilZ{\max\lrp{\frac{2^{\ebit}+\mbit}{\eta + \expo{\sigma'(\delta_1)}}, \frac{2^{\ebit}+\mbit}{-\emin - \expo{\sigma'(\delta_1)}}}}$, there exists an $L$-layer $\sigma$ network $f$ starting with $\sigma$ such that 
    $$f:  \delta_1 \mapsto  0~~\text{and}~~\mcD_{f, \delta_1}(g_0) = g_1.$$
    \end{lemma}
\begin{proof}

By \cref{condition:gradient_two_exponent}, we have 
$|w \otimes \sigma(\delta_0) | , |w \otimes \sigma(\delta_1) |  < 2^{\emax+1}$ for $w \in \fpq$ with $|w| <2^{\eta+1}$. 
Hence for all $w \in \fpq$ with $|w| <2^{\eta+1}$, we have $\AS{\delta_1}{\delta_1}{w} \neq \emptyset$ by \cref{lem:exist_afftrans}. 

Note that $\sigma'(\delta_1)= \sn{\sigma'(\delta_1)}\times  2^{\expo{\sigma'(\delta_1)}}$ for $\sn{\sigma'(\delta_1)} \in \{1, -1 \}$ by \cref{condition:gradient_two_exponent}.
We only consider the case $\sn{\sigma'(\delta_1)} = 1$, as the negative case is symmetric.

First, consider the case $\mant{g_0} = \mant{g_1}$.
      For each $i \in [\emin, \eta]_\bbZ$, define $w_i\defeq    2^{i}$ and choose $h_{w_i}\in \ASP{\delta_1,}{0,}{w_i}$.
    Then $ \mcD_{h_i, \delta_1}(g_0) $ can be calculated as 
    \begin{equation*}
        \mcD_{h_i, \delta_1}(g_0) = g_0\otimes w_i\otimes { \sigma'(\delta_1)} = g_0 \otimes 2^{i}\otimes 2^{\expo{\sigma'(\delta_1)}}.
        \end{equation*}
Therefore, we can increase and decrease the exponent by at most $\eta + \expo{\sigma'(\delta_1)}$ and $-\emin - \expo{\sigma'(\delta_1)}$, respectively.
Note that $\eta + \expo{\sigma'(\delta_1)} \ge 1$ by \cref{condition:gradient_two_exponent}.

For any $g_0 = \sn{g_0}\times \mant{g_1} \times 2^{\expo{g_0}}$,
there exists $i_1, \dots, i_{m}\in  [\emin,\eta]_{\bbZ}$ such that $\sum_{j=1}^m \lrp{i_j +\expo{\sigma'(\delta_1)}}= \expo{g_1} - \expo{g_0}$ with $m\ge \max\lrp{\frac{2^{\ebit}+\mbit}{\eta + \expo{\sigma'(\delta_1)}}, \frac{2^{\ebit}+\mbit}{-\emin  - \expo{\sigma'(\delta_1)}}}$.

Let $w_j = \begin{cases}
    2^{i_j} \; &\text{if} \; j \in [m-1] \\
    \sn{g_0} \times 2^{i_m} \; &\text{if} \; j = m
\end{cases}$ and $h_{j}\in \AS{\delta_1}{0}{w_j}$ and $h\defeq h_1\circ h_2\circ\dots \circ h_m$.
Then, we have
\begin{equation*}
   h(\delta_0) = \delta_1, \quad \mcD_{h,\delta_1}(g_0) = g_1.
\end{equation*}

Now consider the general case ($\mant{g_0} \ne \mant{g_1}$ ).
We now adjust the mantissa by modifying some of the weights $2^{i_j}$ to $\mant{}\times 2^{i_j}$ to adjust the mantissa.
Note that we may target either $2g_1$ or $g_1/2$ as the original maximal exponent gap is $2^{\ebit}+\mbit-1$. 

By \cref{lem:one-plus}, there exists $\mant{g_0}^\dag \in \fpq$ such that $\mant{g_0}^\dag \otimes \mant{g_0}=  1^+$.

If $\mant{g_0} = 1$, let $w_j = \begin{cases}
    2^{i_j} \; &\text{if} \; j \in [m-1] \\
    \sn{g_0} \times \mant{g_1} \times 2^{i_m} \; &\text{if} \; j = m
\end{cases}$.

If $\mant{g_0} = 1^+$, let $w_j = \begin{cases}
    2^{i_j} \; &\text{if} \; j \in [m-2] \\
    2^- \times 2^{i_{m-1}} \; &\text{if} \;  j = m-1 \\
    \sn{g_0} \times \mant{g_1} \times 2^{i_m} \; &\text{if} \; j = m
\end{cases}$.

If $\mant{g_0} \ne 1 , 1^+$, let $w_j = \begin{cases}
    2^{i_j} \; &\text{if} \; j \in [m-1] \\
    \sn{g_0} \times \mant{g_0}^\dag \times 2^{i_m} \; &\text{if} \; j = m
\end{cases}$.

If $\mant{\mcD_{h_1, \delta_1}(g_0)} = 1^+ $, by taking $w_j = \begin{cases}
    \mant{g_1} \times 2^{i_j} \; &\text{if} \; j =1 \\
    2^- \times 2^{i_j} \; &\text{if} \; j =2 \\
    2^{i_j} \; &\text{if} \; j \in [m-1] \setminus [2] \\
    \sn{g_0} \times \mant{g_0}^\dag \times 2^{i_m} \; &\text{if} \; j = m
\end{cases}$, we have $\mcD_{h_1, \delta_1}(g_0)=g_1$. 

If $\mant{\mcD_{h_1, \delta_1}(g_0)} \ne 1^+$, then $\mcD_{h_1, \delta_1}(g_0)$ is subnormal and there are two possible cases. 

If $\mant{\mcD_{h_1, \delta_1}(g_0)} = 1$, by taking $w_j = \begin{cases}
    \mant{g_1} \times 2^{i_j} \; &\text{if} \; j =1 \\
    2^{i_j} \; &\text{if} \; j \in [m-1] \setminus [1] \\
    \sn{g_0} \times \mant{g_0}^\dag \times 2^{i_m} \; &\text{if} \; j = m
\end{cases}$, \\
we have $\mcD_{h_1, \delta_1}(g_0)=g_1$. 

If $\mant{\mcD_{h_1, \delta_1}(g_0)} = 1^{++}$, by taking $w_j = \begin{cases}
    \mant{g_1} \times 2^{i_j} \; &\text{if} \; j =1 \\
     (2- 2^{-\mbit+2}) \times 2^{i_j} \; &\text{if} \; j = 2 , \mbit \ge 3 \\
     1^+ \times 2^{i_j} \; &\text{if} \; j = 2 , \mbit =2 \\
    2^{i_j} \; &\text{if} \; j \in [m-1] \setminus [2] \\
    \sn{g_0} \times \mant{g_0}^\dag \times 2^{i_m} \; &\text{if} \; j = m
\end{cases}$, \\
we have $\mcD_{h_1, \delta_1}(g_0)=g_1$.

This completes the proof.
 
\end{proof}

\begin{proof}[Proof of \cref{lemma:gradient_control2}]
Consider $L\ge  \ceilZ{\max\lrp{\frac{2^{\ebit}+\mbit}{\eta + \expo{\sigma'(\delta_1)}}, \frac{2^{\ebit}+\mbit}{-\emin - \expo{\sigma'(\delta_1)}}}}$.
    For any $k\in [n]$ and $y\in \fpq$ such that $0<|y|\le \Omega\times 2^{\expo{\sigma'(\delta_1)}}$, by \cref{lemma:gradient_control}, there exists an $L$-layer $\sigma$ network $f_k:\efpq\to\efpq$ starting with $\sigma$ such that 
    $$f_k:  \delta_1 \mapsto  0~~\text{and}~~\mcD_{f, \delta_1}(y) = y^*_k.$$
    Then $\widetilde{f_k}:x\mapsto f_k(x_k)$ is an $L$-layer neural network starting with $\sigma$.
    Define $f$ as 
    \begin{equation*}
        f\defeq \bigfuncsum_{k=1}^n \widetilde{f_k}.
    \end{equation*}
    Then $f(\delta_1\otimes \boldsymbol{1}_n)=0$ and $\mcD_{f,\delta_1\otimes \boldsymbol{1}_n} = y^*$.
    This completes the proof.
\end{proof}

\subsection{Proof of \cref{lemma:first_layer2}}\label[appendix]{appendix:proof_first_layer}

\begin{lemma}\label[lemma]{lemma:make_proper_number}
    Let $x, x_0 \in \fpq$ satisfy
    \begin{equation*}
       \frac{1}{2}\times 2^{\expo{x} -\mbit} < {x_0} < \frac{3}{2}\times 2^{\expo{x} -\mbit}.
    \end{equation*}
   Then there exists a natural number $n\in\bbN$ such that
    \begin{equation*}
        \bigoplus_{i=1}^n x_0 = |x|.
    \end{equation*}
\end{lemma}
\begin{proof}
    Consider the largest integer $n$ such that 
    \begin{equation*}
        \bigoplus_{i=1}^n x_0 < |x|.
    \end{equation*}
    Such an $n$ exists, since for any $y\in \fpq$ such that $y<2^{\expo{x}}$,
    \begin{equation*}
        y\oplus x_0 > y.
    \end{equation*}
    By maximality of $n$, we have
\begin{equation*}
    \bigoplus_{i=1}^{n+1} x_0 \ge |x|.
\end{equation*}
On the other hand,
\begin{equation*}
    \bigoplus_{i=1}^{n+1} x_0 \le \bigoplus_{i=1}^{n} x_0\oplus x_0\le |x|^- \oplus x_0\le |x|,
\end{equation*}
    Therefore,
    \begin{equation*}
        \bigoplus_{i=1}^{n+1} x_0 = x.
    \end{equation*}
    This completes the proof.
\end{proof}

\begin{lemma}\label[lemma]{lemma:addition_representability}
Consider $x, y \in \fpq$ satisfying $xy > 0$, $|y| \ge |x|$, and
$\mant{y,\mbit} = 0$.
Then there exists a floating-point number $b \in \fpq$ such that
\begin{equation*}
    b \oplus x = y,
\end{equation*}
and
\begin{equation*}
    |b| \le |y|.
\end{equation*}
\end{lemma}
\begin{proof}
Without loss of generality, assume that $y > x > 0$.

If $\expo{x} = \expo{y}$, define
\begin{equation*}
    b \defeq (\mant{y} - \mant{x}) \times 2^{\expo{x}}.
\end{equation*}
Then $b \oplus x = y$.

Hence, it suffices to consider the case $\expo{y} > \expo{x} \ge \emin$.
Define $b_0 \in \fpq$ as the largest floating-point number such that
\begin{equation*}
    b_0 \oplus x \le y.
\end{equation*}
By definition,
\begin{equation*}
    b_0^+ \oplus x > y.
\end{equation*}
Assume, for the sake of contradiction, that
\begin{equation*}
    b_0 \oplus x < y.
\end{equation*}

\textbf{Case 1: $\mant{y,\mbit} = 0$ and $y \neq 2^{\expo{y}}$.}
As $b_0 \oplus x < y$, we have
\begin{equation*}
    b_0 + x < y - 2^{\expo{y} - \mbit - 1}.
\end{equation*}
Moreover, $b_0^+ \oplus x > y$ implies
\begin{equation*}
    b_0 + 2^{\expo{b_0} - \mbit} + x
    > y + 2^{\expo{y} - \mbit - 1} = ( y  - 2^{\expo{y} - \mbit - 1} ) + 2^{\expo{y}-\mbit} > b_0 + x + 2^{\expo{y}-\mbit}.
\end{equation*}
Therefore,
\begin{equation*}
    \expo{b_0} > \expo{y},
\end{equation*}
which contradicts $b_0 + x < y$.

\textbf{Case 2: $y = 2^{\expo{y}}$.}
From $b_0 \oplus x < y$, we obtain
\begin{equation*}
    b_0 + x < y - 2^{\expo{y} - \mbit - 2}.
\end{equation*}
From $b_0^+ \oplus x > y$, we obtain
\begin{equation*}
    b_0 + 2^{\expo{b_0} - \mbit} + x
    > y + 2^{\expo{y} - \mbit - 1}.
\end{equation*}
Combining the two inequalities yields
\begin{equation*}
    2^{\expo{b_0} - \mbit}
    > \frac{3}{4} \cdot 2^{\expo{y} - \mbit}.
\end{equation*}
Hence,
\begin{equation*}
    \expo{b_0} = \expo{y}.
\end{equation*}
Since $b_0 \le y$, it follows that
\begin{equation*}
    b_0 = y = 2^{\expo{y}}.
\end{equation*}
Consequently,
\begin{equation*}
    x < -2^{\expo{y} - \mbit - 2},
\end{equation*}
which contradicts $x > 0$.
Thus, the assumption $b_0 \oplus x < y$ leads to a contradiction.
Therefore,
\begin{equation*}
    b_0 \oplus x = y.
\end{equation*}
This completes the proof.

\end{proof}

\begin{lemma}\label[lemma]{lemma:first_layer}
Suppose $\sigma$ satisfies \cref{condition:first_layer_new} and  $\domain\subset[-\zeta,\zeta]_{\fpq}$ is $\sigma$-distinguishable with range $[-2^{\emax}, 2^{\emax}]$.
Then, for any $x_0\in \domain$ and $g\in\fpq$, there exist $n\in \bbN$, $n_i\in \bbN$ for $i\in [n]$, and parameters $g_{i,j}, w_{i},b_i, w_{i,j}'\in \fpq$ for $i\in [n]$ and $j\in [n_i]$
    such that $ |g_{i,j}|\le 2^{\emax+\expo{\sigma'(\delta_1)}}$, $ g= \bigoplus_{i=1}^n \lrp{\bigoplus_{j=1}^{n_i} g_{i,j}\otimes w_{i,j}'}\otimes \sigma'(w_i\otimes x_0\oplus b_i)\otimes w_i$,
    and $\lrb{w_{i,j}'\otimes \sigma (w_i\otimes x\oplus b_i)}\le 2^{\emax}$, for any $x\in \domain$.
\end{lemma}
\begin{proof}
First, for each $i\in [n]$, by \cref{lemma:make_proper_number}, we can choose appropriate $n_i$, $w_{i,j}'=1$, and $g_{i,j} \in \fpq$ such that $\bigoplus_{j=1}^{n_i} g_{i,j}\otimes w_{i,j}'$ realizes an arbitrary floating-point number.
Thus, the problem reduces to find $g'_i, w_i, b_i\in \fpq$ such that
\begin{equation*}
    g = \bigoplus_{i=1}^n g_i'\otimes \sigma'(w_i\otimes x_0\oplus b_i)\otimes w_i,
\end{equation*}
and
\begin{equation*}
    \lrb{\sigma (w_i\otimes x\oplus b_i)}\le 2^{\emax} \text{  for } x \in \domain.
\end{equation*}
By \cref{condition:first_layer_new}, there exists $x_1\in\fpq$ such that $x_1$ is even and $|x_1| > \zeta \ge  |x_0|$,
\begin{equation*}
    \sigma\left( \; \left[\; -( \zeta \oplus |x_1|),  \zeta  \oplus |x_1| \; \right]_{\fpq}\; \right)\subset \left[-2^{\emax}, 2^{\emax}\right], \text{ and } |\sigma'(x_1)|\ge 2^{-\mbit},
\end{equation*}
or
\begin{equation*}
     \sigma([-2 \zeta,  2 \zeta ]_{\fpq})\subset \left[-2^{\emax}, 2^{\emax}\right], \text{ and }  |\sigma'(0)|\ge 2^{-\mbit}.
\end{equation*}

In the first case, 
by \cref{lemma:addition_representability}, there exists $b\in \fpq$ such that $b\oplus |x_0| = |x_1|$, and $|b|\le |x_1|$. For any $|w|\le 1$, we have $|b\oplus w \otimes x| \le |x_0| + |x_1| \le \zeta + |x_1|$ and   $|\sigma( b \oplus w\otimes x)| \le 2^{\emax}$ for $x \in \domain$.  \\
In the second case, we define $x_1$ and $b$ as $x_1\defeq 0$ and $b\defeq -x_0$, respectively. Then $| b \oplus w\otimes x| \le |x_0|+|x| \le 2 \zeta$ and $|\sigma( b \oplus w\otimes x)| \le 2^{\emax}$ for $x \in \domain$. 

In both cases, we have $|\sigma( b \oplus w\otimes x)| \le 2^{\emax}$ for $x \in \domain$.

Now choose $j_1,j_2\in [\emin, \emax]_{\bbZ}$ such that $j_1+j_2 = \expo{g}-\expo{\sigma'(x_1)}-\mbit$, $\emin -\expo{\sigma'(x_1)}\le j_1\le \emax -\expo{\sigma'(x_1)}  $, and $\emin+1\le j_2\le 0$.

By \cref{lem:one-plus}, there exists $\mant{} \in \fpq$ such that $\mant{}\otimes \mant{\sigma'(x_1)} =1^+$.

Define $g_i'$ and $w_i$ as $g'_i\defeq 
       \sn{g} \times  \mant{} \times 2^{j_1}$ and $w_i\defeq 2^{j_2} \le 1$. 
By \cref{lemma:addition_representability}, there exist $b_i \in \fpq$ such that $b_i \oplus (w_i \otimes x_0) = x_1$. 
Then
\begin{equation*}
   {g'_i\otimes \sigma'(w_i \otimes x_0 \oplus b_i) \otimes w_i = }  g'_i\otimes \sigma'(x_1) \otimes w_i = \sn{g} \times \lrp{1^+\times 2^{j_1 + \expo{\sigma'(x_1)}}} \otimes 2^{j_2}
    = \sn{g} \times \round{1^+\times 2^{\expo{g}-\mbit}}.
\end{equation*}
By \cref{lemma:make_proper_number}, there exists a natural number $n$ such that
\begin{equation*}
    {\bigoplus_{i=1}^n g_i'\otimes \sigma'(w_i\otimes x_0\oplus b_i)\otimes w_i} = \bigoplus_{i=1}^n g'_i\otimes \sigma'(x_1) \otimes w_i= g.
\end{equation*}

    This completes the proof.
\end{proof}

\begin{proof}[Proof of \cref{lemma:first_layer2}]
    By \cref{lemma:first_layer}, for any $x_0= (x_{0,1},\dots, x_{0,d})\in \domain$, there exists $g_{i,j}, w_i,b_i, w'_{i,j}$ such that $   \lrp{y^*}_k = \bigoplus_{i=1}^n \lrp{\bigoplus_{j=1}^{n_i} g_{i,j}\otimes w_{i,j}'}\otimes \sigma'(w_i\otimes x_{0,k}\oplus b_i)\otimes w_i$ with $ |g_{i,j}|\le 2^{\emax+\expo{\sigma'(\delta_1)}}$.
Observe that the right-hand side coincides with the output gradient of a network $H_k:\fpq^d\to\fpq^{\sum_{i=1}^n n_i}$ whose output is the concatenation of $n$ blocks, where the $i$-th block is the $n_i$-dimensional vector computed by
\begin{equation*}
   h_i: x\mapsto \bigconcat_{j=1}^{n_i} \lrp{w_{i,j}\otimes \sigma(w_j\otimes x_k \oplus b_j)}\in \fpq^{n_i},
\end{equation*}
where $x = (x_1,\dots,x_d)$. 
Consider a sequential addition $s:(-2^{\emax}, 2^{\emax})\tran (\delta_1, \delta_1)$.
We define $H_k$ by
\begin{equation*}
    H_k\defeq \bigconcat_{i=1}^n s\circ h_i.
\end{equation*}
Then we can check that
\begin{equation*}
    H_k: \domain\mapsto (\delta_1,\dots, \delta_1). 
\end{equation*}
Note that when the input gradient of $H_k$ corresponding to the $(i,j)$-th output coordinate is $g_{i,j}$ and $x_k = x_{0,k}$, the resulting output gradient is $(y^*)_k\otimes e_k$.
As $H_k$ depends only on the $k$-th input coordinate, the function $\tilde H$ defined by 
\begin{equation}
   \tilde H\defeq  \bigconcat_{k=1}^d H_k ,
\end{equation}
is the desired function.

\end{proof}

\subsection{Proof of \cref{lemma:gradient_splitting}}\label[appendix]{appendix:proof_gradient_splitting}
\begin{proof}
We first assume $c_1 \le c_2$. In other case, we can apply similar argument with negative weight.

Consider $\phi_1\in \ASP{\delta_0,\delta_1}{\delta_0,\delta_1}{2^{\expo{\sigma'(\delta_1)}}}$ and $\phi_2\in \ASP{\delta_0,\delta_1}{\delta_0,\delta_0}{2^{\expo{\sigma'(\delta_1)}}}$.
The we have 
\begin{align*}
    \phi_1(\delta_0) &= \delta_0, \quad \phi_2(\delta_0) = \delta_0, \\
    \phi_1(\delta_1) &= \delta_1, \quad \phi_2(\delta_1) = \delta_0.
\end{align*}
By \cref{condition:gradient_two_exponent}, $\sigma(\delta_0) \ne \sigma(\delta_1)$ and $ 2^{-\expo{\sigma'(\delta_1)}} \otimes \sigma(\delta_0) , 2^{-\expo{\sigma'(\delta_1)}} \otimes \sigma(\delta_0) \in \fpq$, hence
\begin{align*}
    2^{-\expo{\sigma'(\delta_1)}} \otimes \sigma(\delta_1) \ominus 2^{-\expo{\sigma'(\delta_1)}} \otimes \sigma(\delta_0) \ne 0. 
\end{align*}
Then consider a sequential addition $s$ such that 
\begin{equation*}
    s: (0,2^{-\expo{\sigma'(\delta_1)}} \otimes \sigma(\delta_1) \ominus 2^{-\expo{\sigma'(\delta_1)}} \otimes \sigma(\delta_0)  )\tran (c_1,c_2).
\end{equation*}

Now define the function $f$ as 
\[
f(x)=s\left(2^{-\expo{\sigma'(\delta_1)}}\otimes \sigma(\phi_1(x)) \ominus 2^{-\expo{\sigma'(\delta_1)}}\otimes \sigma(\phi_2(x)) \right),
\]
and we have 
\begin{align*}
    f(\delta_0) = s(0) = c_1, \quad f(\delta_1) = s(2^{-\expo{\sigma'(\delta_1)}}\otimes \delta_1 \ominus   \expo{\sigma'(\delta_1)}\otimes \delta_1)  = c_2. 
\end{align*}
For  any $y \in [-\Omega\times 2^{\expo{\sigma'(\delta_1)}}, \Omega\times 2^{\expo{\sigma'(\delta_1)}}]_{\fpq}$,  we have 
\begin{align*}
    \mcD_{f,\delta_0}(y) &= 0, \\
    \mcD_{f,\delta_1}(y) &= (y \otimes \sigma'(\delta_1) \otimes \sigma'(\delta_1))  \ominus (y \otimes \sigma'(\delta_1) \otimes \sigma'(\delta_0)), \\
    0 < |\mcD_{f,\delta_1}(y)| &\le y.
\end{align*}
This completes the proof.

\end{proof}

\begin{proof}
Note that $\sigma'(\delta_1)=\sn{\sigma'(\delta_1)} \times  2^{\expo{\sigma'(\delta_1)}}$ for $\sn{\sigma'(\delta_1)} \in \{1, -1 \}$ by \cref{condition:gradient_two_exponent}.
       Consider $w_1 =\sn{\sigma'(\delta_1)}\times  2^{-\expo{\delta_1}}$ 
       and let $f_1\in \ASP{\delta_0, \delta_1}{\delta_0,\delta_1}{w_1}$.
    Then $\mcD_{f_1, x}(g)$ is given by
    \begin{equation*}
        \mcD_{f_1, x}(g) = \begin{cases}
      g\otimes \lrp{\sn{\sigma'(\delta_1)} \times 2^{-\expo{\sigma'(\delta_1)}}} \otimes \lrp{\sn{\sigma'(\delta_1)} \times 2^{\expo{\sigma'(\delta_1)}}} = g  &\text{ if } x= \delta_0,
      \\ g \otimes \lrp{\sn{\sigma'(\delta_1)} \times 2^{-\expo{\sigma'(\delta_1)}}} \otimes \sigma'(\delta_0)  &\text{ if } x= \delta_1.
  \end{cases} 
    \end{equation*}
    Moreover, we have
    \begin{equation*}
     \lrb{g \otimes \lrp{\sn{\sigma'(\delta_1)} \times 2^{-\expo{\sigma'(\delta_1)}}} \otimes \sigma'(\delta_0)}   = \lrb{\lrp{g\times  2^{-\expo{\sigma'(\delta_1)}}} \otimes \sigma'(\delta_0)}
     \le g \otimes 2^{-e_0}.
    \end{equation*}
by \cref{condition:gradient_two_exponent}.
    Therefore, for $n\in \bbN$ satisfying $n\ge (2^{\ebit}+\mbit)/e_0$, we obtain 
    \begin{equation*}
         \mcD_{f_1^{\circ n}, x}(g) =\begin{cases}
     0  &\text{ if } x= \delta_0,
      \\ g  &\text{ if } x= \delta_1.
  \end{cases} 
    \end{equation*}  
    Finally, note that the output values can be replaced by any $c_1, c_2\in\fpq$ satisfying $|c_1|,|c_2|,|c_1-c_2|\le 2^{\emax}$ by composing a sequential addition in the final layer.
This completes the proof.
\end{proof}

\section{Proofs for Real Activation Functions}

\begin{lemma}\label[lemma]{lem:2k_control}
    For $a,b \in \bbR$ with $2^{\emin} \le a  \le 2^{-k} b \le  b \le  2^{\emax}$ for some $k \in \bbN_{\ge 0}$, then $\round{a} \le 2^k \round{b}$. 
\end{lemma}
\begin{proof}
Since $ 2^{\emin} \le b , \; 2^{-k}b \le 2^{\emax}$, we have $2^{-k}\round{b}=\round{2^{-k}b}$ and we get the desired result. 
\end{proof}

\subsection{Proof of \cref{thm:main1-activation}}\label[appendix]{appendix:proof_lemma_main1-activation}
To prove \cref{thm:main1-activation}, we use \cref{lemma:real_activations,lemma:real_activations_sigmoid,lemma:real_activations_tanh}  presented below. \\

First note that, by \cref{lemma:real_activations,lemma:real_activations_sigmoid,lemma:real_activations_tanh}, $\domain= [-M_\sigma,M_\sigma]_{\fpq}$ is $\round{\hatsig}$-distinguishable with range $\left[-2^{\emax}, \allowbreak 2^{\emax}\right]_{\fpq}$ for $\sigma \in \Sigma$ .

Suppose $\sigma \in \{\round{\relu},\round{\elu}, \round{\GeLU}, \round{\Swish}\}$. By \cref{lemma:real_activations}, $\sigma$  satisfies \cref{condition:final_layer_refined,condition:gradient_two_exponent,condition:first_layer_new} with $\kappa= \emax+1, \delta_0=-2^{\ebit+1}, \sigma'(\delta_1)=1$,  $\eta=\emax-1, \zeta= 2^{\emax-2}$.
Note that $\tau$ (\cref{eq:tau} in \cref{sec:formal}) is 
\begin{align*}
    \tau &\defeq\ceilZ{{\max\lrp{\frac{2^{\ebit}+\mbit}{\eta +\expo{\sigma'(\delta_1)}}+4,  \frac{2^{\ebit}+\mbit}{-\emin - \expo{\sigma'(\delta_1)}}+ 4,7}}} = \ceilZ{{\max\lrp{\frac{2^{\ebit}+\mbit}{\emax-1 }+4,  \frac{2^{\ebit}+\mbit}{-\emin }+ 4,7}}} \\
&= \ceilZ{{\max\lrp{\frac{2^{\ebit}+\mbit}{ 2^{\ebit-1}-2 }+4,  \frac{2^{\ebit}+\mbit}{ 2^{\ebit-1}-2 }+ 4,7}}} = 7, \\
     &\min\lrp{\Omega\times 2^{\sigma'(\delta_1)}, 2^{\kappa}}  = \min\lrp{ \Omega , 2^{\emax+1} } = \Omega,
\end{align*}
since $2^{\ebit} + \mbit \le 5 \cdot  2^{{\ebit}-2}   \le 3 \cdot (2^{{\ebit}-1} -2)$ for ${\ebit} \ge 6$.

By \cref{thm:main1}, for any $L\ge 7$, there exists an $L$-layer $\sigma$ network $f$ such that $f(x)=f^*(x)$ and  
$\mcD_{f,x}(h^*(x))=g^*(x)$ for $|h^*(x)|\le \Omega = H_\sigma$ and all $x\in\domain$.

Suppose $\sigma$ is $\round{\Sigmoid}$. By \cref{lemma:real_activations_sigmoid}, $\sigma$ satisfies \cref{condition:final_layer_refined,condition:gradient_two_exponent,condition:first_layer_new} with $\kappa= \mbit-1, \round{\hatsig'}(\delta_1)=\tfrac14$,  $\eta=\emax, \zeta= 2^{\emax-1}$. Then we have 
\begin{align*}
    \tau &\defeq\ceilZ{{\max\lrp{\frac{2^{\ebit}+\mbit}{\eta +\expo{\sigma'(\delta_1)}}+4,  \frac{2^{\ebit}+\mbit}{-\emin - \expo{\sigma'(\delta_1)}}+ 4,7}}} = \ceilZ{{\max\lrp{\frac{2^{\ebit}+\mbit}{\emax-3 }+4,  \frac{2^{\ebit}+\mbit}{-\emin - 2 }+ 4,7}}} \\
 &= \ceilZ{{\max\lrp{\frac{2^{\ebit}+\mbit}{ 2^{\ebit-1}-4 }+4,  \frac{2^{\ebit}+\mbit}{ 2^{\ebit-1}-4 }+ 4,7}}} = 7, \\
    &\min\lrp{\Omega\times 2^{\sigma'(\delta_1)}, 2^{\kappa}}  = \min\lrp{ \Omega \times 2^{-2}, 2^{\mbit-1} } = 2^{\mbit-1},
\end{align*}
since $2^{\ebit} + \mbit \le 5 \cdot  2^{{\ebit}-2} \le 3 \cdot (2^{{\ebit}-1}-4)$ for ${\ebit} \ge 6$.

By \cref{thm:main1}, for any $L\ge 7 $, there exists an $L$-layer $\sigma$ network $f$ such that $f(x)=f^*(x)$ and  
$\mcD_{f,x}(h^*(x))=g^*(x)$ for $|h^*(x)|\le 2^{\mbit-1}$  all $x\in\domain$. \\
Suppose $\sigma$ is $\round{\tanh}$. By \cref{lemma:real_activations_tanh} $\sigma$ satisfies \cref{condition:final_layer_refined,condition:gradient_two_exponent,condition:first_layer_new} with $\kappa= \mbit-2, \round{\hatsig'}(\delta_1)=1$,  $\eta=\emax-1$. Then we have 
\begin{align*}
    \tau &\defeq\ceilZ{{\max\lrp{\frac{2^{\ebit}+\mbit}{\eta +\expo{\sigma'(\delta_1)}}+4,  \frac{2^{\ebit}+\mbit}{-\emin - \expo{\sigma'(\delta_1)}}+ 4,7}}} = \ceilZ{{\max\lrp{\frac{2^{\ebit}+\mbit}{\emax-1 }+4,  \frac{2^{\ebit}+\mbit}{-\emin }+ 4,7}}} \\
 &= \ceilZ{{\max\lrp{\frac{2^{\ebit}+\mbit}{ 2^{\ebit-1}-2 }+4,  \frac{2^{\ebit}+\mbit}{ 2^{\ebit-1}-2 }+ 4,7}}} = 7, \\
    &\min\lrp{\Omega\times 2^{\sigma'(\delta_1)}, 2^{\kappa}}  = \min\lrp{ \Omega , 2^{\mbit-1} } = 2^{\mbit-2},
\end{align*}
since $2^{\ebit} + \mbit \le  5\cdot 2^{{\ebit}-2} \le 3\cdot (2^{{\ebit}-1}-2)$ for ${\ebit} \ge 6$.

By \cref{thm:main1}, for any $L\ge 7$, there exists an $L$-layer $\sigma$ network $f$ such that $f(x)=f^*(x)$ and  
$\mcD_{f,x}(h^*(x))=g^*(x)$ for all $x\in\domain$. \\

\begin{lemma}\label[lemma]{lemma:real_activations}
Let $\hatsig$ be one of $\relu$, $\elu$, $\GeLU$, $\Swish$, $\elu$. Then $\domain= [-2^{\emax-2},2^{\emax-2}]_{\fpq} $ is $\round{\hatsig}$-distinguishable with range $\left[-2^{\emax}, \allowbreak 2^{\emax}\right]_{\fpq}$ and $\round{\hatsig}$ satisfies \cref{condition:final_layer_refined,condition:gradient_two_exponent,condition:first_layer_new} with $\kappa= \emax+1, \delta_0=-2^{{\ebit}}, \round{\hatsig'}(\delta_1)=1$,  $\eta=\emax-1, \zeta= 2^{\emax-2}$. 
\end{lemma}

\begin{lemma}\label[lemma]{lemma:real_activations_sigmoid}
Let $\hatsig$ be  $\Sigmoid$. Then $\domain= [-2^{\emax-1},2^{\emax-1}]_{\fpq} $ is $\round{\hatsig}$-distinguishable with range $\left[-2^{\emax}, \allowbreak 2^{\emax}\right]_{\fpq}$ and $\round{\hatsig}$ satisfies \cref{condition:final_layer_refined,condition:gradient_two_exponent,condition:first_layer_new} with $\kappa= \mbit-1, \round{\hatsig'}(\delta_1)=\tfrac14$,  $\eta=\emax-1, \zeta= 2^{\emax-1}$. 
\end{lemma}

\begin{lemma}\label[lemma]{lemma:real_activations_tanh}
Let $\hatsig$ be  $\tanh$. Then $\domain= [-2^{\emax-1},2^{\emax-1}]_{\fpq} $ is $\round{\hatsig}$-distinguishable with range $\left[-2^{\emax}, \allowbreak 2^{\emax}\right]_{\fpq}$ and $\round{\hatsig}$ satisfies \cref{condition:final_layer_refined,condition:gradient_two_exponent,condition:first_layer_new} with $\kappa= \mbit-2, \round{\hatsig'}(\delta_1)=1$,  $\eta=\emax-1, \zeta= 2^{\emax-1}$. 
\end{lemma}

The detailed proofs of \cref{lemma:real_activations,lemma:real_activations_sigmoid,lemma:real_activations_tanh} are presented in \cref{appendix:proof_real_activations,appendix:real_activations_sigmoid,appendix:real_activations_tanh}.

\subsection{Proof of \cref{thm:main2-activation}}\label[appendix]{appendix:proof_lemma_main2-activation}

From the proof of \cref{thm:main1-activation}, we observe that $\sigma$ satisfies \cref{condition:final_layer_refined,condition:gradient_two_exponent,condition:first_layer_new} with {$\tau = 7$}. Hence, by \cref{thm:main2}, for $L\ge 2^{\ebit+1}+2\mbit+9$, there exists an $L$-layer $\sigma$ network $f$ such that $f(x)=f^*(x)$ and  
$\mcD_{f,x}(y)=g^*(x,y)$ for all $x\in\domain$ and {$y\in [-H_\sigma,H_\sigma]_{\fpq}$}.

\subsection{The case of $4\le {\ebit} \le5$  in \cref{thm:main1-activation,thm:main2-activation}}\label[appendix]{appendix:proof_q5}

In the case $\ebit=5$, the number of layers required in 
\cref{thm:main1-activation} increases to $8$ when 
$\sigma=\round{\Sigmoid}$, while it remains unchanged for the other activations.
In the case $\ebit=4$, the number of layers required in 
\cref{thm:main1-activation} increases by $1$ for all activations in $\Sigma$.

For \cref{thm:main2-activation} with $\ebit=5$, the required number of layers 
increases to $2^{\ebit+1}+2\mbit+10$ when $\sigma=\round{\Sigmoid}$, while it 
remains unchanged for the other activations.
In the case $\ebit=4$, the number of layers required in 
\cref{thm:main2-activation} increases by $1$ for all activations in $\Sigma$.

Suppose $\sigma \in \{\round{\relu},\round{\elu}, \round{\GeLU}, \round{\Swish}\}$. Then $\tau$ is 
\begin{align*} 
\tau = \ceilZ{{\max\lrp{\frac{2^{\ebit}+\mbit}{ 2^{\ebit-1}-2 }+4,  \frac{2^{\ebit}+\mbit}{ 2^{\ebit-1}-2 }+ 4,7}}} = \begin{cases} 
8 \; &\text{if} \; {\ebit} = 4 \\
7  \; &\text{if} \; {\ebit} = 5 
\end{cases},
\end{align*}

Suppose $\sigma$ is $\round{\Sigmoid}$. Then we have 
\begin{align*} 
\tau = \ceilZ{{\max\lrp{\frac{2^{\ebit}+\mbit}{ 2^{\ebit-1}-4 }+4,  \frac{2^{\ebit}+\mbit}{ 2^{\ebit-1}-4 }+ 4,7}}} = \begin{cases} 
9 \; &\text{if} \; {\ebit} = 4 \\
8  \; &\text{if} \; {\ebit} = 5 
\end{cases},
\end{align*}

Suppose $\sigma$ is $\round{\tanh}$. Then we have
\begin{align*} 
\tau = \ceilZ{{\max\lrp{\frac{2^{\ebit}+\mbit}{ 2^{\ebit-1}-2 }+4,  \frac{2^{\ebit}+\mbit}{ 2^{\ebit-1}-2 }+ 4,7}}} = \begin{cases} 
8 \; &\text{if} \; {\ebit} = 4 \\
7  \; &\text{if} \; {\ebit} = 5 
\end{cases},
\end{align*}

\begin{corollary}
Let $4 \le \ebit \le 5$, $\sigma\in\Sigma$, $\domain = [-M_\sigma,M_\sigma]_{\fpq}^d$, $f^*:\domain\to\fpq$, $h^*:\domain\to[-H_\sigma,H_\sigma]_{\fpq}$, and $g^*:\fpq^{d}\to\fpq^d$ such that $g^*(x)=0$ for all $x\in\mcX$ with $h^*(x)=0$.
Then, for any $L\ge \tau$, there exists an $L$-layer $\sigma$ network $f$ such that $f(x)=f^*(x)$ and  
$\mcD_{f,x}(h^*(x))=g^*(x)$ for all $x\in\domain$ where $$\tau = \begin{cases}
    13 \; &\text{if} \; {\ebit} = 4 ,\sigma = \round{\Sigmoid}, \\
    10 \; &\text{if} \; {\ebit}=4 ,\sigma \in \Sigma_\sigma \setminus \{\round{\Sigmoid}\} \; \text{or} \; {\ebit} = 5, \sigma \in \{ \Sigma \setminus \{\round{\tanh}\}  \} \\
    9 \; &\text{if} \; {\ebit}= 5, \sigma = \tanh.
\end{cases}$$
\end{corollary}
\begin{corollary}
Let $4 \le \ebit \le 5$, $\sigma\in\Sigma$, $\domain = [-M_\sigma,M_\sigma]_{\fpq}^d$, $f^*:\domain\to\fpq$, and $g^*:\fpq^{d}\times\fpq\to\fpq^d$ such that $g^*(x,-y)=-g(x,y)$.
Then, for any $L\ge 2^{\ebit+1}+2\mbit+2 + \tau$, there exists an $L$-layer $\sigma$ network $f$ such that $f(x)=f^*(x)$ and  
$\mcD_{f,x}(y)=g^*(x,y)$ for all $x\in\domain$ and $y\in [-H_\sigma,H_\sigma]_{\fpq}$ where $$\tau = \begin{cases}
    13 \; &\text{if} \; {\ebit} = 4 ,\sigma = \round{\Sigmoid}, \\
    9 \; &\text{if} \; {\ebit}= 5, \sigma = \tanh.
\end{cases}$$
\end{corollary}

\subsection{Lemmas about \cref{condition:final_layer_refined,condition:gradient_two_exponent,condition:first_layer_new} on real activation functions}

\begin{lemma}\label[lemma]{lemma:realcondition}
Suppose a continuous function $\hat{\sigma} : \bbR \to \bbR$ is  $C^2$ on $(0,\infty)$ and  satisfies the following with $\tfrac{1}{2}\le c_1\le 1 \le c_2\le2$:
\begin{enumerate}[label=\textnormal{(A\arabic*)}]
    \item \label{cond:increasing}  $\hatsig(x)$ is increasing and  $ c_1x \le \hatsig(x) \le c_2 x $ for $x \ge 0$.  
     \item \label{cond:prime} $ \tfrac12 \le \hat{\sigma}'(x) \le \tfrac32$  for $x \ge 0$.
     \item  \label{cond:twoprime} $\begin{cases}
          1 \in \hat{\sigma}'\left( [\tfrac12,1] \right) \text{ and } |\hat{\sigma}''(x)| \le  \tfrac{3}{4}\; \text{ for } \;\tfrac12 \le x \le 1 \\ 
          1 \in \hat{\sigma}'\left( [1,2] \right) \text{ and }  |\hat{\sigma}''(x)| \le  \tfrac{2}{5} \;  \text{ for } \;  1 \le x \le 2. 
     \end{cases}$ 
     \item \label{cond:real_neg}  $ |\hat{\sigma}(x)| \le 1$ for $x \le 0$, and $| \hat{\sigma}'(x) | \le 2^{x}$ for $x \le -10$. 
\end{enumerate}
Then $\round{\hat{\sigma}}=\sigma$ satisfies \cref{condition:final_layer_refined} 
with 
$\nu=5, \gamma_1=2^{-\mbit-1}, \gamma_2 = 2^{-3}, \gamma_3 = 2, \gamma_4 = 2^4 ,  \gamma_5 = (2-2^{-\mbit}) \times 2^{\emax-\mbit+2} ,   \kappa = \emax + 1 $, 
satisfies \cref{condition:gradient_two_exponent} with $\delta_0=-2^{{\ebit}+1},$  $\eta = \emax- 1$ and  some $\delta_1 \in (\tfrac{1}{2},2)_\fpq$, and satisfies \cref{condition:first_layer_new} with $\zeta=2^{\emax-2}$. 

\end{lemma}

\begin{proof}
    
\textbf{\cref{condition:final_layer_refined} : } \\
By \cref{cond:increasing}, we have $\hat{\sigma}(0)=0$ leading to $\gamma_0=0$ and $\round{\hatsig}(\gamma_0)=0$ and $\round{\hatsig'}(\gamma_0)=0$ by \cref{cond:prime}.

Let $\nu=5, \gamma_1=2^{-\mbit-1}, \gamma_2 = 2^{-3}, \gamma_3 = 2, \gamma_4 = 2^4 ,  \gamma_5 = (2-2^{-\mbit}) \times 2^{\emax-\mbit+2} ,   \kappa = \emax + 1 $. 

 By \cref{cond:prime}, we have 
\begin{align*}
  \tfrac12 \le  | \hatsig'(\gamma_i)| &\le \frac32 \le  2^{-k_i} c_1 \gamma_i  \le 2^{-k_i}|\hatsig(\gamma_i)|, \quad 2^{k_i} \le \frac13 2^{\expo{\gamma_i}} , \;    k_i \le  \expo{\gamma_i} - 2, \quad i=1, \dots , 4\\  
\tfrac12 \le  | \hatsig'(\gamma_5)| &\le \frac32 \le  2^{-k_5`} c_1 \gamma_5  \le 2^{-k_i}|\hatsig(\gamma_i)|, \quad 2^{k_5} \le \frac13 \times  (2-2^{-\mbit}) \times 2^{\emax-\mbit+2}, \; k_i \le \emax-\mbit+1, 
\end{align*}
By \cref{lem:2k_control}, we have 
\begin{align*}
   | \round{\hatsig'}(\gamma_i)| &\le  2^{-k_i}|\round{\hatsig}(\gamma_i)|.
\end{align*}

Hence we pick $k_1=-\mbit-3$, $k_2 = -5,k_3=-1, k_4 = 2, k_5 = \emax-\mbit+1$.

Since $\emin \le -\mbit-1, \emin \le -6, \expo{\gamma_i}-1 \le \expo{\hatsig(\gamma_i)} \le \expo{\gamma_i}+1$ by \cref{cond:increasing}, we have 
\begin{align*}
 \bigcup_{i=1}^{\nu} 
       & [\expo{\hatsig(\gamma_i)} + \emin,\emax -\kappa  +\min\lrp{\expo{\hatsig(\gamma_i)} ,  k_i}] \supset  \\
       & [\emin - \mbit , -\mbit - 4]_\bbZ \cup [\emin-2 , -6]_\bbZ \cup [\emin+2 , -2]_\bbZ \cup [\emin+5 , 1]_\bbZ \cup  [-\mbit+4 , \emax -\mbit]_\bbZ \\
       &= [\emin-\mbit, \emax -\mbit] .
\end{align*}

\textbf{\cref{condition:gradient_two_exponent} : } \\
Suppose there is no $\delta_1 \in [1,2]_\fpq$ such that $\round{\hatsig'}(\delta_1)=1$. 
Since $\hatsig$ satisfies \cref{cond:twoprime}, we consider following cases. 

\paragraph{\textbf{Case 1:} } $ 1 \in \hat{\sigma}'\left( [\tfrac12,1] \right) \text{ and } |\hat{\sigma}''(x)| \le  \tfrac{3}{4}\; \text{ for } \;\tfrac12 \le x \le 1 $.  \\  

\paragraph{\textbf{Case 2:} } $ 1 \in \hat{\sigma}'\left( [1,2] \right) \text{ and } |\hat{\sigma}''(x)| \le  \tfrac{2}{5}\; \text{ for } \;1 \le x \le2  $.  \\

Let $[a,b] = \begin{cases}
    [\tfrac12,1] \; &\text{if} \; \textbf{Case 1},\\ 
    [1,2] \; &\text{if} \; \textbf{Case 2}. \\ 
\end{cases},  M_* = \begin{cases}
    \tfrac34 \; &\text{if} \; \textbf{Case 1},\\ 
    \tfrac25 \; &\text{if} \; \textbf{Case 2}. \\ 
\end{cases}$ 

Let $A :=  \{ x \in [a,b]_\fpq  :  \; \round{\hatsig'}(y) < 1 \text{ for } y\in [a, x]_\fpq  \}$. We claim that $A \ne \emptyset, [a,b]_\fpq $. 

First suppose $A = \emptyset$, \text{i.e.} $\round{\hatsig'}(x) > 1$ for all $x \in [a,b]_\fpq$. 
By \cref{cond:twoprime}, pick $\xi \in [a,b] \setminus [a,b]_\fpq $ such that $\hatsig'(\xi)=1$. Then there exist $\theta \in [a,b)_\fpq$ such that $\theta < \xi < \theta^+$ with $\round{\hatsig'}(\theta),\round{\hatsig'}(\theta^+) > 1$. Then we have $\hatsig'(\theta),\hatsig'({\theta^+}) > 1+2^{-\mbit-1}$. By mean value theorem, there exist $t_1,t_2 \in (\theta,\theta^+)$ such that 
\begin{align*}
    \hatsig'(\xi)-\hatsig'(\theta)= \hatsig''(t_1)(\xi-\theta), \quad \hatsig'(\theta^+)-\hatsig'(\hatsig)= \hatsig''(t_2)(\theta^+-\xi). 
\end{align*}
Since $(\xi-\theta)+(\theta^+-\xi) = \theta^+-\theta \le \begin{cases}
    2^{-\mbit-1} \; &\text{if} \; \textbf{Case 1} \\
    2^{-\mbit} \; &\text{if} \; \textbf{Case 2}
\end{cases} $, we have 
$$(\xi-\theta) \le  \begin{cases}
    2^{-\mbit-2} \; &\text{if} \; \textbf{Case 1}, \\
    2^{-\mbit-1} \; &\text{if} \; \textbf{Case 2},
\end{cases}   \text{  or  }  (\theta^+-\xi) \le   \begin{cases}
    2^{-\mbit-2} \; &\text{if} \; \textbf{Case 1}, \\
    2^{-\mbit-1} \; &\text{if} \; \textbf{Case 2}.
\end{cases} $$ 
Since 
\begin{align*}
    |\hatsig''(t_1)| = \frac{ \hatsig'(\xi)-\hatsig'(\theta)}{\xi-\theta} > \frac{2^{-\mbit-1}}{\xi-\theta}, \quad  |\hatsig''(t_2)|  > \frac{2^{-\mbit-1}}{\theta^+-\xi}, 
\end{align*}
we have 
\[ |\hatsig''(t_1)| >  \begin{cases}
     2            \; &\text{if} \; \textbf{Case 1}, \\
    1 \; &\text{if} \; \textbf{Case 2},
\end{cases} \text{  or  } |\hatsig''(t_2)| >  \begin{cases}
     2            \; &\text{if} \; \textbf{Case 1}, \\
    1 \; &\text{if} \; \textbf{Case 2},
\end{cases}   \]
which is contradiction to \cref{cond:twoprime}. 
Hence $A \ne \emptyset$. \\

Next suppose $A =  [a, b]_\fpq$, \text{i.e.} $\round{\hatsig'}(x) < 1$ for all $x \in [a ,b]_\fpq$.  Pick $\xi \in [a,b] \setminus [a,b]_\fpq $ such that $\hatsig'(\xi)=1$. Then there exist $\theta \in [a,b)_\fpq$ such that $\theta < \xi < \theta^+$ with $\round{\hatsig'}(\theta),\round{\hatsig'}(\theta^+) < 1$. Then we have $\hatsig'(\theta),\hatsig'({\theta^+}) < 1-2^{-\mbit-2}$.

Since $(\xi-\theta)+(\theta^+-\xi) = \theta^+-\theta \le \begin{cases}
    2^{-\mbit-1} \; &\text{if} \; \textbf{Case 1} \\
    2^{-\mbit} \; &\text{if} \; \textbf{Case 2}
\end{cases} $, we have 
$$(\xi-\theta) \le  \begin{cases}
    2^{-\mbit-2} \; &\text{if} \; \textbf{Case 1}, \\
    2^{-\mbit-1} \; &\text{if} \; \textbf{Case 2},
\end{cases}   \text{  or  }  (\theta^+-\xi) \le   \begin{cases}
    2^{-\mbit-2} \; &\text{if} \; \textbf{Case 1}, \\
    2^{-\mbit-1} \; &\text{if} \; \textbf{Case 2}.
\end{cases} $$ 
Since 
\begin{align*}
    |\hatsig''(t_1)| = \frac{ \hatsig'(\xi)-\hatsig'(\theta)}{\xi-\theta} > \frac{2^{-\mbit-2}}{\xi-\theta}, \quad  |\hatsig''(t_2)|  > \frac{2^{-\mbit-2}}{\theta^+-\xi}, 
\end{align*}
we have 
\[ |\hatsig''(t_1)| >  \begin{cases}
     1            \; &\text{if} \; \textbf{Case 1}, \\
    \tfrac12 \; &\text{if} \; \textbf{Case 2},
\end{cases} \text{  or  } |\hatsig''(t_2)| >  \begin{cases}
     1            \; &\text{if} \; \textbf{Case 1}, \\
    \tfrac12 \; &\text{if} \; \textbf{Case 2},
\end{cases}   \]
which is contradiction to \cref{cond:twoprime}.

Therefore $A \ne \emptyset , [a,b]_\fpq$. Let $\xi_* := \max A <b$. \\
 Then by \cref{cond:twoprime}, we have  $\round{\hatsig'}(\xi_*^+) > 1$ and $\hatsig'(\xi_*^+) > 1+ 2^{-\mbit-1}$. 

\[ \hatsig'(\xi_*) \ge  \hatsig'(\xi_*^+) - \max_{x\in(a,b)} |\hatsig''(x)|(\xi_*^+- \xi_* ) \ge \hatsig'(\xi_*^+) - M_* (\xi_*^+- \xi_* ) > 1 - 2^{-\mbit-2}, \]
leading to $\round{\hatsig'}(\xi_*) \ge 1$ which is a contradiction. Therefore, we have $\delta_1 \in [1,2]_\fpq$ such that $\round{\hatsig'}(\delta_1)=1$.

By \cref{cond:increasing}, $ \tfrac12\le \hatsig(\delta_1) \le 2$. 

Let $  \delta_0= -2^{{\ebit}+1}$ and $\eta = \emax-1$. Since $-2^{2^{{\ebit}-1}-1} \le \delta_0 \le -10$, $\delta_0 \in \fpq$. By \cref{cond:real_neg}, we have $|\hatsig'(\delta_0)| \le 2^{-2^{{\ebit}+1}} \le 2^{-8\mbit}$. By \cref{cond:prime,cond:real_neg},
\begin{align*}
    |\delta_0|,|\delta_1|, |\delta_0-\delta_1| &\le 2^{{\ebit}+1}+2  \le  2^{\emax}, \\
   |\round{\hatsig}(\delta_0)|,|\round{\hatsig}(\delta_1)|, |\round{\hatsig}(\delta_0)-\round{\hatsig}(\delta_1)| &\le 5 < 8 \le  2^{\emax}, \\ 
     - 2^{-\mbit - 1}  \round{\hatsig'}(\delta_1) &= -  2^{-\mbit - 1}  \le  2 \round{\hatsig'}(\delta_0) \le 1 = \round{\hatsig'}(\delta_1), \\ 
\max(|\round{\hatsig}(\delta_0)|,|\round{\hatsig}(\delta_1)|) &\le 2 = 2^{\emax-\eta}.
\end{align*}

\textbf{\cref{condition:first_layer_new} : } \\
Take $\zeta = 2^{\emax-2}$. 
Now we verify the first item in \cref{condition:first_layer_new}. Let $x_0 \in [-\zeta, \zeta]$ and take $x = \zeta$ and $x$ is even. By \cref{cond:prime}, $|\round{\hatsig'}(x) | \ge \tfrac12 \ge 2^{-\mbit}$. By \cref{cond:increasing}, we have 
\begin{align*}
    x \oplus \zeta = 2\zeta \in \fpq, \; \zeta \oplus |x| = 2 \zeta, \; \hatsig([-( \zeta \oplus |x|),0]) \subset [0,\frac{1}{4}], \; \hatsig([0, \zeta \oplus |x|]) \subset [0,4 \zeta ] \subset [0,2^{\emax}]. 
\end{align*} 
Therefore, 
\begin{align*}
    \round{\hatsig}([-(\zeta \oplus |x|),(\zeta \oplus |x|)]) \subset [-2^{\emax},2^{\emax}].
\end{align*}

\end{proof}

\begin{lemma}\label[lemma]{lemma:realcondition_sigmoid}

Suppose continuous function $\hatsig : \bbR \to \bbR$ satisfies the following with $\tfrac{1}{2}\le c_1\le 1 \le c_2\le2$:

\begin{enumerate}[label=\textnormal{(B\arabic*)}]
    \item $ 4^{x} \le |\hatsig(x)| , |\hatsig'(x)| \le 2^{x}$ for $x \le -2$. \label{condB:zero}
    \item \label{condB:bound} $ c_1 \le |\hatsig(x)| \le c_2  $, $\hatsig'(x) \le 1$ for $x \ge 0$.  
     \item \label{condB:prime} $\hatsig(x)$ is increasing and $\hatsig'(0) \ge \tfrac14$. 
     \item $|\hatsig''(x)| \le \tfrac{1}{10}$ and is $C^2$ on $(0,\infty)$.  \label{condB:twoprime}
     \item  $ \tfrac14 \in \hatsig'\left( [0,2] \right)$. \label{condB:range}
     \item $ |\hatsig(x)| \le 1$ for $x \le 0$. \label{condB:real_neg}
\end{enumerate}

Then $\round{\hatsig}=\sigma$ satisfies \cref{condition:final_layer_refined} 
with  $\kappa= \mbit-1$, 
satisfies \cref{condition:gradient_two_exponent} with $\eta=\emax-1$, and satisfies \cref{condition:first_layer_new} with $\zeta=2^{\emax-1}$. 

\end{lemma}

\begin{proof}

\textbf{\cref{condition:final_layer_refined} : } \\
Let $\gamma_0 = - 2^{{\ebit}}$. Since $ - 2^{\ebit} \ge -2^{\emax}= -2^{2^{{\ebit}-1}-1}$ for ${\ebit} \ge 3$, $\gamma_0 \in \fpq$. 
By \cref{condB:zero}, we have 
\[ |\hatsig(\gamma_0)| , |\hatsig'(\gamma_0)| \le 2^{-2^{\ebit}} \le  2^{-2^{{\ebit}-1}-\mbit+1 } =  2^{\emin-\mbit -1} = \frac{\fmin}{2} ,  \]
leading to $\round{\hatsig}(\gamma_0)=0$ and $\round{\hatsig'}(\gamma_0)=0$.

Let $\nu=\mbit-2 , \kappa = \mbit-1$. Pick $\gamma_1 , \gamma_2 , ... , \gamma_{\mbit-3}, \gamma_{\mbit-2} \in \fpq$ such that $\gamma_1, \dots , \gamma_{\mbit-3} < 0$, $\gamma_{\mbit-2}>0$ and 
\begin{align*}
     2^{-\mbit-2+i}\le |\hatsig(\gamma_i)| \le 2^{-\mbit-1+i}, \; i=1,\dots,\mbit-3,  \quad    2^{-1} \le |\hatsig(\gamma_{\mbit-2})|
\end{align*}

Since $\hatsig$ is continuous and \cref{condB:zero,condB:bound}, such $\gamma_1, \dots, \gamma_{\mbit-3} \in \fpq$ exists with
\[  -\mbit-2+i   \le  \gamma_i \le   \frac{-\mbit-1+i}{2},  \; i=1,\dots,\mbit-3,  \quad  0 <  \gamma_{\mbit-2} \le 1.  \]
By \cref{condB:zero}, we have 
\[  2^{-2\mbit-4+2i}   \le  |\hatsig'(\gamma_i)| \le   2^{\frac{-\mbit-1+i}{2}},  \; i=1,\dots,\mbit-3,  \quad  |\hatsig'(\gamma_{\mbit-2})| \le 1.  \]

Note that 
\begin{align*}
    |\gamma_i - \gamma_j| \le 2^{\ebit}+ 1 \le  2^{\emax}, \quad i,j \in \{0\}\cup [\mbit-2].
\end{align*}

 By \cref{condB:bound,condB:prime}, we have 
$$k_i=  \floorZ{-{\tfrac{\mbit}{2} } -3/2 + i/2 }, i=1,\dots \mbit-3 , \quad k_{\mbit-2} =-1$$
and 
\begin{align*}
    | \hatsig'(\gamma_i)| &\le 2^{(-\mbit-1+i)/2}  \le  2^{-k_i-\mbit-2+i}   \le 2^{-k_i}|\hatsig(\gamma_i)|, \quad i= 1 , \dots \mbit - 4\\
    | \hatsig'(\gamma_{\mbit-3})| &\le 2^{-2} \le  2^{-k_{\mbit-3}-5}   \le 2^{-k_{\mbit-3}}|\hatsig(\gamma_{\mbit-3})|, \\
    | \hatsig'(\gamma_{\mbit-2})| & \le 1  \le  2^{-k_{\mbit-2}-1}   \le 2^{-k_{\mbit-2}}|\hatsig(\gamma_{\mbit-2})|.
\end{align*}
Then we have 
\begin{align*}
 \bigcup_{i=1}^{n_0} 
       & [\expo{\hatsig(\gamma_i)} + \emin,\emax -\kappa  +\min\lrp{\expo{\hatsig(\gamma_i)} ,  k_i}]_\bbZ \supset  \\
       & [\emin - \mbit , \emax -2\mbit  ]_\bbZ \cup \dots \cup [\emin-4 , \emax -\mbit -4  ]_\bbZ  \cup [\emin-1 , \emax -\mbit]_\bbZ  = [\emin-\mbit, \emax -\mbit]_\bbZ .
\end{align*}

\textbf{\cref{condition:gradient_two_exponent} : } \\

Suppose there is no $\delta_1 \in [0,2]_\fpq$ such that $\round{\hatsig'}(\delta_1)=2^{-2}$. 

Let $A :=  \{ x \in [0,2]_\fpq  :  \; \round{\hatsig'}(y) < 2^{-2} \text{ for } y\in [0, 2]_\fpq  \}$. We claim that $A \ne \emptyset, [0,2]_\fpq $.

Suppose $A =\emptyset$. Pick $\xi \in [0,2] \setminus [0,2]_\fpq $ such that $\hatsig'(\xi)=2^{-2}$. Then there exist $\theta \in [0,2)_\fpq$ such that $\theta < \xi < \theta^+$ with $\round{\hatsig'}(\theta),\round{\hatsig'}(\theta^+) > 2^{-2}$. Then we have $\hatsig'(\theta),\hatsig'(\theta) > 2^{-2}+2^{-\mbit-3}$. By mean value theorem, there exist $t_1,t_2 \in (\theta,\theta^+)$ such that 
\begin{align*}
    \hatsig'(\xi)-\hatsig'(\theta)= \hatsig''(t_1)(\xi-\theta), \quad \hatsig'(\theta^+)-\hatsig'(\xi)= \hatsig''(t_2)(\theta^+-\xi). 
\end{align*}
Since $(\xi-\theta)+(\theta^+-\xi) \le 2^{-\mbit}$, we have  $(\xi-\theta) \le 2^{-\mbit-1}$ or  $(\theta^+-\xi) \le 2^{-\mbit-1} $. Since 
\begin{align*}
    |\hatsig''(t_1)|  > \frac{2^{-\mbit-3}}{\xi-\theta}, \quad  |\hatsig''(t_2)|  > \frac{2^{-\mbit-3}}{\theta^+-\xi}, 
\end{align*}
we have $|\hatsig''(t_1)| > 2^{-2}$ or $|\hatsig''(t_2)| > 2^{-2}$ which is contradiction to \cref{condB:twoprime}. Hence $A \ne \emptyset$. \\

Now suppose $A =[0,2]_\fpq$.  Pick $\xi \in [0,2] \setminus [0,2]_\fpq $ such that $\hatsig'(\xi)=2^{-2}$. Then there exist $\theta \in [0,2)_\fpq$ such that $\theta < \xi < \theta^+$ with $\round{\hatsig'}(\theta),\round{\hatsig'}(\theta^+) < 2^{-2}$. Then we have $\hatsig'(\theta),\hatsig'(\theta) > 2^{-2}-2^{-\mbit-4}$. By mean value theorem, there exist $t_1,t_2 \in (\theta,\theta^+)$ such that 
\begin{align*}
    \hatsig'(\xi)-\hatsig'(\theta)= \hatsig''(t_1)(\xi-\theta), \quad \hatsig'(\theta^+)-\hatsig'(\xi)= \hatsig''(t_2)(\theta^+-\xi). 
\end{align*}
Since $(\xi-\theta)+(\theta^+-\xi) \le 2^{-\mbit}$, we have  $(\xi-\theta) \le 2^{-\mbit-1}$ or  $(\theta^+-\xi) \le 2^{-\mbit-1} $. Since 
\begin{align*}
    |\hatsig''(t_1)|  > \frac{2^{-\mbit-4}}{\xi-\theta}, \quad  |\hatsig''(t_2)|  > \frac{2^{-\mbit-4}}{\theta^+-\xi}, 
\end{align*}
we have $|\hatsig''(t_1)| > 2^{-3}$ or $|\hatsig''(t_2)| > 2^{-3}$ which is contradiction to \cref{condB:twoprime}. Hence $A \ne [0,2]_\fpq$. 

Let $\xi_* := \max A < 2 $. 
Then by \cref{condB:range}, we have  $\round{\hatsig'}(\xi_*^+) > 2^{-2}$ and $\hatsig'(\xi_*^+) > 2^{-2}+ 2^{-\mbit-3}$. 
Let $M_*:= \max_{x\in(0,2)} |\hatsig''(x)|$. By \cref{condB:twoprime}, $M_* \le \tfrac{1}{10}$ and  
\[ \hatsig'(\xi_*) \ge  \hatsig'(\xi_*^+) - M_*(\xi_*^+- \xi_* ) > 2^{-2} + \frac{1}{40}, \]
leading to $\round{\hatsig'}(\xi_*) \ge 2^{-2}$ which is a contradiction.

Therefore, we have $\delta_1 \in \fpq$ such that $\round{\hatsig'}(\delta_1)=2^{-2}$. 

By \cref{condB:bound,condB:prime}, $ \tfrac12\le \hatsig(\delta_1) \le 2$.
Let $\delta_0=-2^{{\ebit}+2}$ and $\eta=\emax-1$. Since $ -2^{2^{{\ebit}-1}-1} \le  -2^{{\ebit}+2}$ for ${\ebit} \ge 4$, we have $\delta_0 \in \fpq$.
By \cref{condB:zero}, $0 \le \hatsig'(\delta_0) \le 2^{-2^{{\ebit}+2}}$ and $\round{\hatsig'}(\delta_0) \le 2^{-2^{{\ebit}+2}}$. Then we have 
\begin{align*}
    |\delta_0|,|\delta_1|, |\delta_0-\delta_1| \le 2^{{\ebit}+2} + 2   &<  2^{\emax}, \\
   |\round{\hatsig}(\delta_0)|,|\round{\hatsig}(\delta_1)|, |\round{\hatsig}(\delta_0)-\round{\hatsig}(\delta_1)| &\le 3  <  2^{\emax} \\ 
    - 2^{-\mbit - 1}  \round{\hatsig'}(\delta_1) &= -  2^{-\mbit - 3} < 0 \le  2 \round{\hatsig'}(\delta_0) \le 2^{-2} = \round{\hatsig'}(\delta_1), \\ 
\max(|\round{\hatsig}(\delta_0)|,|\round{\hatsig}(\delta_1)|)
   &\le 2 \le 2^{\emax-\eta}.
\end{align*}

\textbf{\cref{condition:first_layer_new} : } \\

Take $\zeta = 2^{\emax-1}$. 
Now we verify the second item in \cref{condition:first_layer_new}. By \cref{condB:prime}, we have $\hatsig'(0)\ge \tfrac14$ and $\round{\hatsig'}(0)\ge \tfrac14$. By \cref{condB:bound,condB:real_neg}, we have $\hatsig([-2\zeta,2\zeta]) \in [-2,2] \subset [-2^{\emax},2^{\emax}] $.
\end{proof}

\begin{lemma}\label[lemma]{lemma:realcondition_tanh}

Suppose continuous function $\hatsig : \bbR \to \bbR$ satisfies the following with $\tfrac{1}{2}\le c_1\le 1 \le c_2\le \tfrac32$:

\begin{enumerate}[label=\textnormal{(C\arabic*)}]
    \item \label{condC:zero} $\hatsig(0)=0$. 
    \item $   8^x \le |\hatsig'(x)| \le 2^{x}$ for $x \le -2$. \label{condC:sigma_bound}
    \item \label{condC:bound} $ c_1x \le \hatsig(x), \hatsig'(x) \le c_2x  $ for $ 0 \le x \le 1$.  
     \item \label{condC:prime} $\hatsig(x)$ is increasing and $\hatsig'(0) =  1$. 
     \item $ |\hatsig(x)| \le 1$ for $x \in \bbR$. \label{condC:real_neg}
\end{enumerate}

Then $\round{\hatsig}=\sigma$ satisfies \cref{condition:final_layer_refined} 
with $\kappa= \mbit-2$, 
satisfies \cref{condition:gradient_two_exponent} with $\eta=\emax-1$ and satisfies \cref{condition:first_layer_new} with $\zeta=2^{\emax-1}$.

\end{lemma}
\begin{proof}
    
\textbf{\cref{condition:final_layer_refined} : } \\

Let $\nu=3, \gamma_1=2^{-\mbit-1}, \gamma_2 = 2^{-3}, \gamma_3 = 1,   \kappa = \mbit -2 $.

Then we have 
$k_1=-\mbit-3$, $k_2 = -5,k_3=-2$ since
\begin{align*}
    | \hatsig'(\gamma_i)| &\le  \frac32 \le  2^{-k_i} c_1 \gamma_i  \le 2^{-k_i}|\hatsig(\gamma_i)|, \quad k_i =\expo{\gamma_i} - 2.
\end{align*}

Since $\emin \le -\mbit-1$, we have 
\begin{align*}
 \bigcup_{i=1}^{n_0} 
       & [\expo{\hatsig(\gamma_i)} + \emin,\emax -\kappa  +\min\lrp{\expo{\hatsig(\gamma_i)} ,  k_i}]_\bbZ \supset  \\
       & [\emin - \mbit , \emax - 2\mbit-1 ]_\bbZ \cup [\emin-2 , \emax -\mbit  - 3 ]_\bbZ  \cup [\emin + 1  , \emax -\mbit]_\bbZ \\
       &= [\emin-\mbit, \emax -\mbit]_\bbZ .
\end{align*}

\textbf{\cref{condition:gradient_two_exponent} : } \\

By \cref{condC:prime}, we have $\hatsig'(\delta_1) = 1$ for $\delta_1 = 0$. 
Let $\delta_0=-2^{{\ebit}+2}$ and $\eta = \emax- 1$. Similar to above case, $\delta_0 \in \fpq$.  By \cref{condC:sigma_bound}, $\hatsig'(\delta_0) \le 2^{-2^{{\ebit}+2}}$ and $\round{\hatsig'}(\delta_0) \le 2^{-2^{{\ebit}+2}} \le 2^{-16\mbit}$. Then we have 
\begin{align*}
    |\delta_0|,|\delta_1|, |\delta_0-\delta_1| \le 4\emax   &<  2^{\emax}, \\
   |\round{\hatsig}(\delta_0)|,|\round{\hatsig}(\delta_1)|, |\round{\hatsig}(\delta_0)-\round{\hatsig}(\delta_1)| \le 1   &<  2^{\emax} \\ 
   - 2^{-\mbit - 1}  \round{\hatsig'}(\delta_1) = -  2^{-\mbit - 1} \le 2\round{\hatsig'}(\delta_0)  & \le   1 = \round{\hatsig'}(\delta_1), \\ 
  \max(|\round{\hatsig}(\delta_0)|,|\round{\hatsig}(\delta_1)|)
  & =1  \le 2^{\emax-1}.
\end{align*}

\textbf{\cref{condition:first_layer_new} : } \\

Take $\zeta = 2^{\emax-1}$. 
Now we verify the second item in \cref{condition:first_layer_new}. By \cref{condC:prime}, we have $\hatsig'(0) = 1$ and $\round{\hatsig'}(0)=1$. By \cref{condC:real_neg}, we have $\hatsig([-2\zeta,2\zeta]) \in [-2,2] \subset [-2^{\emax},2^{\emax}] $.

\end{proof}

\subsection{Proof of \cref{lemma:real_activations}}\label[appendix]{appendix:proof_real_activations}

Since conditions \cref{cond:increasing}-\cref{cond:real_neg} implies \cref{condition:final_layer_refined,condition:gradient_two_exponent,condition:first_layer_new} by \cref{lemma:realcondition}, we need to show that  $\hatsig$ satisfies \cref{cond:increasing}-\cref{cond:real_neg}.

For $\relu$, $\elu$, $\GeLU$, $\Swish$, $\elu$, one can verify \cref{cond:increasing}-\cref{cond:twoprime} and the first condition of \cref{cond:real_neg} by referring \cref{table:realcond}. Also note that 
 $\domain= [-\zeta,\zeta]_{\fpq} $ is $\round{\hatsig}$-distinguishable with range $\left[-2^{\emax}, \allowbreak 2^{\emax}\right]_{\fpq}$ by Lemma 3.10 in \cite{hwang2025floating}.
where $$2^{\emax-2} \le \zeta = 2^{\emax-\floor{\log_2(c_2)}-1} \le 2^{\emax-1}.$$ 

Hence we only need to show the second of \cref{cond:real_neg}, \textit{i.e.} $| \hat{\sigma}'(x) | \le 2^{x}$ for $x \le -10$.\\

For $\sigma=\relu$, it is trivial since $\sigma'(x)=0$ for $x \le 0$. \\
For $\sigma=\elu= \begin{cases}
    x \quad \text{if} \; x >0 \\
    e^x-1 \quad \text{if} \; x >0 \\
\end{cases}$, we have $\sigma'(x) = e^x \le 2^x $ for $ x \le 0$. \\
For $\sigma=\GeLU= x \Phi(x)$ where $\Phi(x) = \int_{-\infty}^x \frac{1}{\sqrt{2\pi}} e^{-t^2/2}dt$ is the standard normal cumulative distribution function. Then we have $\sigma'(x) = \Phi(x)+x\phi(x)$ where $\phi(x) = \frac{1}{\sqrt{2\pi}} e^{-x^2/2}$.
Note that for $ a > 0$, 
\begin{align*}
    \int_a^\infty \phi(t)dt =  \int_a^\infty \frac{-\phi'(t)}{t} dt = \left[ - \frac{\phi(t)}{t} \right]_{a}^\infty  - \int_a^\infty  \frac{\phi'(t)}{t^2} \le \frac{\phi(a)}{a}
\end{align*}
where we use $\phi'(x) = -x\phi(x)$. Hence for $x < 0$,  
\begin{align*}
    \Phi(x) = \int_{-\infty}^x \phi(t)dt = \int_{-x}^\infty \phi(t)dt \le \frac{\phi(-x)}{-x}.
\end{align*}

Therefore, for $x \le -1$, 
\begin{align*}
    \sigma'(x) = \Phi(x) + x\phi(x) \le \frac{\phi(x)}{-x} + x \phi(x) = \phi(x) \left( x-\frac{1}{x} \right) \le \phi(x) \left( \frac{x^2-1}{x} \right) \le 0,
\end{align*}
and we have $|\sigma'(x)| \le x \phi(x)$ since $\Phi(x)>0$.
Hence we need to show $|\sigma'(x)| \le |x \phi(x)| \le 2^x$ for $x \le -3$. It suffices to show :
\begin{align*}
    a \phi(a) = a \frac{1}{\sqrt{2\pi}}e^{-a^2/2} \le 2^{-a}, \; a \ge 3.
\end{align*}
Since both are positive, take natural logs and we have 
\begin{align*}
    \ln{a} -\frac{1}{2}\ln{2 \pi} - \frac{a^2}{2} \le -a \ln{2}. 
\end{align*}
Define $F(\cdot)$ as 
\begin{align*}
    F(a) = \frac{a^2}{2} - a \ln{2}-\ln{a}+\frac{1}{2}\ln{2 \pi},
\end{align*}
and we need to show $F(a) \ge 0 , a \ge 3$. Note that $F(3) \approx 2.24 > 0$. By taking derivative, 
\begin{align*}
    F'(a) = a - \ln{2}-\frac{1}{a} > 0 ,\quad  a \ge3. 
\end{align*}
Since $F(3)>0$ and $F(a)$ is increasing for $a \ge 3$, we have $F(a) > 0$ for $a>3$ and $|\sigma'(x)|\le 2^{x}$ for $x \le -3$. \\

For $\sigma(x)= \Swish(x) =x s(x)$, where $s(x) = \frac{1}{1+e^{-x}}$ is $\Sigmoid$. We have $\sigma'(x) = s(x)+xs'(x) = s(x)+xs(x)(1-s(x))$. For $x\le -3$, we have $s(x) = \frac{1}{1+e^{-x}} \ge \frac{1}{1+e^{-3}} >  \frac{1}{2} $, 
\begin{align*}
    \sigma'(x)  = s(x) \left( 1 + x(1-s(x)) \right) < s(x) \left( 1 + (-3)(\frac{1}{2}) \right) = s(x)(-\frac{1}{2}) < 0. 
\end{align*}
Since $\sigma'(x) < 0$ for $ x \le -6$, we have
\begin{align*}
    | \sigma'(x)| =  | s(x)+xs(x)(1-s(x))| \le |xs(x)(1-s(x))| \le |\frac{x}{1+e^{-x}}| \le |xe^x| \le 2^x.
\end{align*}
for $x \le -10$. 

\subsection{Proof of \cref{lemma:real_activations_sigmoid}}\label[appendix]{appendix:real_activations_sigmoid}

Since conditions \cref{condB:zero}-\cref{condB:real_neg} implies \cref{condition:final_layer_refined,condition:gradient_two_exponent,condition:first_layer_new} by \cref{lemma:realcondition_sigmoid}, we need to show that  $\hatsig$ satisfies \cref{condB:zero}-\cref{condB:real_neg}.

One can verify \cref{condB:bound}-\cref{condB:real_neg} by referring \cref{table:realcond_sigmoid}. 

 By \cref{condB:zero}, there exists a separating point $\theta_1\le -4$. By \cref{condB:bound}, there exist a separating point $ \theta_2 \in [0,2)_\fpq$. 
By Lemma 3.7 in \cite{hwang2025floating}, $\domain= [-\zeta,\zeta]_{\fpq} $ is $\round{\hatsig}$-distinguishable with range $\left[-2^{\emax}, \allowbreak 2^{\emax}\right]_{\fpq}$,
where $\zeta \le  \fmax.$

Hence we only need to show the second of \cref{condB:zero}, \textit{i.e.} $4^x \le | \hat{\sigma}(x)|, |\hat{\sigma}'(x) | \le 2^{x}$ for $x \le -2$.
Since $1+e^t \le 4^t$ for $t \ge 2$, we have   
\begin{align*}
    4^{x} \le \hatsig(x) = \frac{1}{1+e^{-x}} \le e^{x} \le  2^{x}, \quad x \ge -2.
\end{align*}
Since $\hatsig'(x) = \hatsig(x)(1-\hatsig(x))$, we have 
\begin{align*}
    \hatsig(x) (1- \hatsig(-2)) \le \hatsig(x)(1-\hatsig(x)) \le \hatsig(x) \le  2^{x}, \quad x \ge -2.
\end{align*}
We need to show $\hatsig(x)(1-\hatsig(-2)) \ge \hatsig(x)(0.8)\ge 4^x$. Since $\hatsig(-2)(0.8) \approx 0.105 \ge 4^{-2}$ and $4^x$ is decreasing faster as $x \to -\infty$, we get the desired result.

\subsection{Proof of \cref{lemma:real_activations_tanh}}\label[appendix]{appendix:real_activations_tanh}

Since conditions \cref{condC:zero}-\cref{condC:real_neg}  implies \cref{condition:final_layer_refined,condition:gradient_two_exponent,condition:first_layer_new} by \cref{lemma:realcondition_tanh}, we need to show that  $\hatsig$ satisfies \cref{condC:zero}-\cref{condC:real_neg}.

By \cref{condC:sigma_bound}, there exists a separating point $\theta_1\le -4$. By \cref{condC:bound}, there exist a separating point $ \theta_2 \in [0,1)_\fpq$. 

By Lemma 3.7 in \cite{hwang2025floating}, $\domain= [-\zeta,\zeta]_{\fpq} $ is $\round{\hatsig}$-distinguishable with range $\left[-2^{\emax}, \allowbreak 2^{\emax}\right]_{\fpq}$,
where $$\zeta \le  \fmax.$$ 

One can verify \cref{condC:zero}-\cref{condC:real_neg}  since $\hatsig(0)=0, \hatsig'(0)=1$, $0.76 x \le \hatsig(x) \le x $ for $0 \le x \le 1$, $\hatsig(x)$ is increasing and $|\hatsig(x)|\le 1$ for $x \in \bbR$. We only need to show  $8^x \le |\hat{\sigma}'(x) | \le 2^{x}$ for $x \le -2$ and we have 
\begin{align*}
  8^x \le e^{2x} \le \frac{4e^{2x}}{4}  \le \hatsig'(x) = \frac{4e^{2x}}{(1+e^{2x})^2} \le 4e^{2x} \le 2^x. 
\end{align*}

\begin{table}[h]
\centering
    \begin{tabular}{cccccc@{\;\;\;}ccccc}
\toprule
\begin{tabular}{@{}c@{}}
  $\hatsig(x)$ 
\end{tabular}
& $\gamma_0$ & $(c_1,c_2)$ & $\hatsig'((0,\infty))$ & $\hatsig''((a,b))$ & $\hatsig'([a,b])$ & $\hatsig((-\infty,0])$ & $f(x)=\tfrac{\hatsig'(x)}{2^{-x}}, f(-\infty,-3)$   \\ 
\midrule
$\relu$   & 0 & (1,1)  & $\{1\}$   & $\{0\}$  & $\{1\}$ & $\{0\}$ & $\{0\}$    \\
$\elu$    & 0 & (1,1)    &  $\{1\}$  & $\{0\}$& $\{1\}$ & $\subset [-1,0]$ & $\subset(-0.08,0)$    \\
$\GeLU$   & 0 & (0.5,1) & $\subset[0.5,1.13)$    & $\subset (0.24,0.62) $ & $\supset(0.85,1.08)$ & $\subset(-0.18,0]$ & $\subset(-0.002,0)$ \\
$\Swish$   & 0 & (0.5,1) & $\subset(0.5,1.10)$   & $\subset (0.05,0.31)$ & $\supset(0.93,1.09)$ & $\subset((-0.28,0))$ & $\subset(-0.03,0)$ \\
\bottomrule
\end{tabular}
\vspace{0.2in}
\caption{%
  Properties of various real activations functions. Note that $(a,b)$ is $(\tfrac12,2)$ for $\hatsig=\relu,\elu$, $(\tfrac12,1)$ for $\hatsig=\GeLU$ and $(1,2)$ for $\hatsig=\Swish$.
}
\label{table:realcond}
\end{table}

\begin{table}[h]
\centering
 \begin{tabular}{cccccc@{\;\;\;}ccccc}
\toprule
$\hatsig(x)$ & $\gamma_0$ & $(c_1,c_2)$ & $\hatsig'(0)$ & $\hatsig''(\bbR)$ & $\hatsig'([0,2])$ & $\hatsig((-\infty,0])$    \\ 
\midrule
$\Sigmoid$   & 0 & (0.5,1)  & $\tfrac14$   & $\subset(-0.10,0) $  & $\subset (0.1,0.25]$ &$(0,0.5]$    \\
\midrule
$\hatsig(x)$ & $\gamma_0$ & $(c_1,c_2)$ & $\hatsig'(0)$ & $\hatsig''(\bbR)$ & $\hatsig'([0,2])$ & $\hatsig((-\infty,0])$    \\ 
\midrule
$\tanh$   & 0 & (0.7616,1)  & $\tfrac14$   & $\subset(-0.10,0) $  & $\subset (0.1,0.25]$ &$(0,0.5]$    \\
\bottomrule
\end{tabular}
\vspace{0.2in}
\caption{%
  Properties of floating-point format for verifying the conditions.
}
\label{table:realcond_sigmoid}
\end{table}

\bibliography{reference}
\bibliographystyle{plainnat}

\end{document}

%% file: figures/table_activation.tex
\begin{table}[h]
\caption{Constants in \cref{condition:final_layer_refined,condition:gradient_two_exponent,condition:first_layer_new}. $\relu^\ast$ denotes $\relu$, $\elu$, $\GeLU$, and $\Swish$.
See \cref{lemma:real_activations,lemma:real_activations_sigmoid,lemma:real_activations_tanh} for derivations.}
\normalsize
\centering
\setlength{\tabcolsep}{1pt}
    \begin{tabular}{@{}p{1.2cm}c@{\hspace{1pt}}c@{\hspace{3pt}}c@{\hspace{3pt}}c@{\hspace{3pt}}c@{\hspace{3pt}}c@{\hspace{3pt}}@{}}
\toprule
\begin{tabular}{@{}c@{}}
  $\hatsig(x)$ 
\end{tabular}
& $\kappa$ & $\sigma'(\delta_1) $ & $\eta$ & $\zeta$  \\ 
\midrule
$\relu^\ast$  & $\emax-1$ &  1 & $2^{\ebit\hspace{-0.007in}-\hspace{-0.007in}1}-2$ & $2^{\emax\hspace{-0.007in}-\hspace{-0.007in}2}$   \\
$\Sigmoid$   & $\mbit-1$ & $1/4$ & $2^{\ebit\hspace{-0.007in}-\hspace{-0.007in}1}-2$ & $2^{\emax\hspace{-0.007in}-\hspace{-0.007in}1}$\\
$\tanh$   & $\mbit-2$& $ 1$ & $2^{\ebit\hspace{-0.007in}-\hspace{-0.007in}1}-1$ & $2^{\emax-1}$ \\
\bottomrule
\end{tabular}
\vspace{0pt}
\label{table:act}
\end{table}